\documentclass[11 pt]{article}
\usepackage{style}
\usepackage{setspace}
\usepackage{mlmodern}
\newtheorem{condition}{Condition}[section]
\newcommand{\sm}{\sigma_\tau^{\bar{\epsilon}}}

\title{\LARGE\bfseries Decentralized Best-Response-Based Learning in Two-Player Zero-Sum Stochastic Games: A Finite-Sample Analysis}
\author{
Zaiwei Chen\thanks{IE, Purdue University. Email: \href{mailto:zchen458@purdue.edu}{chen5252@purdue.edu}.}
\and
Kaiqing Zhang\thanks{ECE, University of Maryland, College Park. Email: \href{mailto:kaiqing@umd.edu}{kaiqing@umd.edu}.}
\and
Eric Mazumdar\thanks{CMS, Caltech. Email: \href{mailto:mazumdar@caltech.edu}{mazumdar@caltech.edu}.}
\and
Asuman Ozdaglar\thanks{EECS, MIT. Email: \href{mailto:asuman@mit.edu}{asuman@mit.edu}.}
\and
Adam Wierman\thanks{CMS, Caltech. Email: \href{mailto:adamw@caltech.edu}{adamw@caltech.edu}.}
}

\date{}
\begin{document}

\maketitle
\vspace{-0.5 in}
\begin{abstract}
We present a finite-sample analysis of decentralized learning in two-player zero-sum matrix games and stochastic games, with a focus on best-response-based learning algorithms. In matrix games, the learning algorithm is payoff-based and symmetric: each player updates its policy using only its own payoff observations, incrementally moving toward an estimated smoothed best response to the opponent's latest policy. For stochastic games, we build on this matrix-game primitive to develop a learning algorithm called value iteration with smoothed best response (VI-SBR), which combines smoothed-best-response learning in induced matrix games with a decentralized, model-free approximation of minimax value iteration. We establish finite-sample guarantees in both settings. For matrix games, our results imply a sample complexity of $\mathcal{O}(\epsilon^{-1})$ for finding an $\epsilon$-Nash distribution and, with explicit exploration, $\tilde{\mathcal{O}}(\epsilon^{-8})$ for finding an $\epsilon$-Nash equilibrium. For stochastic games, we prove that the exploration-enhanced VI-SBR algorithm achieves a sample complexity of $\tilde{\mathcal{O}}(\epsilon^{-8})$ for finding an $\epsilon$-Nash equilibrium. Technically, our analysis develops a coupled Lyapunov-drift framework. This framework simultaneously handles stochastic iterative algorithms with multiple interacting stochastic iterates, the non-zero-sum auxiliary games generated by independently updated value functions, and the time-inhomogeneous Markovian noise induced by time-varying policies. The resulting tools may be useful more broadly for analyzing learning algorithms with coupled stochastic iterates and nonstationary sampling processes.
\end{abstract}

\section{Introduction}
	Reinforcement learning has become an increasingly popular framework for solving large-scale sequential decision-making problems due to its solid theoretical foundation \citep{sutton2018reinforcement} and remarkable empirical successes \citep{mnih2013playing}. However, real-world applications, such as autonomous driving and robotics, often involve multiple decision-makers interacting in a common environment with possibly misaligned goals \citep{zhang2021multi}. In these scenarios, practical multi-agent reinforcement learning (MARL) algorithms are often extended directly from their single-agent counterparts, overlooking the adaptive strategies of multiple agents. Consequently, the resulting algorithms can be unreliable. For example, even in the game of Go, super-human AIs are susceptible to adversarial attacks \citep{wang2023adversarial}.
	
	To realize the practical potential of MARL, a growing body of literature seeks to provide theoretical insights into MARL and inform the design of efficient and provably convergent algorithms. Related work in this area can be broadly categorized into cooperative MARL, where agents share a common goal \citep{arslan2017decentralized, zhang2018fully, qu2020scalable, zhang2022global}, and competitive MARL, where agents have individual objectives \citep{littman1994markov, hu2003nash, daskalakis2020independent, bai2020provable, jin2021v}. While early work in this area focused on asymptotic convergence, there is increasing interest in understanding the finite-time/sample behavior of MARL algorithms. Compared with asymptotic analysis, finite-sample analysis provides more detailed theoretical insights into algorithmic behavior and can also guide implementation.
	
	In this paper, we consider the benchmark settings of two-player zero-sum matrix games and stochastic (Markov) games, and focus on best-response-based learning algorithms. As one of the most natural and fundamental classes of adaptive learning algorithms \citep{hofbauer2002global, hofbauer2005learning,hofbauer2006best}, these learning algorithms update each player's policy incrementally toward an estimated smoothed best response to its opponent's latest policy. Moreover, they are independent, requiring no explicit coordination between players during learning, and rational, in the sense that each player converges to the (smoothed) best response to the opponent when the opponent plays an (asymptotically) stationary policy \citep{bowling2001rational}. This is consistent with the setting of learning among self-interested players, where communication or coordination should not be imposed. Moreover, we study these learning algorithms in the challenging 
    setting of decentralized  learning, where each player observes only the state and its own realized payoff at each stage, without knowing the opponent's policy or actions. 
    
    From this perspective, our goal is to understand whether such canonical best-response-based learning algorithms can be implemented in a decentralized manner while retaining provable finite-sample guarantees. Our contributions are detailed as follows.

	\begin{itemize}
    \item \textbf{Two-Player Zero-Sum Matrix Games.}
    We begin with the matrix-game setting, which serves both as a stateless benchmark and as the building block for the stochastic-game setting. We analyze the best-response-based learning algorithm first proposed in \cite{leslie2003two}, in which both players update symmetrically using only their own realized payoff observations. We establish finite-sample bounds under both constant and diminishing stepsizes. These bounds imply a sample complexity of $\mathcal{O}(\epsilon^{-1})$ for finding an $\epsilon$-Nash distribution \citep{leslie2005individual}. Since a Nash distribution is generally different from a Nash equilibrium, this guarantee does not directly yield a sample complexity polynomial in $\epsilon^{-1}$ for finding an $\epsilon$-Nash equilibrium. We identify the lack of exploration as the key obstacle, introduce an exploration-encouraging variant, and show that it achieves a sample complexity of $\tilde{\mathcal{O}}(\epsilon^{-8})$ for finding an $\epsilon$-Nash equilibrium. To the best of our knowledge, for zero-sum matrix games, this is the first finite-sample guarantee for best-response-based learning algorithm that is simultaneously payoff-based, convergent, and symmetric between the two players.

    \item \textbf{Two-Player Zero-Sum Stochastic Games.}
    Building on the matrix-game analysis, we develop a learning algorithm named value iteration with smoothed best-response (VI-SBR) for stochastic games. The learning algorithm consists of two loops: an inner loop that runs the aforementioned smoothed best-response-based learning algorithm for an induced auxiliary matrix game, and an outer loop that performs a decentralized and model-free approximation of minimax value iteration. We establish finite-sample bounds for the VI-SBR algorithm. As in the matrix-game setting, however, the lack of exploration prevents these bounds from directly yielding a sample complexity polynomial in $\epsilon^{-1}$ for finding an $\epsilon$-Nash equilibrium. We therefore introduce an exploration-enhanced variant and prove that it achieves a sample complexity of $\tilde{\mathcal{O}}(\epsilon^{-8})$ for finding an $\epsilon$-Nash equilibrium. To the best of our knowledge, this is the first finite-sample analysis of best-response-based learning algorithm that is simultaneously decentralized, convergent, and rational for two-player zero-sum stochastic games.
\item \textbf{Technical Contributions.}
The key challenge in our analysis, in both the matrix-game and stochastic-game settings, is that they involve multiple sets of stochastic iterates updated in a coupled manner. To address this challenge, we develop a coupled Lyapunov-based approach: we construct a Lyapunov function for each set of stochastic iterates, establish the corresponding Lyapunov drift inequality, and then combine these drift inequalities to derive convergence rates. Several components are specific to stochastic games: a generalized regularized Nash-gap Lyapunov function for auxiliary matrix games that are not exactly zero-sum during learning, a Lyapunov function that tracks the non-zero-sum error induced by independently maintained value functions, and a treatment of time-inhomogeneous Markovian noise generated by time-varying policies. A more detailed discussion of the challenges and proof techniques is provided in Section \ref{sec:proof_outline}. More broadly, the coupled Lyapunov-drift framework developed in this paper offers a systematic approach to analyzing learning algorithms in which multiple stochastic iterates evolve jointly under nonstationary, policy-dependent sampling processes.
	\end{itemize}
	
	This paper builds upon and significantly extends the conference version \citep{chen2024game_tabular}, which provided a finite-sample analysis of decentralized learning algorithms for two-player zero-sum matrix and stochastic games. A major limitation of \citep{chen2024game_tabular} is that its bounds do not yield polynomial dependence on $\epsilon^{-1}$ for finding an $\epsilon$-Nash equilibrium. We address this limitation by identifying the lack of exploration as the source of the exponential dependence and by introducing exploration-enhanced variants whose sample complexities depend polynomially on $\epsilon^{-1}$.

\subsection{Related Literature}\label{sec:related_work}
In this subsection, we review the most relevant work on reinforcement learning for matrix and stochastic games.

\textbf{Zero-Sum Matrix Games.} 
Fictitious play (FP) \citep{brown1951iterative,robinson1951iterative} is one of the earliest methods of independent learning in zero-sum matrix games and, together with its smoothed variant \citep{ref:Fudenberg93,hofbauer2002global}, can be analyzed by the associated ordinary differential equation/inclusion of \emph{(smoothed) best-response dynamics}  
using a Lyapunov approach \citep{hofbauer2005learning}. Notably, the Lyapunov function used in such an analysis is the regularized Nash gap \citep{shamma2004unified, hofbauer2005learning, leslie2005individual}, a variant of which is also employed in our analysis framework. 
To adapt the learning algorithm to the payoff-based setting, \cite{leslie2003two} developed a two-timescale reinforcement learning algorithm with asymptotic convergence guarantees. The learning algorithm in \cite{leslie2003two} serves as the basis for ours in the matrix-game setting. More broadly, no-regret learning algorithms, which have been extensively studied in online learning, can also serve as independent learning algorithms for matrix games \citep{cesa2006prediction}; these algorithms are convergent, rational \cite{bowling2001rational}, and implemented symmetrically by the players. However, regret bounds generally do not imply convergence bounds measured by the last iterate, as we establish in this work.

For finite-sample analysis in settings with payoff feedback, the closest work to ours is \cite{cai2023uncoupled}, which was concurrent with the conference version of this work. Specifically, \cite{cai2023uncoupled} studied payoff-based learning algorithms based on online mirror descent, and established high-probability and anytime last-iterate convergence bounds, which imply a sample complexity of $\tilde{\mathcal{O}}(\epsilon^{-8})$. 
Subsequently, more recent papers \citep{cai2026average,fiegel2025harderpath,fiegel2026optimal} have enhanced the last-iterate convergence of such mirror-descent-based algorithms, improving the sample complexity to $\tilde{\mathcal{O}}(\epsilon^{-5})$ and $\tilde{\mathcal{O}}(\epsilon^{-4})$, respectively. In comparison, we provide a finite-sample analysis for the natural best-response-based learning algorithm, which also serves as the basis for our learning algorithm for stochastic games studied later. Since their learning algorithms are based on mirror descent whereas ours are based on smoothed best responses, the exact sample complexity bounds are not directly comparable. For such a best-response-based learning algorithm, \cite{faizal2024finite} established a sample complexity of $\tilde{\mathcal{O}}(\epsilon^{-8-\nu})$ for some $\nu>0$; in comparison, our bound is $\tilde{\mathcal{O}}(\epsilon^{-8})$.

\textbf{Zero-Sum Stochastic Games.} 
For zero-sum stochastic games, establishing non-asymptotic, finite-sample analyses has received increasing attention recently, 
see e.g.,  \citep{bai2020provable,bai2020near,liu2020sharp,xie2020learning,jin2021v,songcan,mao2022improving}. Most of these studies, however,  focused on the finite-horizon setting with online exploration, and conducted  regret analysis, which differs from our last-iterate finite-sample analysis under the \emph{stochastic approximation} framework, for the infinite-horizon discounted  setting. Additionally, due to the finite-horizon nature of the setting, these algorithms are episodic and  not best-response-type independent learning dynamics  that are run for infinitely long, 
as a non-equilibrating adaptation process. 
Finite-sample complexity has also been established  particularly for the 
policy-gradient methods \citep{daskalakis2020independent,zhao2021provably,zhang2021derivative,alacaoglu2022natural}. However, 
these methods are \emph{asymmetric} among players, requiring them to update their policies on different timescales, 
which thus enforces explicit coordination across players. 

When it comes to symmetric and  independent/decentralized learning, 
\cite{wei2021last, chenzy2021sample,cai2023uncoupled} are the recent studies that provided  
non-asymptotic analyses. In comparison, the learning algorithms in \cite{wei2021last, chenzy2021sample, cai2023uncoupled} are variants of mirror-descent-based methods (e.g., optimistic gradient descent/ascent and extragradient descent/ascent), whereas our focus is on analyzing the natural  learning dynamics based on smoothed-best-response. Additionally, the methods in \cite{wei2021last, chenzy2021sample} require some coordination between players during the sampling process. Best-response-type independent learning for zero-sum stochastic games has also been explored recently in 
\cite{sayin2021decentralized, sayin2022fictitious, baudinsmooth, baudin2022fictitious}. 
Yet, these works only established  asymptotic convergence guarantees,  which motivated the present work. 

\textbf{Organization.}
The rest of this paper is organized as follows. In Section \ref{sec:bandit}, we focus on zero-sum matrix games and present payoff- and best-response-based reinforcement learning algorithms, together with their finite-sample guarantees. In Section \ref{sec:stochastic_game}, we extend the algorithmic idea to stochastic games and present the VI-SBR learning algorithm, together with its finite-sample guarantees. In Section \ref{sec:proof_outline}, we present the proof of one of our main results, discuss the key technical challenges, and explain how we overcome them using a coupled Lyapunov-based approach. The detailed proofs of the other theoretical results are provided in the appendix. We conclude the paper in Section \ref{sec:conclusion}.

\section{Two-Player Zero-Sum Matrix Games}\label{sec:bandit}
For $i\in \{1,2\}$, let $\mathcal{A}^i$ be the finite action space of player $i$, and denote $m_i=|\mathcal{A}^i|$. Let $R_i\in\mathbb{R}^{m_i\times m_{-i}}$ be the payoff matrix of player $i$, where $-i$ denotes the opponent of player $i$. Note that $R_1+R_2^\top =0$ in the zero-sum setting. Since there are finitely many actions, we assume without loss of generality that $\max_{a^1,a^2}|R_1(a^1,a^2)|\leq 1$. Furthermore, we denote $m=\max(m_1,m_2)$. The decision variables here are the policies $\pi^i\in \Delta(\mathcal{A}^i)$, $i\in \{1,2\}$, where $\Delta(\mathcal{A}^i)$ denotes the probability simplex supported on $\mathcal{A}^i$. Given a joint policy $(\pi^1,\pi^2)$, the expected reward received by player $i$ is $\mathbb{E}_{A^i\sim \pi^i(\cdot),A^{-i}\sim \pi^{-i}(\cdot)}[R_i(A^i,A^{-i})]=(\pi^i)^\top R_i\pi^{-i}$.
Both players aim to maximize their rewards against their opponents. 

Unlike the single-player setting, since the performance of player $i$'s policy depends on its opponent $-i$'s policy, there is, in general, no universal optimal policy. Instead, we use the \emph{Nash gap} and also the \emph{regularized Nash gap} as measurements of the performance of the learning algorithm, as formally defined below. 
\begin{definition}
Given a joint policy $\pi=(\pi^1,\pi^2)$, the Nash gap $NG(\cdot,\cdot)$ is defined as 
\begin{align*}
    \text{NG}(\pi^1,\pi^2)=\sum_{i=1,2}\max_{\mu^i\in\Delta(\mathcal{A}^i)}(\mu^i-\pi^i)^\top R_i\pi^{-i}.
\end{align*}
\end{definition}

By definition, $\text{NG}(\pi^1,\pi^2)=0$ if and only if $(\pi^1,\pi^2)$ is a Nash equilibrium of the matrix game, in which no player has an incentive to deviate from its current policy. Note that the Nash equilibrium need not be unique.
\begin{definition}
Given a joint policy $\pi=(\pi^1,\pi^2)$ and $\tau>0$, the entropy-regularized Nash gap $\text{NG}_\tau(\pi^1,\pi^2)$ is defined as 
\begin{align*}
    \text{NG}_\tau(\pi^1,\pi^2)=\sum_{i=1,2}\left\{\max_{\mu^i\in\Delta(\mathcal{A}^i)}(\mu^i-\pi^i)^\top R_i\pi^{-i}+\tau \nu(\mu^i)-\tau \nu(\pi^i)\right\},
\end{align*}
where $\nu(\cdot)$ is the Shannon entropy defined as $\nu(\mu^i)=-\sum_{a^i\in\mathcal{A}^i}\mu^i(a^i)\log(\mu^i(a^i))$ for $i\in \{1,2\}$.
\end{definition}

A joint policy $(\pi^1,\pi^2)$ satisfying $\text{NG}_\tau(\pi^1,\pi^2)=0$ is called the Nash distribution \citep{leslie2005individual} or the quantal response equilibrium \citep{mckelvey1995quantal}, which, unlike the Nash equilibrium, is unique in a two-player zero-sum matrix game. Note that, as the parameter $\tau$ approaches $0$, the corresponding Nash distribution approximates a Nash equilibrium \citep{govindan2003short}.

\subsection{Algorithm}\label{subsec:bandit_algorithm}

We start by presenting in Algorithm \ref{algo:matrix_fast} the payoff-based reinforcement learning algorithm for zero-sum matrix games firstly  proposed in \cite{leslie2003two}. Given $\tau > 0$ and $i \in \{1, 2\}$, we use $\sigma_\tau : \mathbb{R}^{m_i} \to \mathbb{R}^{m_i}$ to represent the softmax function with temperature $\tau$, i.e.,
\begin{align*}
    [\sigma_\tau(q^i)](a^i)= \frac{\exp(q^i(a^i)/\tau)}{\sum_{\tilde{a}^i \in \mathcal{A}^i} \exp(q^i(\tilde{a}^i)/\tau)}
\end{align*}
for all $a^i \in \mathcal{A}^i$ and $q^i \in \mathbb{R}^{m_i}$.

\begin{algorithm}[ht]\caption{Decentralized  Learning in Zero-Sum Matrix Games (of Player $i$)}\label{algo:matrix_fast}
	\begin{algorithmic}[1]
		\STATE \textbf{Input:} Integer $K$, initializations $q_0^i=0\in\mathbb{R}^{m_i}$ and $\pi_0^i = \text{Unif}(\mathcal{A}^i)$
		\FOR{$k=0,1,\cdots,K-1$}
		\STATE $\pi_{k+1}^i=\pi_k^i+\beta_k(\sigma_\tau(q_k^i)-\pi_k^i)$
		\STATE Play
		$A_k^i\sim \pi_{k+1}^i(\cdot)$ (against $A_k^{-i}$), and receive reward $R_i(A_k^i,A_k^{-i})$
		\STATE $q_{k+1}^i(a^i)=q_k^i(a^i)+\alpha_k\mathds{1}_{\{a^i=A_k^i\}} \left(R_i(A_k^i,A_k^{-i})-q_k^i(A_k^i)\right)$ for all $a^i\in\mathcal{A}^i$
		\ENDFOR
	\end{algorithmic}
\end{algorithm} 

We next provide a detailed illustration of Algorithm \ref{algo:matrix_fast}, which also motivates our learning algorithm for stochastic games in Section \ref{sec:stochastic_game}. At a high level, Algorithm \ref{algo:matrix_fast} can be viewed as a discrete and smoothed variant of the best-response dynamics, where each player constructs an approximation of the smoothed best response to its opponent's policy using the $q$-function. The update equation for the $q$-function is in the spirit of the TD-learning method in reinforcement learning \citep{sutton1988learning}.  

\textbf{The Policy Update.}
To understand the update equation for the policies (cf. Algorithm \ref{algo:matrix_fast}, Line $3$), consider the discrete version of the smoothed best-response dynamics:
\begin{align}\label{eq:FP}              \pi_{k+1}^i=\pi_k^i+\beta_k(\sigma_\tau(R_i\pi_k^{-i})-\pi_k^i),\quad i\in \{1,2\}.
\end{align}
In (\ref{eq:FP}), each player updates its policy $\pi_k^i$ incrementally towards the smoothed best response to its opponent's current policy. While (\ref{eq:FP}) provably converges for two-player zero-sum matrix games \citep{hofbauer2006best}, implementing it requires player $i$ to compute $\sigma_\tau(R_i\pi_k^{-i})$. Note that $\sigma_\tau(R_i\pi_k^{-i})$ involves the  knowledge of the opponent's policy and the reward matrix, both of which cannot be accessed in decentralized and model-free reinforcement learning. This leads to the update equation for the $q$-function, which estimates the quantity $R_i\pi_k^{-i}$ needed for implementing (\ref{eq:FP}).

\textbf{The $q$-Function Update.}
Suppose for now that we are given a stationary joint policy $(\pi^1,\pi^2)$. Fixing $i\in \{1,2\}$, the problem of player $i$ estimating $R_i\pi^{-i}$ can be viewed as a policy evaluation problem in reinforcement learning, which can be solved by TD-learning \citep{sutton1988learning}. 
Specifically, the two players repeatedly play the matrix game with the joint policy $(\pi^1,\pi^2)$ and produce a sequence of joint actions $\{(A_k^1,A_k^2)\}$. Then,  player $i$ estimates $R_i\pi^{-i}$ iteratively through the following algorithm: 
\begin{align}\label{eq:q_bandit}
	q_{k+1}^i(a^i)=q_k^i(a^i)+\alpha_k \mathds{1}_{\{a^i=A_k^i\}}(R_i(A_k^i,A_k^{-i})-q_k^i(A_k^i)),\quad \forall\,a^i\in\mathcal{A}^i,
\end{align}
with an arbitrary initialization $q_0^i\in\mathbb{R}^{m_i}$, where $\alpha_k>0$ is the stepsize. To understand (\ref{eq:q_bandit}), suppose that $q_k^i$ converges to some $\Bar{q}^i$. Then, the update equation (\ref{eq:q_bandit}) should be ``stationary'' at the limit point $\Bar{q}^i$ in the sense that $\mathbb{E}_{A^i\sim \pi^i(\cdot),A^{-i}\sim \pi^{-i}(\cdot)}[\mathds{1}_{\{a^i=A^i\}}(R_i(A^i,A^{-i})-\Bar{q}^i(A^i))]=0$
for all $a^i\in\mathcal{A}^i$,
which implies $\Bar{q}^i=R_i\pi^{-i}$, as desired. Although we motivated (\ref{eq:q_bandit}) assuming the joint policy $(\pi^1,\pi^2)$ is stationary, the joint policy 
$(\pi_k^1,\pi_k^2)$ from (\ref{eq:FP}) is time-varying. A natural approach to address this issue is to make sure that the policies evolve much more slowly compared to that of the $q$-functions, so that $\pi_k$ is close to being stationary from the perspective of $q_k^i$. This can be achieved by making $\beta_k\ll \alpha_k$, where $\beta_k$ is the stepsize for updating the policies and $\alpha_k$ is the stepsize for updating the $q$-functions. When Algorithm \ref{algo:matrix_fast} was first proposed \cite{leslie2003two}, this was achieved by requiring $\lim_{k\rightarrow\infty}\beta_k/\alpha_k=0$, making Algorithm \ref{algo:matrix_fast} a \emph{two-timescale}  algorithm. In this work, we propose to update $\pi_k^i$ and $q_k^i$ on a \emph{single timescale} but with only a multiplicative constant difference in their stepsizes, i.e., $\beta_k=c_{\alpha,\beta}\alpha_k$ for some sufficiently small $c_{\alpha,\beta}\in (0,1)$.  

\subsection{Finite-Sample Analysis Measured by the Regularized Nash Gap}\label{sec:bandit_analysis}
To study Algorithm \ref{algo:matrix_fast}, we consider either constant stepsizes, i.e., $\alpha_k\equiv \alpha$ and $\beta_k\equiv \beta$, or harmonically diminishing stepsizes, i.e., $\alpha_k=\alpha/(k+h)$ and $\beta_k=\beta/(k+h)$, where $h\geq 0$ is a tunable parameter. In either case, we ensure that $\alpha_k,\beta_k\in (0,1)$ for all $k$ and $\beta=c_{\alpha,\beta}\alpha$ so that the algorithm operates on a single timescale. We begin by stating the requirements for choosing the stepsizes.

\begin{condition}\label{con:stepsize_matrix}
We choose $\tau \leq 1$ and $c_{\alpha,\beta} \leq \min\{\tau \ell_\tau^3/32, \ell_\tau \tau^3/(128m^2)\}$, where $\ell_\tau:=[(m-1)\exp(2/\tau)+1]^{-1}$. 
\end{condition}

\begin{remark}\label{remark:exp_bound}
     The parameter $\ell_\tau$ plays an important role in our analysis, as it captures the exploration capability of Algorithm \ref{algo:matrix_fast}. Specifically, we show that $\min_{a^i\in\mathcal{A}^i}\pi_k^i(a^i)\geq \ell_\tau$ for all $k\geq 0$ (cf. Lemma \ref{le:boundedness_matrix}). Due to the exponential structure of softmax policies, the parameter $\ell_\tau$ is itself an exponential function of the temperature parameter $\tau$.
\end{remark}

We next state the finite-sample bounds of Algorithm \ref{algo:matrix_fast}. The proof of the following theorem is presented in Appendix \ref{ap:proof_matrix_game}.
\begin{theorem}\label{thm:matrix_fast}
    Suppose that both players follow Algorithm \ref{algo:matrix_fast}. 
    \begin{enumerate}[(1)]
     \item When using constant stepsizes (i.e., $\alpha_k\equiv \alpha$ and $\beta_k\equiv \beta$) that satisfy Condition \ref{con:stepsize_matrix}, we have
     \begin{align*}
    \mathbb{E}[\text{NG}_\tau(\pi_K^1,\pi_K^2)]\leq B_{\text{in}}\left(1-\frac{\beta}{4}\right)^K+8L_\tau\beta+\frac{64\alpha}{c_{\alpha,\beta}},
\end{align*}
where $B_{\text{in}}:=4+2\tau \log(m)+2m$ and $L_\tau:=\tau/\ell_\tau+m^2/\tau$.
\item When using diminishing stepsizes of the form $\alpha_k=\alpha/(k+h)$ and $\beta_k=\beta/(k+h)$, by choosing $\beta>4$ and $h\geq 0$ such that Condition \ref{con:stepsize_matrix} is satisfied, we have
\begin{align*}
    \mathbb{E}[\text{NG}_\tau(\pi_K^1,\pi_K^2)]\leq\,& B_{\text{in}}\left(\frac{h}{K+h}\right)^{\beta/4}+\left(64e L_\tau \beta +\frac{512e \alpha}{c_{\alpha,\beta}}\right)\frac{1 }{K+h}.
\end{align*}
\normalsize
 \end{enumerate}
\end{theorem}

The convergence bounds in Theorem \ref{thm:matrix_fast} are qualitatively consistent with the existing results on the finite-sample analysis of general stochastic approximation algorithms \citep{lan2020first,bottou2018optimization,srikant2019finite,chen2021finite,bhandari2018finite,zhang2021global}. Specifically, when using constant stepsizes, the bound consists of a geometrically decaying term (also referred to as the optimization error) and a constant term (also referred to as the statistical error) that are proportional to the stepsizes. When using diminishing stepsizes with suitable  hyperparameters, both  errors can achieve an $\mathcal{O}(1/K)$ rate of convergence. 

Although Theorem \ref{thm:matrix_fast} is stated in terms of the expectation of the regularized Nash gap, it implies the mean-square convergence of the joint policy $(\pi_K^1,\pi_K^2)$. To see this, note that the regularized Nash gap $\text{NG}_\tau(\pi^1,\pi^2)$ is a $\tau$-strongly convex function (see Lemma \ref{le:properties_Lyapunov_main} for a proof), the unique minimizer of which, denote by  $(\pi_{*,\tau}^1,\pi_{*,\tau}^2)$, is the Nash distribution. Therefore, using the quadratic growth property of strongly convex functions, we have $\text{NG}_\tau(\pi_k^1,\pi_k^2)
    \geq \frac{\tau}{2}(\|\pi_k^1-\pi_{*,\tau}^1\|_2^2+\|\pi_k^2-\pi_{*,\tau}^2\|_2^2)$.
As a result, up to a constant multiplicative factor, the convergence bound for $\mathbb{E}[\text{NG}_\tau(\pi_k^1,\pi_k^2)]$ directly implies a convergence bound of $\mathbb{E}[\|\pi_k^1-\pi_{*,\tau}^1\|_2^2]+\mathbb{E}[\|\pi_k^2-\pi_{*,\tau}^2\|_2^2]$.

Based on Theorem \ref{thm:matrix_fast}, we next derive the sample complexity of Algorithm \ref{algo:matrix_fast}, measured by the regularized Nash gap, in the following corollary, whose proof is presented in Appendix \ref{ap:pf:sc_matrix_fast}.
\begin{corollary}\label{co:sample_matrix_matrix_fast}
    Given $\epsilon>0$, to achieve $\mathbb{E}[\text{NG}_\tau(\pi_K^1,\pi_K^2)]\leq \epsilon$ with Algorithm \ref{algo:matrix_fast}, the sample complexity is $\mathcal{O}(L_\tau c_{\alpha,\beta}^{-2}\epsilon^{-1}\log(\epsilon^{-1}))$.
\end{corollary}

To the best of our knowledge, Theorem \ref{thm:matrix_fast} and Corollary \ref{co:sample_matrix_matrix_fast} provide the first finite-sample analysis of Algorithm \ref{algo:matrix_fast}, first proposed in \citep{leslie2003two}. Importantly, using only realized payoff feedback, we establish a sample complexity of $\mathcal{O}(\epsilon^{-1})$ for finding an $\epsilon$-Nash distribution. 

\subsection{Finite-Sample Analysis Measured by the Nash Gap}
Although Theorem \ref{thm:matrix_fast} shows that Algorithm \ref{algo:matrix_fast} achieves a fast $\mathcal{O}(1/K)$ rate of convergence measured by the regularized Nash gap, 
this result does not directly translate to a fast convergence measured by the Nash gap. To illustrate the difference, note that the following bound holds:  
\begin{align}\label{eq:NG_RNG_main}
    \text{NG}(\pi^1,\pi^2)\leq  \text{NG}_\tau(\pi^1,\pi^2)+\underbrace{2\tau \log(m)}_{\text{Smoothing Bias}},\quad \forall\,(\pi^1,\pi^2),
\end{align}
where the second term can be viewed as the bias due to using the smoothed best response. We have conducted numerical experiments in Appendix \ref{ap:numerical} to demonstrate that such a smoothing bias is, in general, not removable. Combining  (\ref{eq:NG_RNG_main}) with Corollary \ref{co:sample_matrix_matrix_fast}, we have the following result, whose proof is presented in Appendix \ref{pf:co:sample_complexity_matrix_exponential}.
\begin{corollary}\label{co:sample_complexity_matrix_exponential}
    Given $\epsilon>0$, to achieve $\mathbb{E}[\text{NG}(\pi_K^1,\pi_K^2)]\leq \epsilon$ with Algorithm \ref{algo:matrix_fast}, the sample complexity is $\mathcal{O}\left(\frac{\log(1/\epsilon)}{\epsilon^4f(\epsilon)^3 \min(f(\epsilon)^4,\epsilon^4)}\right)$,
    where $f(\epsilon)=[(m-1)\exp(8\log(m)/\epsilon)+1]^{-1}$.
\end{corollary}

Importantly, the function $f(\epsilon)$ is exponentially small in $\epsilon$. Therefore, when using the Nash gap as the performance metric, Algorithm \ref{algo:matrix_fast} does not admit a sample complexity polynomial in $\epsilon^{-1}$. Tracing back to the finite-sample bound of Algorithm \ref{algo:matrix_fast} in Theorem \ref{thm:matrix_fast}, the key reason is that the constant $\ell_\tau$, which we establish in Lemma \ref{le:boundedness_matrix} as a lower bound for $\min_{0\leq k\leq K}\min_{a^i}\pi_k^i(a^i)$, is exponentially small in $\tau$.  
To address this issue, recall that the constant $\ell_\tau$ captures the exploration capability of Algorithm \ref{algo:matrix_fast}; see Remark \ref{remark:exp_bound}. 
We next introduce a more exploration-encouraging variant of Algorithm \ref{algo:matrix_fast} that has a sample complexity polynomial in $\epsilon^{-1}$ for achieving $\mathbb{E}[\text{NG}(\pi_K^1,\pi_K^2)]\leq \epsilon$. 
Note that our goal is to make minimal modification of the algorithm, keeping the (natural) \emph{best-response}-type update rules.

Given $\Bar{\epsilon}\in [0,1]$ and $\tau>0$, let $\sm:\mathbb{R}^{m_i}\to \Delta(\mathcal{A}^i)$ be defined as \begin{align*}
\sm(q^i)
=
\Bar{\epsilon} \cdot \text{Unif}(\mathcal{A}^i)
+
(1-\Bar{\epsilon})\cdot \sigma_\tau(q^i),\quad \forall\,q^i\in\mathbb{R}^{m_i},i\in \{1,2\},
\end{align*}
where $\text{Unif}(\mathcal{A}^i)$ denotes the uniform distribution supported on $\mathcal{A}^i$. With this modification, we can explicitly control the lower bound of the components of $\sm(\cdot)$ through the tunable parameter $\Bar{\epsilon}$, thus preventing the components from being exponentially small in $\tau$. 
By replacing $\sigma_\tau(\cdot)$ with $\sm(\cdot)$ in Algorithm \ref{algo:matrix_fast}, we have Algorithm \ref{algo:matrix_slow} presented as follows.

\begin{algorithm}[ht]\caption{Algorithm \ref{algo:matrix_fast} with $\Bar{\epsilon}$-Exploration (of Player $i$)}\label{algo:matrix_slow}
\begin{algorithmic}[1]
\STATE \textbf{Input:} Integer $K$, initializations $q_0^i=0\in\mathbb{R}^{m_i}$ and $\pi_0^i = \text{Unif}(\mathcal{A}^i)$
		\FOR{$k=0,1,\cdots,K-1$}
		\STATE $\pi_{k+1}^i=\pi_k^i+\beta_k(\sm(q_k^i)-\pi_k^i)$
		\STATE Play
		$A_k^i\sim \pi_{k+1}^i(\cdot)$ (against $A_k^{-i}$), and receive reward $R_i(A_k^i,A_k^{-i})$
		\STATE $q_{k+1}^i(a^i)=q_k^i(a^i)+\alpha_k\mathds{1}_{\{a^i=A_k^i\}} \left(R_i(A_k^i,A_k^{-i})-q_k^i(A_k^i)\right)$ for all $a^i\in\mathcal{A}^i$
		\ENDFOR
\end{algorithmic}
\end{algorithm} 

We next present the last-iterate finite-sample analysis of Algorithm \ref{algo:matrix_slow} measured by the Nash gap. For ease of presentation, we only consider the case with constant stepsizes, which can be directly extended to that with  diminishing stepsizes. Let 
\begin{align*}
    \ell_{\tau,\Bar{\epsilon}}:=\frac{\Bar{\epsilon}}{m}+\frac{(1-\Bar{\epsilon})}{(m-1)\exp(2/\tau)+1},
\end{align*}
which, analogous to $\ell_\tau$ in Condition \ref{con:stepsize_matrix}, is a uniform lower bound of the policies generated by Algorithm \ref{algo:matrix_slow} (cf. Lemma \ref{le:boundedness_matrix2}).
\begin{theorem}\label{thm:matrix_slow}
	Suppose that both players follow Algorithm \ref{algo:matrix_slow}. When choosing 
    $c_{\alpha,\beta}=\beta/\alpha\leq \ell_{\tau,\Bar{\epsilon}}/2$, and $\Bar{\epsilon}=\tau$, we have
 \begin{align*}
    \mathbb{E}[\text{NG}(\pi_K^1,\pi_K^2)]\leq 20m^{1/2}K(1-\beta)^K+  \frac{2m\alpha}{\tau}+\frac{32m^{3/2}\sqrt{\alpha}}{\tau^{3/2}}+\frac{40\beta m^3}{\alpha \tau^2}+18\tau m.
\end{align*}
\end{theorem}

The proof of Theorem \ref{thm:matrix_slow} is presented in Appendix \ref{pf:thm:matrix_slow}.
In view of Theorem \ref{thm:matrix_slow}, we no longer have any problem-dependent constants that are exponential in $\tau$. Based on Theorem \ref{thm:matrix_slow}, we next present the sample complexity of Algorithm \ref{algo:matrix_slow}, whose proof is presented in Appendix \ref{pf:co:sc_matrix_slow}.
\begin{corollary}\label{co:sc_matrix_slow}
    Given $\epsilon>0$, to achieve $\mathbb{E}[NG(\pi_K^1,\pi_K^2)]\leq \epsilon$ with Algorithm \ref{algo:matrix_slow}, the sample complexity is $\Tilde{\mathcal{O}}(\epsilon^{-8})$.
\end{corollary}

We briefly compare Corollary \ref{co:sc_matrix_slow} with existing finite-sample results for payoff-based learning in zero-sum matrix games; see the related work section for a more detailed discussion. Existing results largely fall into two types: online-mirror-descent-based algorithms \citep{cai2023uncoupled,cai2026average,fiegel2025harderpath,fiegel2026optimal}, which are different from the smoothed-best-response-based algorithm studied here, and smoothed-best-response-based algorithms \citep{faizal2024finite}, for which the best available sample complexity is $\tilde{\mathcal{O}}(\epsilon^{-8-\nu})$ for some $\nu>0$. In comparison, Corollary \ref{co:sc_matrix_slow} establishes a sample complexity of $\tilde{\mathcal{O}}(\epsilon^{-8})$.

\section{Two-Player Zero-Sum Stochastic Games}\label{sec:stochastic_game}
Consider an infinite-horizon discounted two-player zero-sum stochastic game $\mathcal{M}=(\mathcal{S},\mathcal{A}^1,\mathcal{A}^2,p,R_1,R_2,\gamma)$, where $\mathcal{S}$ is a finite state space, and $\mathcal{A}^1$ and $\mathcal{A}^2$ are the finite action spaces of players $1$ and $2$, respectively. We denote $n=|\mathcal{S}|$ and $m=\max(m_1,m_2)$, where $m_1=|\mathcal{A}^1|$ and $m_2=|\mathcal{A}^2|$. The transition probabilities are specified by $p$, where $p(s'\mid s,a^1,a^2)$ is the probability of transitioning to state $s'$ when player $1$ takes action $a^1$ and player $2$ takes action $a^2$ simultaneously at state $s$. The reward functions are given by $R_1:\mathcal{S}\times\mathcal{A}^1\times\mathcal{A}^2\to \mathbb{R}$ and $R_2:\mathcal{S}\times\mathcal{A}^2\times \mathcal{A}^1\to \mathbb{R}$ for players $1$ and $2$, respectively, and $\gamma\in (0,1)$ is the discount factor. The zero-sum condition requires that $R_1(s,a^1,a^2)+R_2(s,a^2,a^1)=0$ for all $(s,a^1,a^2)$. As in the matrix-game setting, we assume without loss of generality that $\max_{s,a^1,a^2}|R_1(s,a^1,a^2)|\leq 1$. 

Given a joint (Markov  stationary) policy $\pi=(\pi^1,\pi^2)$, where $\pi^i:\mathcal{S}\to\Delta(\mathcal{A}^i)$, $i\in \{1,2\}$, we define the local $q$-function $q_\pi^i\in\mathbb{R}^{nm_i}$ of player $i$ as 
\begin{align*}
    q_\pi^i(s,a^i)=\mathbb{E}_\pi\left[\sum_{k=0}^\infty\gamma^kR_i(S_k,A_k^i,A_k^{-i})\;\middle|\; S_0=s,A_0^i=a^i\right],\quad \forall\,(s,a^i),
\end{align*}
where we use the notation $\mathbb{E}_\pi[\,\cdot\,]$ to indicate that the actions are chosen according to the joint policy $\pi$. Furthermore, we define the (state) value function $v_\pi^i\in\mathbb{R}^{n}$ as $v_\pi^i(s)=\mathbb{E}_{A^i\sim \pi^i(\cdot|s)}[q_\pi^i(s,A^i)]$ for all $s\in\mathcal{S}$,
and the utility function $U^i(\pi^i,\pi^{-i})\in\mathbb{R}$ as $U^i(\pi^i,\pi^{-i})=\mathbb{E}_{S\sim p_o}[v^i_\pi(S)]$,
where $p_o\in\Delta(\mathcal{S})$ is a fixed initial distribution on the states. 

The Nash gap in the case of stochastic games is defined in the following.
\begin{definition}
Given a joint policy $\pi=(\pi^1,\pi^2)$, the Nash gap $\text{NG}(\cdot,\cdot)$ is defined as 
\begin{align*}
    \text{NG}(\pi^1,\pi^2)=\sum_{i=1,2}\left(\max_{\hat{\pi}^i}U^i(\hat{\pi}^i,\pi^{-i})- U^i(\pi^i,\pi^{-i})\right).
\end{align*}
\end{definition}

A joint policy $\pi=(\pi^1,\pi^2)$ satisfying $\text{NG}(\pi^1,\pi^2)=0$ is then called a Nash equilibrium\footnote{Throughout the paper, we focus on such a notion of \emph{Markov stationary Nash equilibrium}, and will refer to it as \emph{Nash equilibrium}  for short.}.

\textbf{Additional Notation.} Since the stochastic game we study has a finite state-action space, the policy, local $q$-function, and value function can all be viewed as vectors. As we will frequently work with vectors in $\mathbb{R}^{nm_i}$, $\mathbb{R}^{nm_{-i}}$, and $\mathbb{R}^{nm_im_{-i}}$, where $i\in \{1,2\}$, to simplify notation, for any $x\in\mathbb{R}^{nm_im_{-i}}$ (e.g., $x$ may represent a joint policy), we use $x(s)$ to denote the $m_i\times m_{-i}$ matrix whose $(a^i,a^{-i})$-th entry is $x(s,a^i,a^{-i})$. For any $y\in\mathbb{R}^{nm_i}$ (e.g., $y$ may represent the local $q$-function or policy of player $i$), we use $y(s)$ to denote the $m_i$-dimensional vector whose $a^i$-th entry is $y(s,a^i)$. 

\subsection{Value Iteration with Smoothed Best-Response Learning}\label{subsec:Markov_Algorithm}
Our learning algorithm for stochastic games (cf. Algorithm \ref{algo:stochastic_game}) builds on the matrix-game algorithm studied in Section \ref{subsec:bandit_algorithm}, with the additional incorporation of minimax value iteration, a well-known approach for solving two-player zero-sum stochastic games \citep{shapley1953stochastic}.

To motivate the learning algorithm, we first introduce minimax value iteration. For $i\in \{1,2\}$, let $\mathcal{T}^i:\mathbb{R}^{n}\to\mathbb{R}^{nm_im_{-i}}$ be an operator defined as 
\begin{align*}
    [\mathcal{T}^i(v)](s,a^i,a^{-i})= R_i(s,a^i,a^{-i})+\gamma\mathbb{E}\left[ v(S_1)\mid S_0=s,A_0^i=a^i,A_0^{-i}=a^{-i}\right]
\end{align*}
for all $(s,a^i,a^{-i})$ and $v\in\mathbb{R}^{n}$. Given an $m_i\times m_{-i}$ matrix $X_i$, we define  $\textit{val}^i:\mathbb{R}^{m_i\times m_{-i}}\to\mathbb{R}$ as 
\begin{align*} 
	\textit{val}^i(X_i)=\max_{\mu^i\in\Delta(\mathcal{A}^i)}\min_{\mu^{-i}\in\Delta(\mathcal{A}^{-i})}\{(\mu^i)^\top X_i\mu^{-i}\}
	=\min_{\mu^{-i}\in\Delta(\mathcal{A}^{-i})}\max_{\mu^i\in\Delta(\mathcal{A}^i)}\{(\mu^i)^\top X_i\mu^{-i}\}.
\end{align*}
Then, the minimax Bellman operator $\mathcal{B}^i:\mathbb{R}^{n}\to\mathbb{R}^{n}$ is defined as 
\begin{align*}
    [\mathcal{B}^i(v)](s)=\textit{val}^i([\mathcal{T}^i(v)](s)),\quad \forall\,s\in\mathcal{S},
\end{align*}
where $[\mathcal{T}^i(v)](s)$ is an $m_i\times m_{-i}$ matrix according to our notation. It is known that the operator $\mathcal{B}^i(\cdot)$ is a contraction mapping with respect to the $\ell_\infty$-norm \citep{shapley1953stochastic}, hence it admits a unique fixed point, which we denote by $v_*^i$.

A common approach for computing equilibria in zero-sum stochastic games is to first apply minimax value iteration, $v^i_{t+1}=\mathcal{B}^i(v_t^i)$, to obtain $v_*^i$  \citep{banach1922operations}, and then, for each state $s\in\mathcal{S}$, solve the matrix game
$\max_{\mu^i\in\Delta(\mathcal{A}^i)}\min_{\mu^{-i}\in\Delta(\mathcal{A}^{-i})}(\mu^i)^\top\mathcal{T}^i(v_*^i)(s)\mu^{-i}$.
However, implementing this procedure requires complete knowledge of the transition probabilities. Moreover, since its output is independent of the opponent's actual policy and instead assumes that the opponent always best responds, this procedure is better viewed as an equilibrium computation algorithm rather than a rational one in the sense of \cite{bowling2001rational}. 

To make minimax value iteration model-free and rational, let us first rewrite it as
\begin{subequations}\label{eq:minimax_reformulate3}
\begin{align}
\hat{v}(s)=\,&\max_{\mu^i\in\Delta(\mathcal{A}^i)}\min_{\mu^{-i}\in\Delta(\mathcal{A}^{-i})}(\mu^i)^\top \mathcal{T}^i(v_t^i)(s)\mu^{-i},\quad \forall\,s\in\mathcal{S}.\label{eq:minimax_reformulate3:1}\\
v_{t+1}^i=\,&\hat{v}.\label{eq:minimax_reformulate3:2}
\end{align}
\end{subequations}
In view of \eqref{eq:minimax_reformulate3}, minimax value iteration may be understood as a two-step procedure. For each state $s$, one first solves a zero-sum matrix game with payoff matrix $[\mathcal{T}^i(v_t^i)](s)$, and then updates the value of the game to $v_{t+1}^i(s)$. In light of Algorithm \ref{algo:matrix_fast}, we already know how to perform  payoff-based learning in matrix games. Thus, what remains to do is to combine Algorithm \ref{algo:matrix_fast} with (\ref{eq:minimax_reformulate3}). This leads to Algorithm \ref{algo:stochastic_game}.

\begin{algorithm}[ht]\caption{VI-SBR (of Player $i$)}\label{algo:stochastic_game} 
	\begin{algorithmic}[1]
		\STATE \textbf{Input:} Integers $K$ and $T$, initializations $v_0^i=0\in\mathbb{R}^{n}$, $q_{0,0}^i=0\in\mathbb{R}^{nm_i}$, and $\pi_{0,0}^i(s)=\text{Unif}(\mathcal{A}^i)$ for all $s\in\mathcal{S}$
		\FOR{$t=0,1,\cdots,T-1$}
		\FOR{$k=0,1,\cdots,K-1$}
		\STATE $\pi_{t,k+1}^i(s)=\pi_{t,k}^i(s)+\beta_k(\sigma_\tau(q_{t,k}^i(s))-\pi_{t,k}^i(s))$ for all $s\in\mathcal{S}$
		\STATE Play
		$A_k^i\sim \pi_{t,k+1}^i(\cdot|S_k)$ (against $A_k^{-i}$) and observe $S_{k+1}\sim p(\cdot\mid S_k,A_k^i,A_k^{-i})$
		\STATE $q_{t,k+1}^i(s,a^i)=q_{t,k}^i(s,a^i)+\alpha_k\mathds{1}_{\{(s,a^i)=(S_k,A_k^i)\}}(R_i(S_k,A_k^i,A_k^{-i})+\gamma v_t^i(S_{k+1})-$\\
        $q_{t,k}^i(S_k,A_k^i))$ for all $(s,a^i)$
		\ENDFOR
		\STATE $v_{t+1}^i(s)=\pi_{t,K}^i(s)^\top q_{t,K}^i(s)$ for all $s\in\mathcal{S}$
  \STATE Set $S_{0}=S_{K}$, $q_{t+1,0}^i=q_{t,K}^i$, and
$\pi_{t+1,0}^i=\pi_{t,K}^i$
		\ENDFOR
	\end{algorithmic}
\end{algorithm} 

For each state $s$, the inner loop of Algorithm \ref{algo:stochastic_game} is designed to learn the matrix game with payoff matrices $\mathcal{T}^1(v_t^1)(s)$ and $\mathcal{T}^2(v_t^2)(s)$ for each state $s\in\mathcal{S}$. However, in general, since $v_t^1$ and $v_t^2$ are independently maintained by players $1$ and $2$, we do not necessarily have $\mathcal{T}^1(v_t^1)(s,a^1,a^2)+\mathcal{T}^2(v_t^2)(s,a^2,a^1)=0$ for all $(a^1,a^2)$. As a result, the auxiliary matrix game (with payoff matrices $\mathcal{T}^1(v_t^1)(s)$ and $\mathcal{T}^2(v_t^2)(s)$) at state $s$) is not necessarily a \emph{zero-sum} matrix game, which presents a major challenge in the finite-sample analysis. Such a breaking of the  zero-sum structure was also observed as one key challange in the analyses of independent learning dynamics in \cite{sayin2021decentralized,sayin2022fictitious}. We will elaborate on this challenge and our technique to overcome it in more detail in Section \ref{sec:proof_outline}.

The outer loop of Algorithm \ref{algo:stochastic_game} is an approximation of the minimax value iteration. To see this, note that, ideally, we would synchronize $v_{t+1}^i(s)$ with $\pi_{t,K}^i(s)^\top \mathcal{T}^i(v_t^i)(s)\pi_{t,K}^{-i}(s)$, which is an approximation of $[\mathcal{B}^i(v)](s)=\textit{val}^i([\mathcal{T}^i(v_t^i)](s))$ by design of our inner loop. However, player $i$ has no access to $\pi_{t,K}^{-i}$ in our decentralized learning setting. Fortunately, the local $q$-function $q_{t,K}^i$ is precisely constructed to estimate $\mathcal{T}^i(v_t^i)(s)\pi_{t,K}^{-i}(s)$, as illustrated in Section \ref{subsec:bandit_algorithm}, which leads to the outer loop of Algorithm \ref{algo:stochastic_game}. In Algorithm \ref{algo:stochastic_game}, Line $9$, we set $S_0=S_K$ to ensure that Algorithm \ref{algo:stochastic_game} is driven by a single trajectory of Markovian samples. 

\subsection{Finite-Sample Analysis}\label{subsec:finite-sample-analysis}
We now state our main results for two-player zero-sum stochastic games, which is based on the following assumption.

\begin{assumption}\label{as:MC}
For every deterministic stationary policy pair $\pi=(\pi^1,\pi^2)$, the Markov chain $\{S_k\}$ induced by $\pi$ is irreducible and aperiodic.
\end{assumption}

Assumption \ref{as:MC} is imposed to ensure sufficient exploration during learning. Such exploration assumptions are standard in the literature, even for establishing the asymptotic convergence of single-agent algorithms
\citep{khodadadian2021finite,zou2019finite,tsitsiklis1994asynchronous,srikant2019finite}. We note that the aperiodicity assumption could be relaxable using recent approaches based on Poisson-equation decompositions of Markovian noise \citep{chen2026non,chandak2021concentration}. We leave this relaxation as a future direction. Importantly, although Assumption \ref{as:MC} concerns only deterministic policies, we show in the proof that it guarantees a uniform exploration property for all policies generated by Algorithm \ref{algo:stochastic_game}; see Lemma \ref{le:exploration}.

For Algorithm \ref{algo:stochastic_game}, we consider either constant stepsizes, i.e., $\alpha_k\equiv \alpha$ and $\beta_k\equiv \beta$, or harmonically diminishing stepsizes, i.e., $\alpha_k=\alpha/(k+h)$ and $\beta_k=\beta/(k+h)$, where $h\geq 0$. In both cases, we ensure $\alpha_k,\beta_k\in (0,1)$ and $\beta_k=c_{\alpha,\beta}\alpha_k$ for all $k$, where  $c_{\alpha,\beta}\in (0,1)$ is the stepsize ratio. In the stochastic-game setting, we redefine 
\begin{align}\label{definition:ell_tau}
    \ell_\tau=\frac{1}{1+(m-1)\exp(2/[(1-\gamma)\tau])},
\end{align}
which, analogously to the matrix-game setting, is a uniform lower bound on the entries of the policies generated by Algorithm \ref{algo:stochastic_game} (cf. Lemma \ref{le:boundedness_proof_outline}). We next state our requirement for choosing the stepsizes.

\begin{condition}\label{con:stepsize_stochastic_game}
When using constant or diminishing stepsizes, we choose $\tau\leq 1/(1-\gamma)$ and
\begin{align*}
c_{\alpha,\beta}
\leq
\min\left\{
\frac{1}{76L_pnm},
\frac{\mu_{\min}\ell_\tau\tau(1-\gamma)^2}{34nm^2},
\frac{\mu_{\min}\ell_{\tau}^3\tau^3(1-\gamma)^2}{144m^2}
\right\},
\end{align*}
where $\mu_{\min}\in (0,1)$ and $L_p>0$ are problem-dependent constants   explicitly defined in Section \ref{ap:sketchq}. 
In addition, when using $\alpha_k=\alpha/(k+h)$ and $\beta_k=\beta/(k+h)$, we choose $\beta=4$.
\end{condition}
\begin{remark}
When using harmonically diminishing stepsizes, the proof goes through as long as $\beta>2$. In Condition \ref{con:stepsize_stochastic_game}, we choose $\beta=4$ to simplify the statement of the results.
\end{remark}

We next state the finite-sample bound of Algorithm \ref{algo:stochastic_game}. For simplicity of presentation, we use $a\lesssim b$ to denote  that there exists an \textit{absolute} constant $c>0$ such that $a\leq bc$. 
\begin{theorem}\label{thm:stochastic_game}
	Suppose that both players follow Algorithm \ref{algo:stochastic_game}, Assumption \ref{as:MC} is satisfied, and the stepsizes $\{\alpha_k\}$ and $\{\beta_k\}$ satisfy Condition \ref{con:stepsize_stochastic_game}. Then, we have the following results: 
 \begin{enumerate}[(1)]
     \item When using constant stepsizes, there exists $z_\beta=\mathcal{O}(\log(1/\beta))$ such that the following inequality holds as long as $K\geq z_\beta$:
     \begin{align*}
        \mathbb{E}[\text{NG}(\pi_{T,K}^1,\pi_{T,K}^2)]
        \lesssim \,&\underbrace{\frac{m^2T}{\tau(1-\gamma)^3}\left(\frac{1+\gamma}{2}\right)^{T-1}}_{:=\mathcal{E}_1}+\underbrace{\frac{ m^2L_{\text{in}}(K-z_\beta)^{1/2}}{\tau(1-\gamma)^4}\left(1-\frac{\beta}{2}\right)^{\frac{K-z_\beta-1}{2}}}_{:=\mathcal{E}_2}\nonumber\\
 &+\underbrace{\frac{nm}{(1-\gamma)^4c_{\alpha,\beta}}z_\beta^2\alpha^{1/2}}_{:=\mathcal{E}_3}
 +\underbrace{\frac{\tau \log(m)}{(1-\gamma)^2}}_{:=\mathcal{E}_4},
    \end{align*}
    where $L_{\text{in}}=\frac{4}{(1-\gamma)}+2\tau \log(m)+\frac{8nm}{(1-\gamma)^2}$.
    \item When using $\alpha_k=\alpha/(k+h)$ and $\beta_k=\beta/(k+h)$, there exists $k_0>0$ such that the following inequality holds as long as $K\geq k_0$:
    \begin{align*}
        \mathbb{E}[\text{NG}(\pi_{T,K}^1,\pi_{T,K}^2)]\lesssim \frac{m^2 T}{\tau(1-\gamma)^3}\left(\frac{1+\gamma}{2}\right)^{T-1}+\frac{L_{\text{in}}nmz_K^2\alpha_K^{1/2}}{(1-\gamma)^4  \alpha_{k_0}^{1/2}c_{\alpha,\beta}}+\frac{\tau \log(m)}{(1-\gamma)^2},
    \end{align*}
    where $z_K=\mathcal{O}(\log(K))$.
 \end{enumerate}
\end{theorem}

The proof of Theorem~\ref{thm:stochastic_game} is presented in Section~\ref{sec:proof_outline}. We next discuss the implications of the theorem. In Theorem~\ref{thm:stochastic_game} (1), the bound consists of a value-iteration error term $\mathcal{E}_1$, an optimization error term $\mathcal{E}_2$, a statistical error term $\mathcal{E}_3$, and a smoothing-bias term $\mathcal{E}_4$. The term $\mathcal{E}_1$ would be the only error term if exact minimax value iteration could be performed to solve the game. Since minimax value iteration converges geometrically, $\mathcal{E}_1$ also decays at a geometric rate. The optimization error term $\mathcal{E}_2$ and the statistical error term $\mathcal{E}_3$ arise from learning the inner-loop auxiliary matrix games, while the smoothing-bias term $\mathcal{E}_4$ is due to using smoothed best responses instead of exact best responses. These three terms are the stochastic-game counterparts of the error terms that appeared  in the matrix-game analysis. In Theorem~\ref{thm:stochastic_game} (2), due to the use of diminishing stepsizes, both the optimization error and the statistical error converge at a rate of $\Tilde{\mathcal{O}}(K^{-1/2})$.

Although the transient terms in Theorem~\ref{thm:stochastic_game} have desirable convergence rates, including geometric decay in $T$ and $\Tilde{\mathcal{O}}(K^{-1/2})$ decay in $K$, the stepsize ratio $c_{\alpha,\beta}$, which is exponentially small in $\tau^{-1}$, appears as $c_{\alpha,\beta}^{-1}$ in the bound. Therefore, due to the smoothing bias, achieving an $\epsilon$-NE with  $\mathbb{E}[\text{NG}(\pi_{T,K}^1,\pi_{T,K}^2)]\leq \epsilon$ may require an overall sample complexity that has  exponential dependence on $\epsilon^{-1}$. This phenomenon is analogous to the zero-sum matrix-game setting. As explained in detail in Section~\ref{sec:bandit}, it arises from the limited exploration induced by (the natural) exponential softmax policies.

To achieve an overall sample complexity polynomial in $\epsilon^{-1}$, we next modify Algorithm~\ref{algo:stochastic_game} to encourage exploration, leading to Algorithm~\ref{algo:stochastic_game_slow}. 
Recall that our goal is to keep 
the (natural) \emph{best-response}-type update rules, with minimal modification to  Algorithm \ref{algo:stochastic_game}.
Specifically, recall that given $\bar{\epsilon}\in (0,1)$ and $\tau>0$, the operator $\sm:\mathbb{R}^d\to \mathbb{R}^d$, where $d$ can be either $m_1$ or $m_2$, is defined as
\begin{align*}
    \sm(x)
    =
    \bar{\epsilon}\cdot \text{Unif}_d
    +
    (1-\bar{\epsilon})\cdot \sigma_\tau(x),\quad \forall\,x\in\mathbb{R}^d,
\end{align*}
where $\text{Unif}_d$ denotes the $d$-dimensional uniform distribution and $\sigma_\tau(\cdot)$ is the softmax operator with temperature $\tau$. Note that the only difference between Algorithm \ref{algo:stochastic_game} and Algorithm \ref{algo:stochastic_game_slow} is that, in Algorithm \ref{algo:stochastic_game_slow}, Line $4$, we used $\sm(q_{t,k}^i(s))$ instead of $\sigma_\tau(q_{t,k}^i(s))$.

\begin{algorithm}[ht]\caption{VI-SBR with  $\Bar{\epsilon}$-Exploration (of Player $i$)}\label{algo:stochastic_game_slow} 
	\begin{algorithmic}[1]
		\STATE \textbf{Input:} Integers $K$ and $T$, initializations $v_0^i=0\in\mathbb{R}^{n}$, $q_{0,0}^i=0\in\mathbb{R}^{nm_i}$, and $\pi_{0,0}^i(s)=\text{Unif}(\mathcal{A}^i)$ for all $s\in\mathcal{S}$
		\FOR{$t=0,1,\cdots,T-1$}
		\FOR{$k=0,1,\cdots,K-1$}
		\STATE $\pi_{t,k+1}^i(s)=\pi_{t,k}^i(s)+\beta_k(\sm(q_{t,k}^i(s))-\pi_{t,k}^i(s))$ for all $s\in\mathcal{S}$
		\STATE Play
		$A_k^i\sim \pi_{t,k+1}^i(\cdot|S_k)$ (against $A_k^{-i}$) and observe $S_{k+1}\sim p(\cdot\mid S_k,A_k^i,A_k^{-i})$
		\STATE $q_{t,k+1}^i(s,a^i)=q_{t,k}^i(s,a^i)+\alpha_k\mathds{1}_{\{(s,a^i)=(S_k,A_k^i)\}}(R_i(S_k,A_k^i,A_k^{-i})+\gamma v_t^i(S_{k+1})-$\\
        $q_{t,k}^i(S_k,A_k^i))$ for all $(s,a^i)$
		\ENDFOR
		\STATE $v_{t+1}^i(s)=\pi_{t,K}^i(s)^\top q_{t,K}^i(s)$ for all $s\in\mathcal{S}$
  \STATE Set $S_{0}=S_{K}$, $q_{t+1,0}^i=q_{t,K}^i$, and
$\pi_{t+1,0}^i=\pi_{t,K}^i$ 
		\ENDFOR
	\end{algorithmic}
\end{algorithm} 

Next, we present the finite-sample analysis of Algorithm \ref{algo:stochastic_game_slow}; see Appendix \ref{pf:thm:stochastic_game_slow} for the proof. Similarly, we only present the result for constant stepsizes, as the result for diminishing stepsizes follows by a straightforward extension.
\begin{theorem}\label{thm:stochastic_game_slow}
Suppose that Assumption~\ref{as:MC} is satisfied and both players follow
Algorithm~\ref{algo:stochastic_game_slow}. Let $\tau\in(0,1]$, $\bar{\epsilon}=\tau$, and 
$c_{\alpha,\beta}=\beta/\alpha\le \mu_{\min}\tau/(2m)$. Then, for all
$K\ge z_\beta=\mathcal{O}(\log(1/\beta))$, 
    \begin{align*}
        \mathbb{E}[\text{NG}(\pi_{T,K}^1,\pi_{T,K}^2)]
    \lesssim \,&\frac{T\gamma^{T-1}}{(1-\gamma)^2}+\frac{ n^{1/2}m^{5/2}}{\mu_{\min}\tau^2(1-\gamma)^4}\frac{\beta}{\alpha}+\frac{ m^{3/2}z_\beta\alpha^{1/2}}{\mu_{\min}^{1/2}\tau^{3/2}(1-\gamma)^4}\nonumber\\
    &+\frac{n^{1/2}m^{3/2}}{(1-\gamma)^4\tau}K(1-\beta)^K+\frac{m\tau}{(1-\gamma)^3}+\frac{m^2 \beta}{\tau(1-\gamma)^4}.
    \end{align*}
    As a result, to find an $\epsilon$-Nash equilibrium: $\mathbb{E}[\text{NG}(\pi_{T,K}^1,\pi_{T,K}^2)]\leq \epsilon$, the sample complexity is $\Tilde{\mathcal{O}}(\epsilon^{-8})$.
\end{theorem}

To the best of our knowledge, Theorem~\ref{thm:stochastic_game_slow} provides the first finite-sample analysis of a decentralized best-response-based learning algorithm for zero-sum stochastic games with sample complexity polynomial in $\epsilon^{-1}$. {In the concurrent work \citep{cai2023uncoupled}, the authors established an $\mathcal{O}(\epsilon^{-9-\nu})$ (for some $\nu>0$) sample complexity for decentralized learning in zero-sum stochastic games. Although our sample complexity is $\mathcal{O}(\epsilon^{-8})$, these two results are not directly comparable. Their algorithm is based on online mirror descent, whereas ours is based on smoothed best-response dynamics. Moreover, our bound is in expectation and uses stepsizes tuned to the target accuracy level, while their bound holds with high probability and provides anytime last-iterate convergence guarantees.} 

Finally, we consider the case where the opponent of player $i$ plays a stationary policy, and provide a finite-sample bound for player $i$ to find a best response. This is pertinent to the desired feature of \emph{rationality} \cite{bowling2001rational} for decentralized learning dynamics. The proof of the following corollary is presented in Appendix~\ref{pf:co:rationality}.
\begin{corollary}\label{co:rationality}
	Given $i\in \{1,2\}$, suppose that player $i$ follows Algorithm \ref{algo:stochastic_game_slow}, but its opponent player $-i$ follows a stationary policy $\pi^{-i}$. Then, to achieve $\max_{\hat{\pi}^i}U^i(\hat{\pi}^i,\pi^{-i})-\mathbb{E}[U^i(\pi_{T,K}^i,\pi^{-i})]\leq \epsilon$, the sample complexity is $\tilde{\mathcal{O}}(\epsilon^{-8})$.
\end{corollary} 
\begin{remark}
According to the definition in \cite{bowling2001rational}, rationality means that a player's policy converges to a best response to its opponent when the opponent uses an \textit{asymptotically} stationary policy. Since oue focus is on the  finite-sample analysis, we assume that the opponent's policy is stationary; otherwise, the convergence rate of the opponent's policy, which may be arbitrary, would also affect the exact order of the bounds.
\end{remark}

\begin{remark}
We prove Corollary~\ref{co:rationality} directly based on Theorem~\ref{thm:stochastic_game_slow}, which is why it is stated as a corollary. The result is by no means tight: when the opponent of player $i$ uses a stationary policy, player $i$ effectively faces a single-agent problem. In this case, Algorithm~\ref{algo:stochastic_game_slow} reduces to a smoothed stochastic variant of incremental policy iteration, and one could independently prove its convergence rate to obtain sharper sample complexity bounds. This, however, is not the focus of this work.
\end{remark}

The rationality property follows from the \emph{on-policy} nature of the learning algorithm. When the opponent's policy is stationary, player $i$ faces an induced single-agent Markov decision process. Notably, the behavior policy used to generate samples is the same policy that is being updated toward a smoothed best response. This contrasts with off-policy procedures, where the behavior policy may be unrelated to the policy being improved. Thus, the same mechanism that makes the learning algorithm independent and symmetric also yields best-response learning against a stationary opponent.

\section{Proof of Theorem~\ref{thm:stochastic_game}: A Coupled Lyapunov-Based Approach}\label{sec:proof_outline}

In this section, we use Theorem~\ref{thm:stochastic_game} as a representative example to present the key challenges and highlight the main technical novelties in our proof. The main difficulty is that VI-SBR maintains multiple sets of stochastic iterates, including the policies, the local $q$-functions, and the value functions, which are updated in a coupled manner. Thus, no component can be analyzed in isolation: the policy updates depend on the accuracy of the $q$-functions, the $q$-function updates are driven by time-varying policies and value functions, and the value-function updates depend on the outputs of the inner-loop.

Several additional and unique challenges arise in the stochastic-game setting: (1) the independently maintained value functions induce auxiliary matrix games that are \emph{not} exactly \emph{zero-sum} during learning; (2) the data are collected along a \emph{single}  trajectory under \emph{time-varying}  policies, leading to time-inhomogeneous Markovian noises in the stochastic  iterates. To address these challenges, we develop a coupled Lyapunov-based approach: we construct Lyapunov functions for the coupled iterates, establish their drift inequalities, and solve the resulting system \emph{jointly} to obtain the finite-sample bounds. This framework may be useful more broadly for analyzing stochastic iterative algorithms with multiple coupled sets of iterates.

\subsection{Introducing the Lyapunov Functions}
We start by introducing the Lyapunov functions we use to analyze Algorithm \ref{algo:stochastic_game}. Specifically, for any $t,k\geq 0$ and $i\in \{1,2\}$, let $\Bar{q}_{t,k}^i\in\mathbb{R}^{nm_i}$ be defined as $\Bar{q}_{t,k}^i(s)=\mathcal{T}^i(v_t^i)(s)\pi_{t,k}^{-i}(s)$ for all $s\in\mathcal{S}$, and let
 \begin{align*}
    \mathcal{L}_v(t)=\,&\sum_{i=1,2}\|v_t^i-v_*^i\|_\infty,\quad 
     \mathcal{L}_{\text{sum}}(t)=\|v_t^1+v_t^2\|_\infty,\quad   \mathcal{L}_q(t,k)=\sum_{i=1,2}\|q_{t,k}^i-\Bar{q}_{t,k}^i\|_2^2,\\    \mathcal{L}_\pi(t,k)=\,&\max_{s\in\mathcal{S}}	\sum_{i=1,2}\max_{\mu^i\in\Delta(\mathcal{A}^i)}\left\{(\mu^i-\pi_{t,k}^i(s))^\top \mathcal{T}^i(v_t^i)(s)\pi_{t,k}^{-i}(s)+\tau \nu(\mu^i)-\tau\nu(\pi_{t,k}^i(s))\right\},
\end{align*}
where $v_*^i$ is the unique fixed point of the minimax Bellman operator $\mathcal{B}^i(\cdot)$. We will provide more detailed explanations of 
the Lyapunov function construction as we go over the proofs below. 

The first step 
is to bound the Nash gap in terms of the Lyapunov functions introduced above. This is formalized in the following lemma, whose proof is presented in Appendix~\ref{ap:bound_Nash_stochastic}.

\begin{lemma}\label{le:Nash_Combine}
    It holds that
    \begin{align}\label{eq:sketch:overall}
        \text{NG}(\pi_{T,K}^1,\pi_{T,K}^2)\leq \,&\frac{4}{1-\gamma}\left(2\mathcal{L}_{\text{sum}}(T)+\mathcal{L}_v(T)+\mathcal{L}_\pi(T,K)+2\tau \log(m)\right).
    \end{align}
\end{lemma}

The remainder of the proof is devoted to bounding the Lyapunov functions. Before proceeding, we present an important boundedness property of the iterates generated by Algorithm \ref{algo:stochastic_game}.

\begin{lemma}\label{le:boundedness_proof_outline}
For all $t,k\geq 0$ and $i\in \{1,2\}$, we have
\begin{enumerate}[(1)]
    \item $\|v_t^i\|_\infty\leq 1/(1-\gamma)$ and $\|q_{t,k}^i\|_\infty\leq 1/(1-\gamma)$;
    \item $\min_{s\in\mathcal{S},a^i\in\mathcal{A}^i}\pi_{t,k}^i(a^i\mid s)\geq \ell_\tau$, where $\ell_\tau$ is defined in \eqref{definition:ell_tau}.
\end{enumerate}
\end{lemma}

The proof of Lemma \ref{le:boundedness_proof_outline} is based on a nested induction argument, and is presented in Appendix \ref{subsec:boundedness}. This result will be used frequently in our analysis. 

\subsection{Analysis of the Outer Loop: the $v$-Function Update}
We first consider the Lyapunov functions $\mathcal{L}_v(T)$ and $\mathcal{L}_{\text{sum}}(T)$, which are defined in terms of the value functions updated in the outer loop of Algorithm \ref{algo:stochastic_game}. 

Recall from Section \ref{subsec:Markov_Algorithm} that the outer loop of Algorithm \ref{algo:stochastic_game} is designed as an approximation of the minimax value iteration $v_{t+1}^i=\mathcal{B}^i(v_t^i)$, where $i\in \{1,2\}$. Since it is  known that the minimax Bellman operator $\mathcal{B}^i(\cdot)$ is a contraction mapping with respect to the $\ell_\infty$-norm \cite{shapley1953stochastic}, we use $\mathcal{L}_v(t)=\sum_{i=1,2}\|v_t^i-v_*^i\|_\infty$ as the Lyapunov function to study the evolution of the value functions  $(v_t^1,v_t^2)$. In the following lemma, we present the Lyapunov drift inequality for $\mathcal{L}_v(t)$:

\begin{lemma}\label{le:outer-loop}
	It holds for all $t\geq 0$ that	\begin{align}
	\mathcal{L}_v(t+1)\leq \,&\underbrace{\gamma \mathcal{L}_v(t)}_{\text{Drift}}+\underbrace{4\mathcal{L}_{\text{sum}}(t)+2\mathcal{L}_q^{1/2}(t,K)+4\mathcal{L}_\pi(t,K)+6\tau \log(m)}_{\text{Additive Errors}}.\label{eq:sketchv}
\end{align}
\end{lemma}

The proof of Lemma~\ref{le:outer-loop} is presented in Appendix~\ref{subsec:outer-loop}. Since $\gamma\in (0,1)$, the inequality is contractive; equivalently, $\mathcal{L}_v(t)$ has a negative drift, which is consistent with the geometric convergence of minimax value iteration. The additive error terms on the right-hand side of \eqref{eq:sketchv} involve other Lyapunov functions, which is why we refer to \eqref{eq:sketchv} as a coupled Lyapunov drift inequality.

Moving to the Lyapunov function $\mathcal{L}_{\text{sum}}(t)$, recall from Section \ref{subsec:Markov_Algorithm} that, due to decentralized learning, we do not necessarily have $v_t^1+v_t^2=0$. As a result, the auxiliary matrix game at state $s$ with payoff matrices $\mathcal{T}^1(v_t^1)(s)$ and $\mathcal{T}^2(v_t^2)(s)$ that the inner loop of Algorithm \ref{algo:stochastic_game} is designed to solve is not necessarily a zero-sum matrix game, which presents a major challenge in the analysis. The error induced from such a non-zero-sum structure appears in existing work \citep{sayin2021decentralized,sayin2022fictitious}, and was handled by designing a novel truncated Lyapunov function. However, the truncated Lyapunov function was sufficient to establish the asymptotic convergence, but did not provide the explicit rate at which the induced error goes to zero. To enable finite-sample analysis, we introduce $\mathcal{L}_{\text{sum}}(t)=\|v_t^1+v_t^2\|_\infty$ as a Lyapunov function in our coupled Lyapunov framework, which is customized to capture the behavior of the induced error from the non-zero-sum structure of the inner-loop matrix game. 

The next lemma presents the Lyapunov drift inequality for $\mathcal{L}_{\text{sum}}(t)$. Its proof is presented in Appendix~\ref{pf:le:outer-sum}. 

\begin{lemma}\label{le:outer-sum}
	It holds for all $t\geq 0$ that 
    \begin{align}
	\mathcal{L}_{\text{sum}}(t+1)\leq \gamma \mathcal{L}_{\text{sum}}(t)+2\mathcal{L}_q(t,K)^{1/2}.\label{eq:sketchvsum}
\end{align}
\end{lemma}
Note that (\ref{eq:sketchvsum}) is also a coupled Lyapunov drift inequality as it consists of a negative drift and an additive error term defined in terms of other Lyapunov functions.

It now remains to bound $\mathcal{L}_q(t,k)$ and $\mathcal{L}_\pi(t,k)$, which are defined in terms of the $q$-functions and the policies updated in the inner loop of Algorithm \ref{algo:stochastic_game}.

\subsection{Analysis of the Inner Loop: the Policy Update}  As illustrated in Section~\ref{subsec:Markov_Algorithm}, for each state $s$, the policy update can be viewed as a discrete-time stochastic variant of the smoothed best-response dynamics \citep{leslie2005individual}; see \eqref{eq:FP}. For smoothed best-response dynamics in zero-sum matrix games, formulated as an ODE, the regularized Nash gap has been shown to be a valid Lyapunov function \citep{hofbauer2005learning}. Motivated by this observation, given the pair of value functions $v_t=(v_t^1,v_t^2)$ from the outer loop of Algorithm~\ref{algo:stochastic_game}, we would like to construct a Lyapunov function for the inner-loop policy iterates by considering the induced matrix game at each state $s$. Specifically, since the inner loop of Algorithm~\ref{algo:stochastic_game} is designed to solve the matrix game with payoff matrices $\mathcal{T}^1(v_t^1)(s)$ and $\mathcal{T}^2(v_t^2)(s)$, a natural candidate is
\begin{align*}
    \mathcal{L}_\pi(t,k)
    =
    \max_{s\in\mathcal{S}}
    \sum_{i=1,2}
    \max_{\mu^i\in\Delta(\mathcal{A}^i)}
    \Big\{
    (\mu^i-\pi_{t,k}^i(s))^\top \mathcal{T}^i(v_t^i)(s)\pi_{t,k}^{-i}(s) +\tau \nu(\mu^i)-\tau\nu(\pi_{t,k}^i(s))
    \Big\}.
\end{align*}
Here, the operator $\max_{s\in\mathcal{S}}(\cdot)$ accounts for the fact that Algorithm~\ref{algo:stochastic_game} induces a separate matrix game at each state.

While this  construction of Lyapunov function is natural, the results in \cite{hofbauer2005learning} are not directly applicable for establishing a negative drift in our setting for three reasons: (1) our learning algorithm is discrete-time, (2) the induced inner-loop matrix game need not be zero-sum, i.e., $[\mathcal{T}^1(v_t^1)](s)+[\mathcal{T}^2(v_t^2)](s)\neq 0$, and (3) the payoff vectors, i.e., $[\mathcal{T}^1(v_t^1)](s)\pi_{t,k}^2(s)$ and $[\mathcal{T}^2(v_t^2)](s)\pi_{t,k}^1(s)$), are estimated through the local $q$-functions and are thus  not exact. Therefore, to facilitate our analysis, we consider the function
\begin{align}\label{eq:RNG_X}
    V_X(\mu^1,\mu^2)
    =
    \sum_{i=1,2}
    \max_{\hat{\mu}^i\in\Delta(\mathcal{A}^i)}
    \left\{
    (\hat{\mu}^i-\mu^i)^\top X_i\mu^{-i}
    +
    \tau \nu(\hat{\mu}^i)
    -
    \tau\nu(\mu^i)
    \right\},
\end{align}
defined for all $(\mu^1,\mu^2)\in\Delta(\mathcal{A}^1)\times\Delta(\mathcal{A}^2)$, where $X_i$, $i\in \{1,2\}$, is an $m_i\times m_{-i}$ matrix. Note that we do not assume $X_1+X_2^\top =0$. We establish a sequence of properties of $V_X(\cdot,\cdot)$ to overcome the three challenges described above.

Let $\Pi_\tau=\{(\mu^1,\mu^2)\mid \min_{a^1}\mu^1(a^1)\geq \ell_\tau,\min_{a^2}\mu^2(a^2)\geq \ell_\tau\}$. Note that Lemma \ref{le:boundedness_proof_outline} implies that $(\pi_k^1(s),\pi_k^2(s))\in\Pi_\tau$ for all $k\geq 0$ and $s\in\mathcal{S}$. The proof of the following lemma is presented in Appendix~\ref{pf:le:properties_Lyapunov_main}.

 \begin{lemma}\label{le:properties_Lyapunov_main}
	The function $V_X(\cdot,\cdot)$ has the following properties:
	\begin{enumerate}[(1)]
        \item For $i\in \{1,2\}$, fixing $\mu^{-i}\in\Delta(\mathcal{A}^{-i})$, the function $V_X(\mu^1,\mu^2)$ as a function of $\mu^i$ is $\tau$ -- strongly convex with respect to $\|\cdot\|_2$.
		\item $V_X(\cdot,\cdot)$ is $\tilde{L}_\tau$ -- smooth on $\Pi_\tau$, where $\tilde{L}_\tau=2\left(\frac{\tau}{\ell_\tau} +\frac{\max(\|X_1\|_2^2,\|X_2\|_2^2)}{\tau}+\|X_1+X_2^\top\|_2\right)$.
		\item It holds for any $(\mu^1,\mu^2)\in\Delta(\mathcal{A}^1)\times \Delta(\mathcal{A}^2)$ that  
		\begin{align*}
			&\langle \nabla_1V_X(\mu^1,\mu^2),\sigma_\tau(X_1\mu^2)-\mu^1 \rangle+\langle \nabla_2V_X(\mu^1,\mu^2),\sigma_\tau(X_2\mu^1)-\mu^2 \rangle\\
\leq\,& -\frac{7}{8}V_X(\mu^1,\mu^2)+\frac{16}{\tau}\|X_1+X_2^\top \|_2^2.
		\end{align*}
		\item For any $u^1\in\mathbb{R}^{m_1}$ and $u^2\in\mathbb{R}^{m_2}$, we have for all $(\mu^1,\mu^2)\in\Pi_\tau$ that
		\begin{align*}
			&\langle \nabla_1V_X(\mu^1,\mu^2),\sigma_\tau(u^1)-\sigma_\tau(X_1\mu^2)\rangle+\langle \nabla_2V_X(\mu^1,\mu^2),\sigma_\tau(u^2)-\sigma_\tau(X_2\mu^1)\rangle\\
\leq \,&\frac{1}{8}V_X(\mu^1,\mu^2)+\frac{8}{\tau}\left(\frac{1}{\ell_\tau}+\frac{\max(\|X_1\|_2,\|X_2\|_2)}{\tau}\right)^2\sum_{i=1,2}\| u^i-X_i\mu^{-i}\|_2^2.
		\end{align*}
	\end{enumerate}
\end{lemma}

The four properties in Lemma~\ref{le:properties_Lyapunov_main} play distinct roles in the Lyapunov drift analysis. Properties (1) and (2) provide the strong convexity and smoothness needed to control the \emph{discretization}  error when passing from the continuous-time smoothed best-response dynamics to the discrete stochastic updates in our learning algorithm. Property (3) establishes a negative drift for the regularized Nash-gap Lyapunov function, up to an additive error proportional to $\|X_1+X_2^\top\|_2^2$; this term handles the deviation of the induced auxiliary game from being exactly zero-sum. Property (4) controls the error caused by using the estimated local $q$-functions in the policy update, rather than the exact payoff vectors. Together, these properties allow us to derive a Lyapunov drift inequality that simultaneously accounts for discretization, the non-zero-sum perturbation of the auxiliary games, and the local $q$-function evaluation error.

With the properties of $V_X(\cdot,\cdot)$ in hand, we establish the Lyapunov drift inequality for $\mathcal{L}_\pi(t,k)$ in the following lemma, whose proof is deferred to Appendix~\ref{pf:le:policy_drift}. For notational convenience, let $\mathcal{F}_t$ denote the history of Algorithm~\ref{algo:stochastic_game} immediately before the $t$-th outer-loop iteration, and write $\mathbb{E}_t[\cdot]$ for $\mathbb{E}[\cdot \mid \mathcal{F}_t]$.

\begin{lemma}\label{le:policy_drift}
The following inequality holds for all $k\geq 0$:
\begin{align}\label{eq:sketchpi}
    \mathbb{E}_t\left[\mathcal{L}_\pi(t,k+1)\right]
	\leq \,&\underbrace{\left(1-\frac{3\beta_k}{4}\right)\mathbb{E}_t\left[\mathcal{L}_\pi(t,k)\right]}_{\text{Drift}}\nonumber\\
 &+\underbrace{2L_\tau\beta_k^2+ \frac{32m^2\beta_k}{\tau^3\ell_\tau^2(1-\gamma)^2}\mathbb{E}_t[\mathcal{L}_q(t,k)]+\frac{16m^2\beta_k}{\tau}\mathcal{L}_{\text{sum}}(t)^2}_{\text{Additive Errors}},
\end{align}
where $L_\tau=2\left(\frac{\tau}{\ell_\tau}
    +\frac{m^2}{\tau(1-\gamma)^2}
    +\frac{2m}{1-\gamma}\right)$.
\end{lemma}

To interpret (\ref{eq:sketchpi}), suppose that we were considering the continuous-time smoothed best-response dynamics \cite{hofbauer2005learning}. Then, the \textit{Additive Errors} would disappear in the sense that the time-derivative of the Lyapunov function along the trajectory of the ODE is strictly negative. Thus, the three terms in the \textit{Additive Errors} on the right-hand side of (\ref{eq:sketchpi}) can be interpreted,  respectively, as (1) the discretization error in the update equation, (2) the stochastic error in the $q$-function estimate, and (3) the error due to the non-zero-sum structure of the inner-loop auxiliary matrix game.

\subsection{Analysis of the Inner Loop: the $q$-Function Update}\label{ap:sketchq}
Our next focus is the $q$-function, whose update equation is in the same spirit as TD-learning in reinforcement learning \citep{tsitsiklis1994asynchronous,tsitsiklis1997analysis}. Inspired by the existing literature studying TD-learning \citep{tsitsiklis1994asynchronous,tsitsiklis1997analysis,bhandari2018finite,srikant2019finite},
we will reformulate the update equation of the $q$-function as a stochastic approximation algorithm for estimating the solution of a time-varying target equation. For ease of presentation, since we are focusing on the inner loop, we will omit the iteration index $t$ for the outer loop. 

For $i\in \{1,2\}$, fixing a value function $v^i\in\mathbb{R}^{n}$ from the outer loop, let $F^i:\mathbb{R}^{n m_i}\times \mathcal{S}\times \mathcal{A}^i\times \mathcal{A}^{-i}\times \mathcal{S}\to \mathbb{R}^{n m_i}$ be defined as  
\begin{align*}
    [F^i(q^i,s_0,a_0^i,a_0^{-i},s_1)](s,a^i)
	=\mathds{1}_{\{(s,a^i)=(s_0,a_0^i)\}}\left(R_i(s_0,a_0^i,a_0^{-i})+\gamma v^i(s_1)-q^i(s_0,a_0^i)\right)
\end{align*}
for all $(q^i,s_0,a_0^i,a_0^{-i},s_1)$ and $(s,a^i)$. Then Algorithm \ref{algo:stochastic_game}, Line $6$, can be compactly written as
\begin{align}\label{sa:reformulation_sketch}
	q_{k+1}^i=q_k^i+\alpha_k F^i(q_k^i,S_k,A_k^i,A_k^{-i},S_{k+1}).
\end{align}
For any $k\geq 0$, let $\mu_k\in\Delta(\mathcal{S})$ denote the stationary distribution of the Markov chain $\{S_n\}_{n\geq 0}$ induced by the joint policy $\pi_k=(\pi_k^1,\pi_k^2)$, provided it exists and is unique. We will verify this existence and uniqueness shortly.
Let $\Bar{F}_k^i:\mathbb{R}^{n m_i}\to \mathbb{R}^{n m_i}$ be defined as 
\begin{align*}
    \Bar{F}_k^i(q^i)
	=\mathbb{E}_{S_0\sim \mu_k(\cdot),A_0^i\sim \pi_k^i(\cdot|S_0), A_0^{-i}\sim \pi_k^{-i}(\cdot|S_0), S_1\sim p(\cdot|S_0,A_0^i,A_0^{-i})}\left[F^i(q^i,S_0,A_0^i,A_0^{-i},S_1)\right]
\end{align*}
for all $q^i\in\mathbb{R}^{nm_i}$.
Then, the update equation (\ref{sa:reformulation_sketch}) can be viewed as a stochastic approximation algorithm for solving the time-varying equation $\Bar{F}_k^i(q^i)=0$ with time-inhomogeneous Markovian noise $\{(S_k,A_k^i,A_k^{-i},S_{k+1})\}$. Note that the reason for the Markov chain $\{(S_k,A_k^i,A_k^{-i},S_{k+1})\}$ being time-inhomogeneous is that our learning algorithm uses time-varying policies $\{\pi_k\}$.

Before analyzing the update in \eqref{sa:reformulation_sketch}, we first need to show that, under Assumption~\ref{as:MC}, for each $k$, the stationary distribution $\mu_k$ of the Markov chain $\{S_n\}_{n\geq 0}$ induced by $(\pi_k^1,\pi_k^2)$ exists and is unique. In addition, to guarantee exploration, we need $\min_{1\leq k\leq K}\min_s\mu_k(s)>0$. To this end, we present the following lemma, which establishes uniform mixing and uniform exploration properties for all policies encountered by the learning algorithm. Its proof is presented in Appendix~\ref{pf:le:exploration}.

\begin{lemma}\label{le:exploration}
Let $\Pi$ be the set of stationary joint policies.
Under Assumption~\ref{as:MC}, there exist constants $r_*\in\mathbb{N}$, $p_*\in(0,1]$, and $\rho_*\in(0,1)$, depending only on the transition kernel and the finite state-action spaces, but not on $\ell_\tau$, such that the following results hold: 
\begin{enumerate}[(1)]
    \item For any $\pi=(\pi^1,\pi^2)\in \Pi$, the Markov chain $\{S_k\}$ induced by the joint policy $\pi$ is irreducible and aperiodic, and hence admits a unique stationary distribution $\mu_\pi\in\Delta(\mathcal{S})$.

    \item It holds that $\sup_{\pi\in \Pi}
        \max_{s\in\mathcal{S}}
        \|P_\pi^k(s,\cdot)-\mu_\pi(\cdot)\|_{\text{TV}}
        \leq 2\rho_*^k$ for all $k\geq 0$.
    As a result, letting $t_{\pi,\eta}$ be the $\eta$-mixing time of the Markov chain induced by $\pi$, defined as $t_{\pi,\eta}
        =
        \min \{
        k\geq 0
        \,:\,
        \max_{s\in\mathcal{S}}
        \|P_\pi^k(s,\cdot)-\mu_\pi(\cdot)\|_{\text{TV}}
        \leq \eta
        \}$, 
    we have    \begin{align}\label{eq:mixing_time_definition}
        t_\eta
        :=
        \sup_{\pi}t_{\pi,\eta}
        \leq
        \left\lceil
        \frac{\log(2/\eta)}{\log(1/\rho_*)}
        \right\rceil .
    \end{align}

    \item There exists  $L_p:=r_*/p_*\geq 1$ such that
    \begin{align*}
        \|\mu_{\pi}-\mu_{\bar{\pi}}\|_1
        \leq
        L_p
        \left(
        \max_{s\in\mathcal{S}}\|\pi^1(s)-\bar{\pi}^1(s)\|_1
        +
        \max_{s\in\mathcal{S}}\|\pi^2(s)-\bar{\pi}^2(s)\|_1
        \right).
    \end{align*}
    for all $\pi=(\pi^1,\pi^2),\bar{\pi}=(\bar{\pi}^1,\bar{\pi}^2)\in \Pi$.
    \item It holds that 
    $\mu_{\min}
        :=
        \inf_{\pi\in \Pi}
        \min_{s\in\mathcal{S}}\mu_\pi(s)
        \geq p_*
        >0$.
\end{enumerate}
\end{lemma}

In Lemma~\ref{le:exploration}, Part~(1) ensures that every policy encountered by the algorithm induces a well-defined stationary distribution. Part~(2) strengthens this to a uniform mixing bound over the policy class $\Pi$, which allows us to control the Markovian sampling bias uniformly along the algorithm trajectory. Part~(3) provides a Lipschitz-type sensitivity bound for the stationary distribution with respect to the changes in policy, which is used to handle the time-inhomogeneity caused by time-varying policies. Finally, Part~(4) guarantees a uniform positive lower bound on all state stationary probabilities, ensuring that every state is visited with nonvanishing frequency under all policies encountered by the learning algorithm.

Now, we are ready to study the stochastic approximation algorithm \eqref{sa:reformulation_sketch}. We start by presenting a sequence of properties of the operators $F^i(\cdot)$ and $\Bar{F}_k^i(\cdot)$ in the following lemma, whose proof is presented in Appendix \ref{pf:le:operators}. 

\begin{lemma}\label{le:operators}
	The following properties hold for $i\in \{1,2\}$: 
	\begin{enumerate}[(1)]
		\item It holds that $\|F^i(q_1^i,s_0,a_0^i,a_0^{-i},s_1)-F^i(q_2^i,s_0,a_0^i,a_0^{-i},s_1)\|_2\leq \|q_1^i-q_2^i\|_2$ for any $(q_1^i,q_2^i)$ and $(s_0,a_0^i,a_0^{-i},s_1)$.
		\item It holds that $\|F^i(0,s_0,a_0^i,a_0^{-i},s_1)\|_2\leq 1/(1-\gamma)$ for all $ (s_0,a_0^i,a_0^{-i},s_1)$.
		\item $\bar{F}_k^i(q^i)=0$ has a unique solution $\bar{q}_k^i$, which is given as $\bar{q}_k^i(s)=\mathcal{T}^i(v^i)(s)\pi_k^{-i}(s)$ 
        for all $s$.
		\item It holds that $\langle \Bar{F}_k^i(q_1^i)-\Bar{F}_k^i(q_2^i),q_1^i-q_2^i\rangle\leq   -c_\tau\|q_1^i-q_2^i\|_2^2$ for all $(q_1^i,q_2^i)$, where $c_\tau=\mu_{\min}\ell_\tau$,  and $\mu_{\min}$ is defined in Lemma \ref{le:exploration}.
	\end{enumerate}
\end{lemma}

Among the properties established in the previous lemma, Part (4) is particularly important. It justifies the use of the standard quadratic Lyapunov function
\begin{align*}
    \mathcal{L}_q(k)
    =
    \sum_{i=1,2}\|q_{k}^i-\bar{q}_k^i\|_2^2
\end{align*}
to study \eqref{sa:reformulation_sketch}, where we recall that the outer-loop index $t$ is omitted.

Using $\mathcal{L}_q(k)$ as a Lyapunov function, we obtain a negative drift through a binomial decomposition and Lemma~\ref{le:operators} (4). The key challenge is to handle the time-inhomogeneous Markovian noise $\{(S_k,A_k^i,A_k^{-i},S_{k+1})\}$. To overcome this challenge, building on existing results \citep{bhandari2018finite,srikant2019finite,zou2019finite,khodadadian2021finite} and Lemma~\ref{le:exploration}, we develop an argument that combines mixing-time analysis with a sensitivity analysis of the stationary distribution with respect to the learning policies. This leads to the following overall Lyapunov drift inequality for $\mathcal{L}_q(k)$.

\begin{lemma}\label{le:q-function-drift}
Let $z_k = t_{\beta_k}$, where $t_\eta$ is defined for any $\eta > 0$ in Lemma \ref{le:exploration} (2).
The following inequality holds for all $k\geq z_k$:
\begin{align}\label{eq:sketchq}
    \mathbb{E}[\mathcal{L}_q(k+1)]
    \leq
    \underbrace{\left(1-\alpha_k \mu_{\min}\ell_\tau\right)
    \mathbb{E}[\mathcal{L}_q(k)]}_{\text{Drift}}
    +
    \underbrace{
    \frac{100nm}{(1-\gamma)^2}z_k\alpha_k\alpha_{k-z_k,k-1}
    +
    \frac{\beta_k}{4}\mathbb{E}[\mathcal{L}_\pi(k)]
    }_{\text{Additive Errors}},
\end{align}
where $\alpha_{k_1,k_2}:=\sum_{k=k_1}^{k_2}\alpha_k$.
\end{lemma}

The proof of Lemma~\ref{le:q-function-drift} is presented in Appendix~\ref{pf:le:q-function-drift}.

\subsection{Solving Coupled Lyapunov Drift Inequalities}\label{subsec:decouple_stochastic}
Until this point, we have established the Lyapunov drift inequalities for the individual $v$-functions, the sum of the $v$-functions, the policies, and the $q$-functions in (\ref{eq:sketchv}), (\ref{eq:sketchvsum}), (\ref{eq:sketchpi}), and (\ref{eq:sketchq}), respectively. They are restated as follows:
\begin{align}
	\mathcal{L}_{v}(t+1)\leq \,&\gamma\mathcal{L}_{v}(t)+4\mathcal{L}_{\text{sum}}(t)+4\mathcal{L}_{\pi}(t,K)+2\mathcal{L}_q^{1/2}(t,K)+6\tau\log(m),\label{eq:sketch1}\\
	\mathcal{L}_{\text{sum}}(t+1)\leq\,& \gamma \mathcal{L}_{\text{sum}}(t)+2\mathcal{L}_q^{1/2}(t,K),\label{eq:sketch2}\\
	\mathbb{E}_t[\mathcal{L}_{\pi}(t,k+1)]\leq\,& (1-3\beta_k/4)\mathbb{E}_t[\mathcal{L}_{\pi}(t,k)]+C_1(\beta_k^2+\beta_k\mathbb{E}_t[\mathcal{L}_{q}(t,k)]+\beta_k \mathcal{L}_{\text{sum}}^2(t)),\label{eq:sketch3}\\
	\mathbb{E}_t[\mathcal{L}_{q}(t,k+1)]\leq\,& (1-\mu_{\min}\ell_\tau\alpha_k)\mathbb{E}_t[\mathcal{L}_{q}(t,k)]+\beta_k\mathbb{E}_t[\mathcal{L}_{\pi}(t,k)]/4+C_2z_k^2\alpha_k^2,\label{eq:sketch4}
\end{align}
where $C_1,C_2$ are problem-dependent constants 
introduced for the simplicity of notation. Moreover, to obtain (\ref{eq:sketch4}) from (\ref{eq:sketchq}), we used the fact that $\alpha_{k-z_k,k-1}=\mathcal{O}(z_k\alpha_k)$ 
\citep{chen2021finite}.

To decouple the highly coupled Lyapunov inequalities, our high-level ideas are: (1) using the Lyapunov drift inequalities in a combined way instead of in a separate manner, and (2) a novel bootstrapping procedure where we first derive a crude bound on $\mathbb{E}[\mathcal{L}_{q}(t,K)]$ and then substitute the crude bound back into the Lyapunov drift inequalities to derive a tighter bound. We next elaborate on our approach with more details.

For ease of presentation, for a scalar-valued quantity $W$ that is a function of $k$ and/or $t$, we say $W=o_k(1)$ if $\lim_{k\rightarrow\infty} W=0$ and $W=o_t(1)$ if $\lim_{t\rightarrow\infty} W=0$. The explicit convergence rates of the $o_k(1)$ term and the $o_t(1)$ term will be revealed in the complete proof in Appendix \ref{sec:strategy}, but are not important for the illustration here.

\textbf{Step 1.} Adding up (\ref{eq:sketch3}) and (\ref{eq:sketch4}) and then repeatedly using the result, we obtain:
\begin{align}\label{eq:sketch5}
	\mathbb{E}_t[\mathcal{L}_{\pi}(t,k)]\leq \mathbb{E}_t[\mathcal{L}_{\pi}(t,k)+\mathcal{L}_{q}(t,k)]= o_k(1)+\mathcal{O}(1)\mathcal{L}_{\text{sum}}^2(t),\quad \forall\,t,k.
\end{align}

\textbf{Step 2.} Substituting the bound for $\mathbb{E}_t[\mathcal{L}_{\pi}(t,k)]$ in (\ref{eq:sketch5}) back into (\ref{eq:sketch4}) and repeatedly using the resulting inequality, we obtain $\mathbb{E}_t[\mathcal{L}_{q}(t,K)]=o_K(1)+\mathcal{O}(c_{\alpha,\beta})\mathcal{L}_{\text{sum}}^2(t)$ for all $t$,
which in turn implies (by first using Jensen's inequality and then taking the total expectation) that:
\begin{align}\label{eq:sketch6}
    \mathbb{E}[\mathcal{L}_{q}^{1/2}(t,K)]=o_K(1)+\mathcal{O}(c^{1/2}_{\alpha,\beta})\mathbb{E}[\mathcal{L}_{\text{sum}}(t)],\quad \forall\,t,
\end{align}
where we recall that $c_{\alpha,\beta}=\beta_k/\alpha_k$ is the stepsize ratio. The fact that we are able to get a factor of $\mathcal{O}(c^{1/2}_{\alpha,\beta})$ in front of $\mathbb{E}[\mathcal{L}_{\text{sum}}(t)]$ is crucial for the decoupling procedure in the next step.

\textbf{Step 3.} Taking total expectation on both sides of (\ref{eq:sketch2}) and then using the upper bound of $\mathbb{E}[\mathcal{L}^{1/2}_{q}(t,K)]$ we obtained in  (\ref{eq:sketch6}), we further get $\mathbb{E}[\mathcal{L}_{\text{sum}}(t+1)]\leq (\gamma +\mathcal{O}(c^{1/2}_{\alpha,\beta}))\mathbb{E}[\mathcal{L}_{\text{sum}}(t)]+o_K(1)$ for all $t$.
By choosing $c_{\alpha,\beta}$ so that $\mathcal{O}(c^{1/2}_{\alpha,\beta})\leq (1-\gamma)/2$, the previous inequality implies 
\begin{align*}
    \mathbb{E}[\mathcal{L}_{\text{sum}}(t+1)]\leq \left(1-\frac{1-\gamma}{2}\right)\mathbb{E}[\mathcal{L}_{\text{sum}}(t)]+o_K(1),\quad \forall\,t,
\end{align*}
which can be repeatedly used to obtain $\mathbb{E}[\mathcal{L}_{\text{sum}}(t)]= o_t(1)+o_K(1)$.
Substituting the previous bound on $\mathbb{E}[\mathcal{L}_{\text{sum}}(t)]$ into (\ref{eq:sketch5}), we obtain  $\max(\mathbb{E}[\mathcal{L}_{\pi}(t,K)],\mathbb{E}[\mathcal{L}_{q}(t,K)])= o_t(1)+o_K(1)$.

\textbf{Step 4.}
Substituting the bounds we obtained for $\mathbb{E}[\mathcal{L}_{\pi}(t,K)]$,  $\mathbb{E}[\mathcal{L}_{q}(t,K)]$, and $\mathbb{E}[\mathcal{L}_{\text{sum}}(t)]$ into (\ref{eq:sketch1}), and then repeatedly using the resulting inequality from $t=0$ to $t=T$, we obtain $\mathbb{E}[\mathcal{L}_{v}(T)]=o_T(1)+o_K(1)+\mathcal{O}(\tau)$.
Now that we have obtained finite-sample bounds for $\mathbb{E}[\mathcal{L}_{v}(T)]$, $\mathbb{E}[\mathcal{L}_{\text{sum}}(T)]$, $\mathbb{E}[\mathcal{L}_{\pi}(T,K)]$, and $\mathbb{E}[\mathcal{L}_{q}(T,K)]$; using them in (\ref{eq:sketch:overall}), we finally obtain the desired bound for the expected Nash gap.

\begin{remark}
    Looking back at the decoupling procedure, Steps $2$ and $3$ are crucial. In fact, in Step $1$, we already obtain a bound on $\mathbb{E}_t[\mathcal{L}_q(t,k)]$, where the additive error is $\mathcal{O}(1)\mathbb{E}[\mathcal{L}_{\text{sum}}(t)]$. However, directly using this bound on $\mathbb{E}_t[\mathcal{L}_q(t,k)]$ in (\ref{eq:sketch2}) would result in an expansive inequality for $\mathbb{E}[\mathcal{L}_{\text{sum}}(t)]$. By performing Step $2$, we are able to obtain a tighter bound for $\mathbb{E}_t[\mathcal{L}_q(t,k)]$, with the additive error being $\mathcal{O}(c_{\alpha,\beta}^{1/2})\mathbb{E}[\mathcal{L}_{\text{sum}}(t)]$. Furthermore, we can choose $c_{\alpha,\beta}$ to be small enough so that after using the bound from (\ref{eq:sketch6}) in (\ref{eq:sketch2}), the additive error $\mathcal{O}(c_{\alpha,\beta}^{1/2})\mathbb{E}[\mathcal{L}_{\text{sum}}(t)]$ is dominated by the negative drift.
\end{remark}

\noindent \textbf{Proof of Theorem~\ref{thm:stochastic_game}.}
The proof of Theorem~\ref{thm:stochastic_game} is completed by carrying out Steps~1 -- 4 described above with explicitly specified stepsizes. \hfill\qed

\section{Conclusion}\label{sec:conclusion}
We studied decentralized learning in two-player zero-sum matrix games and infinite-horizon
discounted stochastic games. In both settings, we established finite-sample 
guarantees for smoothed-best-response-based learning algorithms. For 
matrix games, the analysis yields an $\mathcal{O}(\epsilon^{-1})$ sample complexity for
finding an $\epsilon$-Nash distribution and, after introducing  explicit exploration, an 
$\tilde{\mathcal{O}}(\epsilon^{-8})$ sample complexity for  
finding an
$\epsilon$-Nash equilibrium. For stochastic games, the exploration-enhanced VI-SBR-based learning 
algorithm also achieves an $\tilde{\mathcal{O}}(\epsilon^{-8})$ sample complexity for finding an $\epsilon$-Nash equilibrium. 

The main technical contribution is a coupled Lyapunov-drift framework that handles several
features arising  simultaneously from decentralized learning in stochastic games: multiple interacting
stochastic iterates, auxiliary games that are not exactly zero-sum during learning, and
time-inhomogeneous Markovian noise generated by time-varying policies. This framework may
be useful for analyzing other learning algorithms with coupled stochastic
iterates.

Two directions remain particularly important as future work. First, focusing on analyzing natural best-response-type algorithms, the current bounds are unlikely to be
optimal, and improving the dependence on $\epsilon^{-1}$, the discount factor, and the
mixing parameters is worth further investigation. Second, the present analyses only focus on the tabular setting, and 
extending the algorithmic and analytical framework to the function approximation setting would be necessary for large-scale real-world applications.

\bibliographystyle{apalike}
\bibliography{references}

\begin{thebibliography}{}

\bibitem[Alacaoglu et~al., 2022]{alacaoglu2022natural}
Alacaoglu, A., Viano, L., He, N., and Cevher, V. (2022).
\newblock {A natural actor-critic framework for zero-sum Markov games}.
\newblock In {\em International Conference on Machine Learning}, pages
  307--366. PMLR.

\bibitem[Anonymous, 2023]{chen2024game_tabular}
Anonymous (2023).
\newblock {[Title omitted for double-blind review]}.
\newblock {\em {[Conference name omitted]}}.

\bibitem[Arslan and Y{\"u}ksel, 2017]{arslan2017decentralized}
Arslan, G. and Y{\"u}ksel, S. (2017).
\newblock {Decentralized $Q$-learning for stochastic teams and games}.
\newblock {\em IEEE Transactions on Automatic Control}, 62(4):1545--1558.

\bibitem[Bai and Jin, 2020]{bai2020provable}
Bai, Y. and Jin, C. (2020).
\newblock Provable self-play algorithms for competitive reinforcement learning.
\newblock In {\em International conference on machine learning}, pages
  551--560. PMLR.

\bibitem[Bai et~al., 2020]{bai2020near}
Bai, Y., Jin, C., and Yu, T. (2020).
\newblock Near-optimal reinforcement learning with self-play.
\newblock {\em Advances in neural information processing systems},
  33:2159--2170.

\bibitem[Banach, 1922]{banach1922operations}
Banach, S. (1922).
\newblock Sur les op{\'e}rations dans les ensembles abstraits et leur
  application aux {\'e}quations int{\'e}grales.
\newblock {\em Fund. math}, 3(1):133--181.

\bibitem[Baudin and Laraki, 2022a]{baudin2022fictitious}
Baudin, L. and Laraki, R. (2022a).
\newblock Fictitious play and best-response dynamics in identical interest and
  zero-sum stochastic games.
\newblock In {\em International Conference on Machine Learning}, pages
  1664--1690. PMLR.

\bibitem[Baudin and Laraki, 2022b]{baudinsmooth}
Baudin, L. and Laraki, R. (2022b).
\newblock Smooth fictitious play in stochastic games with perturbed payoffs and
  unknown transitions.
\newblock {\em Advances in Neural Information Processing Systems},
  35:20243--20256.

\bibitem[Beck, 2017]{beck2017first}
Beck, A. (2017).
\newblock {\em First-Order Methods in Optimization}, volume~25.
\newblock SIAM.

\bibitem[Bhandari et~al., 2018]{bhandari2018finite}
Bhandari, J., Russo, D., and Singal, R. (2018).
\newblock A finite-time analysis of temporal difference learning with linear
  function approximation.
\newblock In {\em Conference on learning theory}, pages 1691--1692. PMLR.

\bibitem[Bottou et~al., 2018]{bottou2018optimization}
Bottou, L., Curtis, F.~E., and Nocedal, J. (2018).
\newblock Optimization methods for large-scale machine learning.
\newblock {\em Siam Review}, 60(2):223--311.

\bibitem[Bowling and Veloso, 2001]{bowling2001rational}
Bowling, M. and Veloso, M. (2001).
\newblock Rational and convergent learning in stochastic games.
\newblock In {\em International Joint Conference on Artificial Intelligence},
  volume~17, pages 1021--1026.

\bibitem[Brown, 1951]{brown1951iterative}
Brown, G.~W. (1951).
\newblock Iterative solution of games by fictitious play.
\newblock {\em Activity Analysis of Production and Allocation}, 13(1):374--376.

\bibitem[Cai et~al., 2024]{cai2023uncoupled}
Cai, Y., Luo, H., Wei, C.-Y., and Zheng, W. (2024).
\newblock Uncoupled and convergent learning in two-player zero-sum {Markov}
  games with bandit feedback.
\newblock {\em Advances in Neural Information Processing Systems}, 36.

\bibitem[Cai et~al., 2026]{cai2026average}
Cai, Y., Luo, H., Wei, C.-Y., and Zheng, W. (2026).
\newblock From average-iterate to last-iterate convergence in games: A
  reduction and its applications.
\newblock {\em Advances in Neural Information Processing Systems},
  38:46937--46967.

\bibitem[Cesa-Bianchi and Lugosi, 2006]{cesa2006prediction}
Cesa-Bianchi, N. and Lugosi, G. (2006).
\newblock {\em Prediction, Learning, and Games}.
\newblock Cambridge University Press.

\bibitem[Chandak and Borkar, 2021]{chandak2021concentration}
Chandak, S. and Borkar, V.~S. (2021).
\newblock {Concentration of Contractive Stochastic Approximation and
  Reinforcement Learning}.
\newblock {\em Preprint arXiv:2106.14308}.

\bibitem[Chen et~al., 2021]{chenzy2021sample}
Chen, Z., Ma, S., and Zhou, Y. (2021).
\newblock Sample efficient stochastic policy extragradient algorithm for
  zero-sum {Markov} game.
\newblock In {\em International Conference on Learning Representations}.

\bibitem[Chen and Maguluri, 2026]{chen2026non}
Chen, Z. and Maguluri, S.~T. (2026).
\newblock Non-asymptotic convergence of stochastic iterative algorithms: A
  lyapunov framework.
\newblock {\em arXiv preprint arXiv:2605.31309}.

\bibitem[Chen et~al., 2023]{chen2021finite}
Chen, Z., Maguluri, S.~T., Shakkottai, S., and Shanmugam, K. (2023).
\newblock A {Lyapunov} theory for finite-sample guarantees of {Markovian}
  stochastic approximation.
\newblock {\em Operations Research}.

\bibitem[Danskin, 2012]{danskin2012theory}
Danskin, J.~M. (2012).
\newblock {\em {\textit{The Theory of Max-Min and Its Application to Weapons
  Allocation Problems}}}, volume~5.
\newblock Springer Science \& Business Media.

\bibitem[Daskalakis et~al., 2020]{daskalakis2020independent}
Daskalakis, C., Foster, D.~J., and Golowich, N. (2020).
\newblock Independent policy gradient methods for competitive reinforcement
  learning.
\newblock {\em Advances in neural information processing systems},
  33:5527--5540.

\bibitem[Faizal et~al., 2024]{faizal2024finite}
Faizal, F.~Z., Ozdaglar, A., and Wainwright, M.~J. (2024).
\newblock Finite-sample guarantees for learning dynamics in zero-sum polymatrix
  games.
\newblock {\em Preprint arXiv:2407.20128}.

\bibitem[Fiegel et~al., 2025]{fiegel2025harderpath}
Fiegel, C., Menard, P., Kozuno, T., Valko, M., and Perchet, V. (2025).
\newblock The harder path: Last iterate convergence for uncoupled learning in
  zero-sum games with bandit feedback.
\newblock In {\em International Conference on Machine Learning}, pages
  17131--17152. PMLR.

\bibitem[Fiegel et~al., 2026]{fiegel2026optimal}
Fiegel, C., Menard, P., Kozuno, T., Valko, M., and Perchet, V. (2026).
\newblock Optimal last-iterate convergence in matrix games with bandit feedback
  using the log-barrier.
\newblock {\em Preprint arXiv:2604.15242}.

\bibitem[Fudenberg and Kreps, 1993]{ref:Fudenberg93}
Fudenberg, D. and Kreps, D. (1993).
\newblock Learning mixed equilibria.
\newblock {\em Games and Economic Behavior}, 5:320--367.

\bibitem[Gao and Pavel, 2017]{gao2017properties}
Gao, B. and Pavel, L. (2017).
\newblock On the properties of the softmax function with application in game
  theory and reinforcement learning.
\newblock {\em Preprint arXiv:1704.00805}.

\bibitem[Govindan et~al., 2003]{govindan2003short}
Govindan, S., Reny, P.~J., Robson, A.~J., et~al. (2003).
\newblock {A short proof of Harsanyi's purification theorem}.
\newblock {\em Games and Economic Behavior}, 45(2):369--374.

\bibitem[Hofbauer and Hopkins, 2005]{hofbauer2005learning}
Hofbauer, J. and Hopkins, E. (2005).
\newblock Learning in perturbed asymmetric games.
\newblock {\em Games and Economic Behavior}, 52(1):133--152.

\bibitem[Hofbauer and Sandholm, 2002]{hofbauer2002global}
Hofbauer, J. and Sandholm, W.~H. (2002).
\newblock On the global convergence of stochastic fictitious play.
\newblock {\em Econometrica}, 70(6):2265--2294.

\bibitem[Hofbauer and Sorin, 2006]{hofbauer2006best}
Hofbauer, J. and Sorin, S. (2006).
\newblock Best response dynamics for continuous zero-sum games.
\newblock {\em Discrete and Continuous Dynamical Systems Series B}, 6(1):215.

\bibitem[Hu and Wellman, 2003]{hu2003nash}
Hu, J. and Wellman, M.~P. (2003).
\newblock Nash {Q}-learning for general-sum stochastic games.
\newblock {\em Journal of Machine Learning Research}, 4(Nov):1039--1069.

\bibitem[Jin et~al., 2023]{jin2021v}
Jin, C., Liu, Q., Wang, Y., and Yu, T. (2023).
\newblock V-learning—a simple, efficient, decentralized algorithm for
  multiagent reinforcement learning.
\newblock {\em Mathematics of Operations Research}.

\bibitem[Khodadadian et~al., 2022]{khodadadian2021finite}
Khodadadian, S., Doan, T.~T., Romberg, J., and Maguluri, S.~T. (2022).
\newblock {Finite sample analysis of two-time-scale natural actor-critic
  algorithm}.
\newblock {\em IEEE Transactions on Automatic Control}.

\bibitem[Lan, 2020]{lan2020first}
Lan, G. (2020).
\newblock {\em {First-order and Stochastic Optimization Methods for Machine
  Learning}}.
\newblock Springer.

\bibitem[Leslie and Collins, 2003]{leslie2003two}
Leslie, D.~S. and Collins, E.~J. (2003).
\newblock Convergent multiple-timescales reinforcement learning algorithms in
  normal form games.
\newblock {\em The Annals of Applied Probability}, 13(4):1231--1251.

\bibitem[Leslie and Collins, 2005]{leslie2005individual}
Leslie, D.~S. and Collins, E.~J. (2005).
\newblock {Individual $Q$-learning in normal form games}.
\newblock {\em SIAM Journal on Control and Optimization}, 44(2):495--514.

\bibitem[Levin and Peres, 2017]{levin2017markov}
Levin, D.~A. and Peres, Y. (2017).
\newblock {\em Markov Chains and Mixing Times}, volume 107.
\newblock American Mathematical Soc.

\bibitem[Littman, 1994]{littman1994markov}
Littman, M.~L. (1994).
\newblock Markov games as a framework for multi-agent reinforcement learning.
\newblock In {\em Proceedings of the Eleventh International Conference on
  International Conference on Machine Learning}, ICML'94, page 157–163, San
  Francisco, CA, USA. Morgan Kaufmann Publishers Inc.

\bibitem[Liu et~al., 2021]{liu2020sharp}
Liu, Q., Yu, T., Bai, Y., and Jin, C. (2021).
\newblock A sharp analysis of model-based reinforcement learning with
  self-play.
\newblock In {\em International Conference on Machine Learning}, pages
  7001--7010. PMLR.

\bibitem[Mao et~al., 2022]{mao2022improving}
Mao, W., Yang, L., Zhang, K., and Ba\c{s}ar, T. (2022).
\newblock On improving model-free algorithms for decentralized multi-agent
  reinforcement learning.
\newblock In {\em International Conference on Machine Learning}, pages
  15007--15049. PMLR.

\bibitem[McKelvey and Palfrey, 1995]{mckelvey1995quantal}
McKelvey, R.~D. and Palfrey, T.~R. (1995).
\newblock Quantal response equilibria for normal form games.
\newblock {\em Games and economic behavior}, 10(1):6--38.

\bibitem[Mnih et~al., 2013]{mnih2013playing}
Mnih, V., Kavukcuoglu, K., Silver, D., Graves, A., Antonoglou, I., Wierstra,
  D., and Riedmiller, M. (2013).
\newblock Playing atari with deep reinforcement learning.
\newblock {\em Preprint arXiv:1312.5602}.

\bibitem[Qu et~al., 2020]{qu2020scalable}
Qu, G., Wierman, A., and Li, N. (2020).
\newblock Scalable reinforcement learning of localized policies for multi-agent
  networked systems.
\newblock In {\em Learning for Dynamics and Control}, pages 256--266. PMLR.

\bibitem[Robinson, 1951]{robinson1951iterative}
Robinson, J. (1951).
\newblock An iterative method of solving a game.
\newblock {\em Annals of Mathematics}, pages 296--301.

\bibitem[Sayin et~al., 2021]{sayin2021decentralized}
Sayin, M., Zhang, K., Leslie, D., Basar, T., and Ozdaglar, A. (2021).
\newblock {Decentralized $Q$-learning in zero-sum Markov games}.
\newblock {\em Advances in Neural Information Processing Systems},
  34:18320--18334.

\bibitem[Sayin et~al., 2022]{sayin2022fictitious}
Sayin, M.~O., Parise, F., and Ozdaglar, A. (2022).
\newblock Fictitious play in zero-sum stochastic games.
\newblock {\em SIAM Journal on Control and Optimization}, 60(4):2095--2114.

\bibitem[Shamma and Arslan, 2004]{shamma2004unified}
Shamma, J.~S. and Arslan, G. (2004).
\newblock Unified convergence proofs of continuous-time fictitious play.
\newblock {\em IEEE Transactions on Automatic Control}, 49(7):1137--1141.

\bibitem[Shapley, 1953]{shapley1953stochastic}
Shapley, L.~S. (1953).
\newblock Stochastic games.
\newblock {\em Proceedings of the National Academy of Sciences},
  39(10):1095--1100.

\bibitem[Song et~al., 2022]{songcan}
Song, Z., Mei, S., and Bai, Y. (2022).
\newblock When can we learn general-sum {Markov} games with a large number of
  players sample-efficiently?
\newblock In {\em International Conference on Learning Representations}.

\bibitem[Srikant and Ying, 2019]{srikant2019finite}
Srikant, R. and Ying, L. (2019).
\newblock Finite-time error bounds for linear stochastic approximation and {TD}
  learning.
\newblock In {\em Conference on Learning Theory}, pages 2803--2830.

\bibitem[Sutton, 1988]{sutton1988learning}
Sutton, R.~S. (1988).
\newblock Learning to predict by the methods of temporal differences.
\newblock {\em Machine learning}, 3(1):9--44.

\bibitem[Sutton and Barto, 2018]{sutton2018reinforcement}
Sutton, R.~S. and Barto, A.~G. (2018).
\newblock {\em {Reinforcement learning: An introduction}}.
\newblock MIT press.

\bibitem[Tsitsiklis, 1994]{tsitsiklis1994asynchronous}
Tsitsiklis, J.~N. (1994).
\newblock Asynchronous stochastic approximation and {$Q$}-learning.
\newblock {\em Machine learning}, 16(3):185--202.

\bibitem[Tsitsiklis and Van~Roy, 1997]{tsitsiklis1997analysis}
Tsitsiklis, J.~N. and Van~Roy, B. (1997).
\newblock An analysis of temporal-difference learning with function
  approximation.
\newblock {\em IEEE transactions on automatic control}, 42(5):674--690.

\bibitem[Wang et~al., 2023]{wang2023adversarial}
Wang, T.~T., Gleave, A., Tseng, T., Pelrine, K., Belrose, N., Miller, J.,
  Dennis, M.~D., Duan, Y., Pogrebniak, V., Levine, S., and Russell, S. (2023).
\newblock Adversarial policies beat superhuman go ais.
\newblock In {\em Proceedings of the 40th International Conference on Machine
  Learning}, ICML'23. JMLR.org.

\bibitem[Wei et~al., 2021]{wei2021last}
Wei, C.-Y., Lee, C.-W., Zhang, M., and Luo, H. (2021).
\newblock Last-iterate convergence of decentralized optimistic gradient
  descent/ascent in infinite-horizon competitive {M}arkov games.
\newblock In {\em Conference on Learning Theory}, pages 4259--4299. PMLR.

\bibitem[Xie et~al., 2020]{xie2020learning}
Xie, Q., Chen, Y., Wang, Z., and Yang, Z. (2020).
\newblock Learning zero-sum simultaneous-move {M}arkov games using function
  approximation and correlated equilibrium.
\newblock In {\em Conference on Learning Theory}, pages 3674--3682. PMLR.

\bibitem[Zhang et~al., 2021a]{zhang2021multi}
Zhang, K., Yang, Z., and Ba{\c{s}}ar, T. (2021a).
\newblock Multi-agent reinforcement learning: {A} selective overview of
  theories and algorithms.
\newblock {\em Handbook of Reinforcement Learning and Control}, pages 321--384.

\bibitem[Zhang et~al., 2018]{zhang2018fully}
Zhang, K., Yang, Z., Liu, H., Zhang, T., and Ba\c{s}ar, T. (2018).
\newblock Fully decentralized multi-agent reinforcement learning with networked
  agents.
\newblock In {\em International Conference on Machine Learning}, pages
  5867--5876.

\bibitem[Zhang et~al., 2021b]{zhang2021derivative}
Zhang, K., Zhang, X., Hu, B., and Ba\c{s}ar, T. (2021b).
\newblock Derivative-free policy optimization for linear risk-sensitive and
  robust control design: {I}mplicit regularization and sample complexity.
\newblock {\em Advances in Neural Information Processing Systems},
  34:2949--2964.

\bibitem[Zhang et~al., 2022]{zhang2021global}
Zhang, S., Tachet, R., and Laroche, R. (2022).
\newblock Global optimality and finite sample analysis of softmax off-policy
  actor critic under state distribution mismatch.
\newblock {\em Journal of Machine Learning Research}, 23(343):1--91.

\bibitem[Zhang et~al., 2023]{zhang2022global}
Zhang, Y., Qu, G., Xu, P., Lin, Y., Chen, Z., and Wierman, A. (2023).
\newblock Global convergence of localized policy iteration in networked
  multi-agent reinforcement learning.
\newblock {\em Proceedings of the ACM on Measurement and Analysis of Computing
  Systems}, 7(1):1--51.

\bibitem[Zhao et~al., 2022]{zhao2021provably}
Zhao, Y., Tian, Y., Lee, J., and Du, S. (2022).
\newblock {Provably efficient policy optimization for two-player zero-sum
  Markov games}.
\newblock In {\em International Conference on Artificial Intelligence and
  Statistics}, pages 2736--2761. PMLR.

\bibitem[Zou et~al., 2019]{zou2019finite}
Zou, S., Xu, T., and Liang, Y. (2019).
\newblock Finite-sample analysis for {SARSA} with linear function
  approximation.
\newblock In {\em Advances in Neural Information Processing Systems}, pages
  8668--8678.

\end{thebibliography}

\begin{center}
    {\LARGE\bfseries Appendices}
\end{center}

\appendix
\section{Proof of Theorem \ref{thm:matrix_fast}}\label{ap:proof_matrix_game}

The proof is divided into $4$ steps. In Appendix \ref{ap:bounded_matrix}, we prove an important boundedness property for the iterates generated by Algorithm \ref{algo:matrix_fast}. In Appendices \ref{ap:matrix_policy} and \ref{ap:matrix_qfunction}, we analyze the evolution of the policies and the $q$-functions by establishing negative drift inequalities with respect to their associated Lyapunov functions. In Appendix \ref{ap:matrix_solving_recursion}, we solve the coupled Lyapunov drift inequalities to prove Theorem \ref{thm:matrix_fast}. The proofs of Corollary \ref{co:sample_matrix_matrix_fast} and Corollary \ref{co:sample_complexity_matrix_exponential} are presented in Appendices \ref{ap:pf:sc_matrix_fast} and \ref{pf:co:sample_complexity_matrix_exponential}, respectively. The statements and proofs of all supporting lemmas used in this section are presented in Appendix \ref{ap:matrix_lemmas}. 

\subsection{Boundedness of the Iterates}\label{ap:bounded_matrix}
The following lemma presents the boundedness property.

\begin{lemma}\label{le:boundedness_matrix}
It holds for all $k\geq 0$ and $i\in \{1,2\}$ that $\|q_k^i\|_\infty\leq 1$ and $\min_{a^i\in\mathcal{A}^i}\pi_k^i(a^i)\geq \ell_\tau$,
where $\ell_\tau=[(m-1)\exp(2/\tau)+1]^{-1}$.
\end{lemma}

\proof{Proof of Lemma~\ref{le:boundedness_matrix}.}
We prove the result by induction. Since $q_0^i=0$ and $\pi_0^i$ is initialized as the uniform distribution on $\mathcal{A}^i$, the base case holds. Suppose that the result holds for some $k\geq 0$. By Line~5 of Algorithm~\ref{algo:matrix_fast}, for any $a^i\in\mathcal{A}^i$,
\begin{align*}
    |q_{k+1}^i(a^i)|
    =\,&
    \left|
    (1-\alpha_k \mathds{1}_{\{a^i=A_k^i\}})q_k^i(a^i)
    +
    \alpha_k\mathds{1}_{\{a^i=A_k^i\}} R_i(A_k^i,A_k^{-i})
    \right|\\
    \leq\,&
    \max\left\{
    |q_k^i(a^i)|,\,
    (1-\alpha_k)|q_k^i(a^i)|
    +
    \alpha_k |R_i(A_k^i,A_k^{-i})|
    \right\}\\
    \leq\,&1,
\end{align*}
where the last inequality follows from the induction hypothesis $\|q_k^i\|_\infty\leq 1$ and the bound $|R_i(a^i,a^{-i})|\leq 1$. Hence, $\|q_{k+1}^i\|_\infty\leq 1$.

Next, by Line~3 of Algorithm~\ref{algo:matrix_fast}, for any $a^i\in\mathcal{A}^i$,
\begin{align*}
    \pi_{k+1}^i(a^i)
    =\,&
    (1-\beta_k)\pi_k^i(a^i)
    +
    \beta_k[\sigma_\tau(q_k^i)](a^i)\\
    \geq\,&
    (1-\beta_k)\ell_\tau
    +
    \frac{\beta_k}{(m-1)\exp(2\|q_k^i\|_\infty/\tau)+1}
    \tag{Lemma~\ref{le:softmax_bound}}\\
    \geq\,&
    (1-\beta_k)\ell_\tau+\beta_k\ell_\tau
    \tag{$\|q_k^i\|_\infty\leq 1$}\\
    =\,&
    \ell_\tau.
\end{align*}
The induction is complete. \hfill\qed
\endproof

\subsection{Analysis of the Policies}\label{ap:matrix_policy}
Let $V_R:\Delta(\mathcal{A}^1)\times \Delta(\mathcal{A}^2)\to \mathbb{R}$ be defined as
\begin{align}\label{def:V_R-matrix}
    V_R(\mu^1,\mu^2)
    =
    \sum_{i=1,2}
    \max_{\hat{\mu}^i\in\Delta(\mathcal{A}^i)}
    \left\{
    (\hat{\mu}^i-\mu^i)^\top R_i\mu^{-i}
    +\tau \nu(\hat{\mu}^i)-\tau\nu(\mu^i)
    \right\},
\end{align}
where $\nu(\cdot)$ is the Shannon entropy. For simplicity of notation, we use $\nabla_1 V_R(\cdot,\cdot)$ and $\nabla_2 V_R(\cdot,\cdot)$ to represent the gradients with respect to the first and second arguments of $V_R(\cdot,\cdot)$, respectively. A sequence of properties regarding $V_R(\cdot,\cdot)$ are provided in Lemma \ref{le:properties_Lyapunov_matrix}.

 Next, we present the negative drift inequality of the policies generated by Algorithm \ref{algo:matrix_fast} with respect to the Lyapunov function $V_R(\cdot,\cdot)$.

\begin{lemma}\label{le:policy_matrix_fast}
    It holds for all $k\geq 0$ that
    \begin{align*}
        \mathbb{E}[V_R(\pi_{k+1}^1,\pi_{k+1}^2)]
        \leq\,&
        \left(1-\frac{\beta_k}{2}\right)\mathbb{E}[V_R(\pi_k^1,\pi_k^2)]
        +\frac{\ell_\tau\alpha_k}{4}\sum_{i=1,2}\mathbb{E}[\|q_k^i-R_i\pi_k^{-i} \|_2^2]
        +2L_\tau\beta_k^2,
    \end{align*}
    where $L_\tau=\tau/\ell_\tau+m^2/\tau$.
\end{lemma}

\proof{Proof of Lemma~\ref{le:policy_matrix_fast}.}
By Lemma~\ref{le:boundedness_matrix}, we have $(\pi_k^1,\pi_k^2)\in\Pi_\tau$ for all $k\geq 0$. Using the smoothness property of $V_R(\cdot,\cdot)$ in Lemma~\ref{le:properties_Lyapunov_matrix} (1) and the update equation in Algorithm~\ref{algo:matrix_fast}, Line~3, we have
\begin{align*}
    V_R(\pi_{k+1}^1,\pi_{k+1}^2)
    \leq\,&
    V_R(\pi_k^1,\pi_k^2)
    +\langle \nabla_1V_R(\pi_k^1,\pi_k^2),\pi_{k+1}^1-\pi_k^1 \rangle
    +\langle \nabla_2V_R(\pi_k^1,\pi_k^2),\pi_{k+1}^2-\pi_k^2 \rangle\\
    &+\frac{L_\tau}{2}\sum_{i=1,2}\|\pi_{k+1}^i-\pi_k^i\|_2^2\\
    =\,&
    V_R(\pi_k^1,\pi_k^2)
    +\beta_k\langle \nabla_1V_R(\pi_k^1,\pi_k^2),\sigma_\tau(q_k^1)-\pi_k^1 \rangle\\
    &+\beta_k\langle \nabla_2V_R(\pi_k^1,\pi_k^2),\sigma_\tau(q_k^2)-\pi_k^2 \rangle
    +\frac{L_\tau\beta_k^2}{2}\sum_{i=1,2}\|\sigma_\tau(q_k^i)-\pi_k^i\|_2^2\\
    \leq\,&
    V_R(\pi_k^1,\pi_k^2)
    +\beta_k\langle \nabla_1 V_R(\pi_k^1,\pi_k^2),\sigma_\tau(R_1\pi_k^2)-\pi_k^1 \rangle\\
    &+\beta_k\langle \nabla_2V_R(\pi_k^1,\pi_k^2),\sigma_\tau(R_2\pi_k^1)-\pi_k^2 \rangle\\
    &+\beta_k\langle \nabla_1 V_R(\pi_k^1,\pi_k^2),\sigma_\tau(q_k^1)-\sigma_\tau(R_1\pi_k^2) \rangle\\
    &+\beta_k\langle \nabla_2V_R(\pi_k^1,\pi_k^2),\sigma_\tau(q_k^2)-\sigma_\tau(R_2\pi_k^1) \rangle
    +2L_\tau\beta_k^2\\
    \leq\,&
    \left(1-\frac{\beta_k}{2}\right)V_R(\pi_k^1,\pi_k^2)
    +4\beta_k\left(\frac{1}{\tau \ell_\tau^2}+\frac{m^2}{\tau^3}\right)
    \sum_{i=1,2}\|q_k^i-R_i\pi_k^{-i} \|_2^2
    +2L_\tau\beta_k^2,
\end{align*}
where the last line follows from Lemma~\ref{le:properties_Lyapunov_matrix} (2) and (3), and we used
$\sum_{i=1,2}\|\sigma_\tau(q_k^i)-\pi_k^i\|_2^2\leq 4$.

Taking expectations on both sides and using
$c_{\alpha,\beta}=\frac{\beta_k}{\alpha_k}\leq \min\left\{\frac{\tau \ell_\tau^3}{32},\frac{\ell_\tau \tau^3}{32m^2}\right\}$, we obtain
\begin{align*}
    \mathbb{E}[V_R(\pi_{k+1}^1,\pi_{k+1}^2)]
    \leq\,&
    \left(1-\frac{\beta_k}{2}\right)\mathbb{E}[V_R(\pi_k^1,\pi_k^2)]
    +\frac{\ell_\tau\alpha_k}{4}\sum_{i=1,2}\mathbb{E}[\|q_k^i-R_i\pi_k^{-i} \|_2^2]
    +2L_\tau\beta_k^2.
\end{align*}
This completes the proof. \hfill\qed
\endproof

\subsection{Analysis of the q-Functions}\label{ap:matrix_qfunction}
For $i\in \{1,2\}$, let $F^i:\mathbb{R}^{m_i}\times \mathcal{A}^i\times \mathcal{A}^{-i}\to \mathbb{R}^{m_i}$ be an operator defined as
\begin{align*}
    [F^i(q^i,a_0^i,a_0^{-i})](a^i)=\mathds{1}_{\{a_0^i=a^i\}} \left(R_i(a_0^i,a_0^{-i})-q^i(a_0^i)\right),\quad \forall\,(q^i,a_0^i,a_0^{-i}) \text{ and }a^i.
\end{align*}
Then, Line $5$ of Algorithm \ref{algo:matrix_fast} can be compactly written as 
\begin{align}\label{eq:q_SA_matrix}          q_{k+1}^i=q_k^i+\alpha_kF^i(q_k^i,A_k^i,A_k^{-i}).
\end{align}
Given a joint policy $(\pi^1,\pi^2)$, let $\bar{F}_\pi^i:\mathbb{R}^{m_i}\to \mathbb{R}^{m_i}$ be defined as
\begin{align*}
    \bar{F}_\pi^i(q^i):=\mathbb{E}_{A^i\sim \pi^i(\cdot),A^{-i}\sim \pi^{-i}(\cdot)}[F^i(q^i,A^i,A^{-i})]=\text{diag}(\pi^i)(R_i\pi^{-i}-q^i).
\end{align*}
Then,  (\ref{eq:q_SA_matrix}) can be viewed as a stochastic approximation algorithm for tracking the solution of the time-varying equation $\bar{F}_{\pi_k}^i(q^i)=0$. We next present the negative drift inequality of the $q$-functions generated by Algorithm \ref{algo:matrix_fast} with respect to a norm-square Lyapunov function.

\begin{lemma}\label{le:q-function-drift-matrix_fast}
	The following inequality holds for all $k\geq 0$:
	\begin{align*}
    \sum_{i=1,2}\mathbb{E}[\|q_{k+1}^i-R_i\pi_{k+1}^{-i}\|_2^2]
	\leq 
 \left(1- \frac{\ell_\tau\alpha_k}{2}\right)\sum_{i=1,2}\mathbb{E}[\|q_k^i-R_i\pi_k^{-i}\|_2^2]+\frac{\beta_k}{4}\mathbb{E}[V_R(\pi_k^1,\pi_k^2)]+16\alpha_k^2.
\end{align*}
\end{lemma}

\proof{Proof of Lemma~\ref{le:q-function-drift-matrix_fast}.}
For any $k\geq 0$ and $i\in \{1,2\}$, define
\[
    \Delta_k^i=q_k^i-R_i\pi_k^{-i},
    \qquad
    \widetilde{\Delta}_k^i=q_k^i-R_i\pi_{k+1}^{-i}.
\]
Since $A_k^i\sim \pi_{k+1}^i(\cdot)$ and $A_k^{-i}\sim \pi_{k+1}^{-i}(\cdot)$, we have
\begin{align*}
    \bar{F}_{\pi_{k+1}}^i(q_k^i)
    =
    \text{diag}(\pi_{k+1}^i)(R_i\pi_{k+1}^{-i}-q_k^i).
\end{align*}
Using the update equation in Line~5 of Algorithm~\ref{algo:matrix_fast}, we obtain
\begin{align*}
    \mathbb{E}[\|q_{k+1}^i-R_i\pi_{k+1}^{-i}\|_2^2]
    =\,&
    \mathbb{E}[\|\widetilde{\Delta}_k^i+\alpha_k F^i(q_k^i,A_k^i,A_k^{-i})\|_2^2]\\
    \leq\,&
    \mathbb{E}[\|\widetilde{\Delta}_k^i\|_2^2]
    +2\alpha_k\mathbb{E}[\langle \bar{F}_{\pi_{k+1}}^i(q_k^i),\widetilde{\Delta}_k^i\rangle]
    +4\alpha_k^2\\
    \leq\,&
    (1-2\ell_\tau\alpha_k)\mathbb{E}[\|\widetilde{\Delta}_k^i\|_2^2]
    +4\alpha_k^2,
\end{align*}
where the last inequality follows from Lemma~\ref{le:boundedness_matrix} and
$\mathbb{E}[\|F^i(q_k^i,A_k^i,A_k^{-i})\|_2^2]\leq 4$.

Next, by Line~3 of Algorithm~\ref{algo:matrix_fast},
\begin{align*}
    \widetilde{\Delta}_k^i
    =
    \Delta_k^i
    +
    R_i(\pi_k^{-i}-\pi_{k+1}^{-i})
    =
    \Delta_k^i
    -
    \beta_k R_i(\sigma_\tau(q_k^{-i})-\pi_k^{-i}).
\end{align*}
Using the inequality $\|x+y\|_2^2\leq (1+\eta)\|x\|_2^2+(1+\eta^{-1})\|y\|_2^2$ with $\eta=\ell_\tau\alpha_k/2$, we have
\begin{align*}
    (1-2\ell_\tau\alpha_k)\|\widetilde{\Delta}_k^i\|_2^2
    \leq\,&
    \left(1-\frac{3\ell_\tau\alpha_k}{2}\right)\|\Delta_k^i\|_2^2
    +
    \frac{3\beta_k^2}{\ell_\tau\alpha_k}
    \|R_i(\sigma_\tau(q_k^{-i})-\pi_k^{-i})\|_2^2,
\end{align*}
where we used the stepsize condition $\ell_\tau\alpha_k\leq 1$. Therefore,
\begin{align*}
    \mathbb{E}[\|q_{k+1}^i-R_i\pi_{k+1}^{-i}\|_2^2]
    \leq\,&
    \left(1-\frac{3\ell_\tau\alpha_k}{2}\right)
    \mathbb{E}[\|q_k^i-R_i\pi_k^{-i}\|_2^2]\\
    &+
    \frac{3m^2\beta_k^2}{\ell_\tau\alpha_k}
    \mathbb{E}[\|\sigma_\tau(q_k^{-i})-\pi_k^{-i}\|_2^2]
    +4\alpha_k^2.
\end{align*}
Moreover,
\begin{align*}
    \mathbb{E}[\|\sigma_\tau(q_k^{-i})-\pi_k^{-i}\|_2^2]
    \leq\,&
    2\mathbb{E}[\|\sigma_\tau(q_k^{-i})-\sigma_\tau(R_{-i}\pi_k^i)\|_2^2]
    +
    2\mathbb{E}[\|\sigma_\tau(R_{-i}\pi_k^i)-\pi_k^{-i}\|_2^2]\\
    \leq\,&
    \frac{2}{\tau^2}
    \mathbb{E}[\|q_k^{-i}-R_{-i}\pi_k^i\|_2^2]
    +
    \frac{4}{\tau}\mathbb{E}[V_R(\pi_k^1,\pi_k^2)].
\end{align*}
where we used the $1/\tau$-Lipschitz continuity of $\sigma_\tau(\cdot)$ and Lemma~\ref{le:quadratic_growth}. Hence,
\begin{align*}
    \mathbb{E}[\|q_{k+1}^i-R_i\pi_{k+1}^{-i}\|_2^2]
    \leq\,&
    \left(1-\frac{3\ell_\tau\alpha_k}{2}\right)
    \mathbb{E}[\|q_k^i-R_i\pi_k^{-i}\|_2^2]\\
    &+
    \frac{6m^2\beta_k^2}{\ell_\tau\tau^2\alpha_k}
    \mathbb{E}[\|q_k^{-i}-R_{-i}\pi_k^i\|_2^2]\\
    &+
    \frac{12m^2\beta_k^2}{\ell_\tau\tau\alpha_k}
    \mathbb{E}[V_R(\pi_k^1,\pi_k^2)]
    +4\alpha_k^2.
\end{align*}
Summing over $i=1,2$, we obtain
\begin{align*}
    \sum_{i=1,2}\mathbb{E}[\|q_{k+1}^i-R_i\pi_{k+1}^{-i}\|_2^2]
    \leq\,&
    \left(
    1-\frac{3\ell_\tau\alpha_k}{2}
    +
    \frac{6m^2\beta_k^2}{\ell_\tau\tau^2\alpha_k}
    \right)
    \sum_{i=1,2}\mathbb{E}[\|q_k^i-R_i\pi_k^{-i}\|_2^2]\\
    &+
    \frac{24m^2\beta_k^2}{\ell_\tau\tau\alpha_k}
    \mathbb{E}[V_R(\pi_k^1,\pi_k^2)]
    +8\alpha_k^2\\
    \leq\,&
    \left(1-\frac{\ell_\tau\alpha_k}{2}\right)
    \sum_{i=1,2}\mathbb{E}[\|q_k^i-R_i\pi_k^{-i}\|_2^2]
    +
    \frac{\beta_k}{4}\mathbb{E}[V_R(\pi_k^1,\pi_k^2)]
    +16\alpha_k^2,
\end{align*}
where the last line follows from Condition~\ref{con:stepsize_matrix}. \hfill\qed
\endproof

\subsection{Solving Coupled Lyapunov Drift Inequalities}\label{ap:matrix_solving_recursion}

For simplicity of notation, denote
$\mathcal{L}_q(k)=\sum_{i=1,2}\mathbb{E}[\|q_k^i-R_i\pi_k^{-i}\|_2^2]$ and
$\mathcal{L}_\pi(k)=\mathbb{E}[V_R(\pi_k^1,\pi_k^2)]$. Then, Lemmas~\ref{le:policy_matrix_fast} and~\ref{le:q-function-drift-matrix_fast} imply that
\begin{align*}
     \mathcal{L}_\pi(k+1)
	\leq \,&
    \left(1-\frac{\beta_k}{2}\right)\mathcal{L}_\pi(k)
    +\frac{\ell_\tau\alpha_k}{4}\mathcal{L}_q(k)
    +2L_\tau\beta_k^2,\\
    \mathcal{L}_q(k+1)
	\leq\,&
    \left(1- \frac{\ell_\tau\alpha_k}{2}\right)\mathcal{L}_q(k)
    +\frac{\beta_k}{4}\mathcal{L}_\pi(k)
    +16\alpha_k^2,\quad \forall\,k\geq 0.
\end{align*}
Adding the previous two inequalities, we obtain
\begin{align}
    \mathcal{L}_q(k+1)+\mathcal{L}_\pi(k+1)
	\leq \,&
    \left(1-\frac{\beta_k}{4}\right)\mathcal{L}_\pi(k)
    +
    \left(1- \frac{\ell_\tau\alpha_k}{4}\right)\mathcal{L}_q(k)
    +2L_\tau\beta_k^2+16\alpha_k^2\nonumber\\
    \leq \,&
    \left(1-\frac{\beta_k}{4}\right)(\mathcal{L}_\pi(k)+\mathcal{L}_q(k))
    +2L_\tau\beta_k^2+16\alpha_k^2,
    \label{matrix_recursion_fast}
\end{align}
where the second inequality follows from $\beta_k\leq \ell_\tau\alpha_k$, or equivalently $c_{\alpha,\beta}\leq \ell_\tau$; see Condition~\ref{con:stepsize_matrix}.

\subsubsection{Constant Stepsizes}

When using constant stepsizes, i.e., $\alpha_k\equiv \alpha$ and $\beta_k\equiv \beta$, iterating \eqref{matrix_recursion_fast} gives, for all $k\geq 0$,
\begin{align*}
    \mathcal{L}_q(k)+\mathcal{L}_\pi(k)
    \leq \,&
    \left(1-\frac{\beta}{4}\right)^k
    (\mathcal{L}_\pi(0)+\mathcal{L}_q(0))
    +8L_\tau\beta+\frac{64\alpha^2}{\beta}\\
    \leq \,&
    \left(1-\frac{\beta}{4}\right)^k
    (4+2\tau \log(m)+2m)
    +8L_\tau\beta+\frac{64\alpha^2}{\beta}\\
    =\,&
    B_{\text{in}}\left(1-\frac{\beta}{4}\right)^k
    +8L_\tau\beta+\frac{64\alpha}{c_{\alpha,\beta}},
\end{align*}
where the second inequality follows from $\mathcal{L}_\pi(0)\leq 4+2\tau \log(m)$ and $\mathcal{L}_q(0)\leq 2m$. Theorem~\ref{thm:matrix_fast} (1) follows by observing that
$\mathcal{L}_q(k)+\mathcal{L}_\pi(k)\geq \mathcal{L}_\pi(k)=\mathbb{E}[\text{NG}_\tau(\pi_k^1,\pi_k^2)]$.

\subsubsection{Diminishing Stepsizes}

Consider using $\alpha_k=\alpha/(k+h)$ and $\beta_k=\beta/(k+h)$, where $\beta=c_{\alpha,\beta}\alpha$. Recursions of the form \eqref{matrix_recursion_fast} have been well studied in the literature on convergence rates of iterative algorithms \citep{lan2020first,srikant2019finite,chen2021finite}. Since $\beta>4$, using the same line of analysis as in \cite[Appendix A.2]{chen2021finite}, we have
\begin{align*}
   \mathbb{E}[\text{NG}_\tau(\pi_k^1,\pi_k^2)]=  \mathcal{L}_\pi(k)\leq \mathcal{L}_q(k)+\mathcal{L}_\pi(k)
    \leq\,&
    B_{\text{in}}\left(\frac{h}{k+h}\right)^{\beta/4}
    +
    \left(64e L_\tau\beta +\frac{512e \alpha}{c_{\alpha,\beta}}\right)
    \frac{1}{k+h}.
\end{align*}

\subsection{Proof of Corollary \ref{co:sample_matrix_matrix_fast}}\label{ap:pf:sc_matrix_fast}

We use Theorem~\ref{thm:matrix_fast} (1) to derive the sample complexity, and choose
$\beta=c_{\alpha,\beta}\alpha$ with $c_{\alpha,\beta}$ satisfying Condition~\ref{con:stepsize_matrix}. To achieve
$\mathbb{E}[\text{NG}_\tau(\pi_K^1,\pi_K^2)]\leq \epsilon$, in view of Theorem~\ref{thm:matrix_fast} (1), it is sufficient that
\begin{align*}
   B_{\text{in}}e^{-\beta K/4}\leq \frac{\epsilon}{3},\quad
   8L_\tau\beta\leq \frac{\epsilon}{3},\quad
   \frac{64\alpha}{c_{\alpha,\beta}}\leq \frac{\epsilon}{3}.
\end{align*}
The first inequality is satisfied when $K\geq 4\log(3B_{\text{in}}/\epsilon)/\beta$. Since $\alpha=\beta/c_{\alpha,\beta}$, the last two inequalities are satisfied when
$\beta\leq \min\left\{\frac{\epsilon}{24L_\tau},\frac{\epsilon c_{\alpha,\beta}^2}{192}\right\}$. Therefore,
$\mathbb{E}[\text{NG}_\tau(\pi_K^1,\pi_K^2)]\leq \epsilon$ holds as long as
\begin{align*}
    K\geq
    \frac{96\log(3B_{\text{in}}/\epsilon)}{\epsilon}
    \max\left\{L_\tau,\frac{8}{c_{\alpha,\beta}^2}\right\}.
\end{align*}
The result follows by observing that
$\max\{L_\tau,8/c_{\alpha,\beta}^2\}\leq 8L_\tau/c_{\alpha,\beta}^2$.

\subsection{Proof of Corollary~\ref{co:sample_complexity_matrix_exponential}}\label{pf:co:sample_complexity_matrix_exponential}

In view of \eqref{eq:NG_RNG_main}, for the output $(\pi_K^1,\pi_K^2)$ of Algorithm~\ref{algo:matrix_fast}, we have
\begin{align*}
    \mathbb{E}[\text{NG}(\pi_K^1,\pi_K^2)]
    \leq
    \mathbb{E}[\text{NG}_\tau(\pi_K^1,\pi_K^2)]
    +
    2\tau\log(m).
\end{align*}
Therefore, given $\epsilon>0$, to achieve
$\mathbb{E}[\text{NG}(\pi_K^1,\pi_K^2)]\leq \epsilon$, it is sufficient that
$\mathbb{E}[\text{NG}_\tau(\pi_K^1,\pi_K^2)]\leq\epsilon/2$ and
$2\tau \log(m)\leq \epsilon/2$. The previous two inequalities hold as long as
\begin{align}\label{eq:sc:matrix_slow}
   \tau\leq \frac{\epsilon}{4\log(m)}, \quad \text{and}\quad
   K\geq \frac{3200L_\tau\log(3B_{\text{in}}/\epsilon)}{c_{\alpha,\beta}^2\epsilon},
\end{align}
where the second condition follows from the same line of analysis as in the proof of
Corollary~\ref{co:sample_matrix_matrix_fast}.

Choose
\begin{align*}
    c_{\alpha,\beta}
    =
    \min\left\{
    \frac{\tau \ell_\tau^3}{32},
    \frac{\ell_\tau \tau^3}{128m^2},
    \frac{2\sqrt{2}}{L_\tau^{1/2}}
    \right\},
\end{align*}
which satisfies Condition~\ref{con:stepsize_matrix}. Taking
$\tau=\epsilon/(4\log(m))$, we have
\begin{align*}
    \ell_\tau
    =
    \frac{1}{(m-1)\exp(2/\tau)+1}
    =
    \frac{1}{(m-1)\exp(8\log(m)/\epsilon)+1}
    =:
    f(\epsilon).
\end{align*}
Moreover, since $L_\tau=\tau/\ell_\tau+m^2/\tau$, we have
$L_\tau=\mathcal{O}(1/(\tau\ell_\tau))$, where the hidden constant depends only on $m$.
Thus,
\begin{align*}
    \frac{L_\tau}{c_{\alpha,\beta}^2}
    =
    \mathcal{O}\left(
    \max\left\{
    \frac{L_\tau}{\tau^2\ell_\tau^6},
    \frac{L_\tau}{\ell_\tau^2\tau^6},
    L_\tau^2
    \right\}
    \right)
    =
    \mathcal{O}\left(
    \max\left\{
    \frac{1}{\tau^3\ell_\tau^7},
    \frac{1}{\tau^7\ell_\tau^3}
    \right\}
    \right),
\end{align*}
where the term $L_\tau^2$ is dominated by the first term because $\tau\leq 1$ and
$\ell_\tau\leq 1$. Substituting $\tau=\epsilon/(4\log(m))$ and
$\ell_\tau=f(\epsilon)$ into \eqref{eq:sc:matrix_slow}, we obtain
\begin{align*}
    K
    =
    \mathcal{O}\left(
    \frac{\log(1/\epsilon)}
    {\epsilon^4 f(\epsilon)^3\min\{f(\epsilon)^4,\epsilon^4\}}
    \right).
\end{align*}

\subsection{Statements and Proofs of Supporting Lemmas}\label{ap:matrix_lemmas}

\begin{lemma}\label{le:softmax_bound}
    For any $x\in\mathbb{R}^d$, we have
    \begin{align*}
        \min_{j\in[d]}[\sigma_\tau(x)]_j
        \geq \frac{1}{(d-1)\exp(2\|x\|_\infty/\tau)+1}.
    \end{align*}
\end{lemma}
\proof{Proof of Lemma \ref{le:softmax_bound}:}
Given any $x\in\mathbb{R}^d$ and $j\in[d]$, we have
\begin{align*}
    [\sigma_\tau(x)]_j
    =\,&\frac{\exp(x_j/\tau)}{\sum_{\ell=1}^d\exp(x_\ell/\tau)}\\
    =\,&\frac{1}{\sum_{\ell\neq j}\exp((x_\ell-x_j)/\tau)+1}\\
    \geq\,&\frac{1}{(d-1)\exp(2\|x\|_\infty/\tau)+1}.
\end{align*}
The claim follows since the right-hand side is independent of $j$.
\hfill\qed
\endproof

\begin{lemma}\label{le:entropy-smoothness}
For any $i\in \{1,2\}$ and any
$\mu_1^i,\mu_2^i\in\{\mu^i\in\Delta(\mathcal{A}^i)\mid \min_{a^i\in\mathcal{A}^i}\mu^i(a^i)\geq \ell_\tau\}$, we have
\begin{align*}
    \|\nabla \nu(\mu_1^i)-\nabla \nu(\mu_2^i)\|_2
    \leq
    \frac{1}{\ell_\tau}\|\mu_1^i-\mu_2^i\|_2.
\end{align*}
\end{lemma}
\proof{Proof of Lemma~\ref{le:entropy-smoothness}.}

Fix $i\in \{1,2\}$. For any $\mu^i\in \Delta(\mathcal{A}^i)$ satisfying $\min_{a^i\in\mathcal{A}^i}\mu^i(a^i)\geq \ell_\tau$, the Hessian of the negative entropy function $\nu(\cdot)$ satisfies
\begin{align*}
    0\cdot I_{m_i} \succeq\nabla^2 \nu(\mu^i)
    =
    -\text{diag}(\mu^i)^{-1}
    \succeq
    \frac{-1}{\min_{a^i\in\mathcal{A}^i}\mu^i(a^i)} I_{m_i}
    \succeq
    \frac{-1}{\ell_\tau} I_{m_i}.
\end{align*}
Therefore, $\nabla \nu(\cdot)$ is $1/\ell_\tau$-Lipschitz continuous with respect to $\|\cdot\|_2$ on the set
$\{\mu^i\in\Delta(\mathcal{A}^i)\mid \min_{a^i\in\mathcal{A}^i}\mu^i(a^i)\geq \ell_\tau\}$ \citep{beck2017first}. This proves the desired inequality. \hfill\qed
\endproof

\begin{lemma}\label{le:quadratic_growth}
For $i\in \{1,2\}$, we have for all $\mu^i\in\Delta(\mathcal{A}^i)$ and $\mu^{-i}\in\Delta(\mathcal{A}^{-i})$ that
\begin{align*}
    \|\sigma_\tau(R_i\mu^{-i})-\mu^i\|_2^2
    \leq
    \frac{2}{\tau} V_R(\mu^1,\mu^2).
\end{align*}
\end{lemma}
\proof{Proof of Lemma~\ref{le:quadratic_growth}.}
Recall that the entropy function $\nu(\cdot)$ is $1$-strongly concave with respect to $\|\cdot\|_2$ \citep{gao2017properties}. Fix $i\in \{1,2\}$ and $\mu^{-i}\in\Delta(\mathcal{A}^{-i})$. Define
\begin{align*}
    F_i(\mu^i)
    :=
    \max_{\hat{\mu}^i\in\Delta(\mathcal{A}^i)}
    \left\{
    (\hat{\mu}^i-\mu^i)^\top R_i\mu^{-i}
    +
    \tau \nu(\hat{\mu}^i)
    -
    \tau\nu(\mu^i)
    \right\}.
\end{align*}
As a function of $\mu^i$, $F_i(\mu^i)$ is $\tau$-strongly convex. Moreover, its minimizer is
$\sigma_\tau(R_i\mu^{-i})$, and $F_i(\sigma_\tau(R_i\mu^{-i}))=0$. Therefore, by the quadratic growth property of strongly convex functions,
\begin{align*}
    \|\sigma_\tau(R_i\mu^{-i})-\mu^i\|_2^2
    \leq
    \frac{2}{\tau}F_i(\mu^i)
    \leq
    \frac{2}{\tau}V_R(\mu^1,\mu^2).
\end{align*}
This proves the desired result. \hfill\qed
\endproof

Denote $\Pi_\tau=\{(\pi^1,\pi^2)\in\Delta(\mathcal{A}^1)\times\Delta(\mathcal{A}^2)\mid \min_{a^1\in\mathcal{A}^1}\pi^1(a^1)\geq \ell_\tau,\min_{a^2\in\mathcal{A}^2}\pi^2(a^2)\geq \ell_\tau\}$. Note that Lemma \ref{le:boundedness_matrix} implies that $(\pi_k^1,\pi_k^2)\in\Pi_\tau$ for all $k\geq 0$.

\begin{lemma}\label{le:properties_Lyapunov_matrix}
The function $V_R(\cdot,\cdot)$ has the following properties.
\begin{enumerate}[(1)]
    \item The function $V_R(\mu^1,\mu^2)$ is $L_\tau$ -- smooth on $\Pi_\tau$, where
    $L_\tau=\frac{\tau}{\ell_\tau}+\frac{m^2}{\tau}$.

    \item It holds for any $(\mu^1,\mu^2)\in\Pi_\tau$ that
    \begin{align*}
        \langle \nabla_1V_R(\mu^1,\mu^2),\sigma_\tau(R_1\mu^2)-\mu^1 \rangle
        +
        \langle \nabla_2V_R(\mu^1,\mu^2),\sigma_\tau(R_2\mu^1)-\mu^2 \rangle
        \leq
        -V_R(\mu^1,\mu^2).
    \end{align*}

    \item For any $q^1\in\mathbb{R}^{m_1}$ and $q^2\in\mathbb{R}^{m_2}$, we have for all $(\mu^1,\mu^2)\in\Pi_\tau$ that
    \begin{align*}
        &\langle \nabla_1V_R(\mu^1,\mu^2),\sigma_\tau(q^1)-\sigma_\tau(R_1\mu^2)\rangle
        +
        \langle \nabla_2V_R(\mu^1,\mu^2),\sigma_\tau(q^2)-\sigma_\tau(R_2\mu^1)\rangle\\
        \leq\,&
        \frac{1}{2}V_R(\mu^1,\mu^2)
        +
        4\left(\frac{1}{\tau \ell_\tau^2}+\frac{m^2}{\tau^3}\right)
        \sum_{i=1,2}\|q^i-R_i\mu^{-i} \|_2^2.
    \end{align*}
\end{enumerate}
\end{lemma}

\proof{Proof of Lemma~\ref{le:properties_Lyapunov_matrix}.}
Recall the definition of $V_R(\cdot,\cdot)$ in \eqref{def:V_R-matrix}. By Danskin's theorem \citep{danskin2012theory} and the zero-sum condition $R_1+R_2^\top=0$, we have
\begin{subequations}\label{eq:gradient_VR}
\begin{align}
    \nabla_1V_R(\mu^1,\mu^2)
    =\,&
    -\tau \nabla \nu(\mu^1)
    +
    R_2^\top \sigma_\tau(R_2\mu^1),
    \label{eq:gradient_VR_1}\\
    \nabla_2V_R(\mu^1,\mu^2)
    =\,&
    -\tau \nabla \nu(\mu^2)
    +
    R_1^\top \sigma_\tau(R_1\mu^2).
    \label{eq:gradient_VR_2}
\end{align}
\end{subequations}

\begin{enumerate}[(1)]
    \item For any $(\mu^1,\mu^2),(\Bar{\mu}^1,\Bar{\mu}^2)\in\Pi_\tau$, we have by \eqref{eq:gradient_VR} that
    \begin{align*}
        \|\nabla_1V_R(\mu^1,\mu^2)-\nabla_1V_R(\Bar{\mu}^1,\Bar{\mu}^2)\|_2
        \leq\,&
        \tau \|\nabla \nu(\Bar{\mu}^1)- \nabla \nu(\mu^1)\|_2
        +
        \|R_2\|_2 \|\sigma_\tau(R_2\mu^1)-\sigma_\tau(R_2\Bar{\mu}^1)\|_2\\
        \leq\,&
        \frac{\tau}{\ell_\tau}\|\mu^1-\Bar{\mu}^1\|_2
        +
        \frac{\|R_2\|_2^2}{\tau}\|\mu^1-\Bar{\mu}^1\|_2\\
        \leq\,&
        \left(\frac{\tau}{\ell_\tau}+\frac{m^2}{\tau}\right)
        \|\mu^1-\Bar{\mu}^1\|_2\\
        =\,&
        L_\tau
        \|\mu^1-\Bar{\mu}^1\|_2
    \end{align*}
    where the second inequality follows from Lemma~\ref{le:entropy-smoothness} and the $1/\tau$ -- Lipschitz continuity of $\sigma_\tau(\cdot)$ in $\|\cdot\|_2$ \citep{gao2017properties}, and the last inequality follows from $\|R_i\|_2\leq \sqrt{m_1m_2}\leq m$ for $i\in \{1,2\}$. Similarly,
    \begin{align*}
        \|\nabla_2V_R(\mu^1,\mu^2)-\nabla_2V_R(\Bar{\mu}^1,\Bar{\mu}^2)\|_2
        \leq
        L_\tau
        \|\mu^2-\Bar{\mu}^2\|_2.
    \end{align*}
    It follows that
    \begin{align*}
        \|\nabla V_R(\mu^1,\mu^2)-\nabla V_R(\Bar{\mu}^1,\Bar{\mu}^2)\|_2^2
        \leq\,&
        L_\tau^2
        \sum_{i=1,2}\|\mu^i-\Bar{\mu}^i\|_2^2.
    \end{align*}
    Therefore, $V_R(\cdot,\cdot)$ is an $L_\tau$ -- smooth function on $\Pi_\tau$ \citep{beck2017first}.

    \item By the optimality condition of the softmax map, we have
    \begin{align*}
        \left\langle
        R_1\mu^2+\tau \nabla \nu(\sigma_\tau(R_1 \mu^2)),
        \sigma_\tau(R_1\mu^2)-\mu^1
        \right\rangle
        =
        0.
    \end{align*}
    Thus, by \eqref{eq:gradient_VR},
    \begin{align*}
        \langle \nabla_1V_R(\mu^1,\mu^2),\sigma_\tau(R_1\mu^2)-\mu^1 \rangle
        =\,&
        \tau\langle
        \nabla \nu(\sigma_\tau(R_1 \mu^2))- \nabla \nu(\mu^1),
        \sigma_\tau(R_1\mu^2)-\mu^1
        \rangle\\
        &+
        ( \sigma_\tau(R_2 \mu^1)-\mu^2)^\top R_2( \sigma_\tau(R_1\mu^2)-\mu^1 ).
    \end{align*}
    By the concavity of $\nu(\cdot)$ and the same optimality condition,
    \begin{align*}
        &\langle
        \nabla \nu(\sigma_\tau(R_1 \mu^2))- \nabla \nu(\mu^1),
        \sigma_\tau(R_1\mu^2)-\mu^1
        \rangle\\
        \leq\,&
        \frac{1}{\tau}
        \left[
        (\mu^1)^\top R_1 \mu^2
        +
        \tau \nu(\mu^1)
        -
        \max_{\hat{\mu}^1\in\Delta(\mathcal{A}^1)}
        \left\{
        (\hat{\mu}^1)^\top R_1\mu^2+\tau \nu(\hat{\mu}^1)
        \right\}
        \right].
    \end{align*}
    Therefore,
    \begin{align*}
        &\langle \nabla_1V_R(\mu^1,\mu^2),\sigma_\tau(R_1\mu^2)-\mu^1 \rangle\\
        \leq\,&
        (\mu^1)^\top R_1 \mu^2
        +
        \tau \nu(\mu^1)
        -
        \max_{\hat{\mu}^1\in\Delta(\mathcal{A}^1)}
        \left\{
        (\hat{\mu}^1)^\top R_1\mu^2+\tau \nu(\hat{\mu}^1)
        \right\}\\
        &+
        ( \sigma_\tau(R_2 \mu^1)-\mu^2)^\top R_2( \sigma_\tau(R_1\mu^2)-\mu^1 ).
    \end{align*}
    Similarly,
    \begin{align*}
        &\langle \nabla_2V_R(\mu^1,\mu^2),\sigma_\tau(R_2\mu^1)-\mu^2 \rangle\\
        \leq\,&
        (\mu^2)^\top R_2 \mu^1
        +
        \tau \nu(\mu^2)
        -
        \max_{\hat{\mu}^2\in\Delta(\mathcal{A}^2)}
        \left\{
        (\hat{\mu}^2)^\top R_2\mu^1+\tau \nu(\hat{\mu}^2)
        \right\}\\
        &+
        ( \sigma_\tau(R_1 \mu^2)-\mu^1)^\top R_1( \sigma_\tau(R_2\mu^1)-\mu^2 ).
    \end{align*}
    Adding the previous two inequalities and using $R_1+R_2^\top=0$, we obtain
    \begin{align*}
        &\langle \nabla_1V_R(\mu^1,\mu^2),\sigma_\tau(R_1\mu^2)-\mu^1 \rangle
        +
        \langle \nabla_2V_R(\mu^1,\mu^2),\sigma_\tau(R_2\mu^1)-\mu^2 \rangle\\
        \leq\,&
        -V_R(\mu^1,\mu^2)
        +
        ( \sigma_\tau(R_1 \mu^2)-\mu^1)^\top
        (R_1+R_2^\top )
        ( \sigma_\tau(R_2\mu^1)-\mu^2 )\\
        =\,&
        -V_R(\mu^1,\mu^2).
    \end{align*}

    \item By the optimality condition of the softmax map, we have
    \begin{align*}
        \left\langle
        R_1\mu^2+\tau \nabla \nu(\sigma_\tau(R_1 \mu^2)),
        \sigma_\tau(q^1)-\sigma_\tau(R_1\mu^2)
        \right\rangle
        =
        0.
    \end{align*}
    Therefore, by \eqref{eq:gradient_VR},
    \begin{align*}
        &\langle \nabla_1V_R(\mu^1,\mu^2),\sigma_\tau(q^1)-\sigma_\tau(R_1\mu^2)\rangle\\
        =\,&
        \tau\langle
        \nabla \nu(\sigma_\tau(R_1 \mu^2))- \nabla \nu(\mu^1),
        \sigma_\tau(q^1)-\sigma_\tau(R_1\mu^2)
        \rangle\\
        &+
        ( \sigma_\tau(R_2 \mu^1)-\mu^2)^\top R_2( \sigma_\tau(q^1)-\sigma_\tau(R_1\mu^2)).
    \end{align*}
    By the Cauchy--Schwarz inequality and the AM-GM inequality, for any $c_1,c_2>0$,
    \begin{align*}
        &\langle \nabla_1V_R(\mu^1,\mu^2),\sigma_\tau(q^1)-\sigma_\tau(R_1\mu^2)\rangle\\
        \leq\,&
        \frac{\tau}{2 c_1}
        \|\nabla \nu(\sigma_\tau(R_1 \mu^2))- \nabla \nu(\mu^1)\|_2^2
        +
        \frac{\tau c_1}{2}
        \|\sigma_\tau(q^1)-\sigma_\tau(R_1\mu^2) \|_2^2\\
        &+
        \frac{1}{2 c_2}\|\sigma_\tau(R_2 \mu^1)-\mu^2\|_2^2
        +
        \frac{c_2}{2}\|R_2( \sigma_\tau(q^1)-\sigma_\tau(R_1\mu^2) )\|_2^2\\
        \leq\,&
        \frac{\tau}{2 c_1 \ell_\tau^2}\|\sigma_\tau(R_1 \mu^2)- \mu^1\|_2^2
        +
        \frac{c_1}{2\tau}\|q^1-R_1\mu^2 \|_2^2\\
        &+
        \frac{1}{2 c_2}\|\sigma_\tau(R_2 \mu^1)-\mu^2\|_2^2
        +
        \frac{c_2 \|R_2\|_2^2}{2 \tau^2}\| q^1-R_1\mu^2 \|_2^2\\
        \leq\,&
        \left(\frac{1}{c_1 \ell_\tau^2}+\frac{1}{\tau c_2}\right)
        V_R(\mu^1,\mu^2)
        +
        \frac{c_1}{2\tau}\|q^1-R_1\mu^2 \|_2^2\\
        &+
        \frac{c_2 \|R_2\|_2^2}{2 \tau^2}\| q^1-R_1\mu^2 \|_2^2,
    \end{align*}
    where the second inequality follows from Lemma~\ref{le:entropy-smoothness} and the $1/\tau$ -- Lipschitz continuity of $\sigma_\tau(\cdot)$, and the last inequality follows from Lemma~\ref{le:quadratic_growth}. Choosing $c_1=8/\ell_\tau^2$ and $c_2=8/\tau$ gives
    \begin{align*}
        \langle \nabla_1V_R(\mu^1,\mu^2),\sigma_\tau(q^1)-\sigma_\tau(R_1\mu^2)\rangle
        \leq\,&
        \frac{1}{4}V_R(\mu^1,\mu^2)
        +
        \frac{4}{\tau \ell_\tau^2}\|q^1-R_1\mu^2 \|_2^2\\
        &+
        \frac{4 \|R_2\|_2^2}{ \tau^3}\| q^1-R_1\mu^2 \|_2^2.
    \end{align*}
    Similarly,
    \begin{align*}
        \langle \nabla_2V_R(\mu^1,\mu^2),\sigma_\tau(q^2)-\sigma_\tau(R_2\mu^1)\rangle
        \leq\,&
        \frac{1}{4}V_R(\mu^1,\mu^2)
        +
        \frac{4}{\tau\ell_\tau^2}\|q^2-R_2\mu^1\|_2^2\\
        &+
        \frac{4 \|R_1\|_2^2}{ \tau^3}\| q^2-R_2\mu^1\|_2^2.
    \end{align*}
    Summing up the previous two inequalities and using $\|R_i\|_2\leq m$ for $i\in \{1,2\}$, we obtain
    \begin{align*}
        &\langle \nabla_1V_R(\mu^1,\mu^2),\sigma_\tau(q^1)-\sigma_\tau(R_1\mu^2)\rangle
        +
        \langle \nabla_2V_R(\mu^1,\mu^2),\sigma_\tau(q^2)-\sigma_\tau(R_2\mu^1)\rangle\\
        \leq\,&
        \frac{1}{2}V_R(\mu^1,\mu^2)
        +
        4\left(\frac{1}{\tau \ell_\tau^2}+\frac{m^2}{ \tau^3}\right)
        \sum_{i=1,2}\|q^i-R_i\mu^{-i} \|_2^2.
    \end{align*}
\end{enumerate}
\hfill\qed
\endproof

\section{Proof of Theorem \ref{thm:matrix_slow}}\label{pf:thm:matrix_slow}
The proof of Theorem~\ref{thm:matrix_slow} follows the same high-level structure as that of Theorem~\ref{thm:matrix_fast}: we first establish boundedness of the iterates, then derive Lyapunov drift inequalities for the policies and the $q$-functions, and finally solve the resulting coupled inequalities. However, because Algorithm~\ref{algo:matrix_slow} replaces $\sigma_\tau(\cdot)$ with $\sm(\cdot)$, the drift inequalities for both the policies and the $q$-functions differ substantially from those in the proof of Theorem~\ref{thm:matrix_fast}.

\subsection{Boundedness of the Iterates}

\begin{lemma}\label{le:boundedness_matrix2}
	It holds for all $k\geq 0$ and $i\in \{1,2\}$ that $\|q_k^i\|_\infty\leq 1$ and $\min_{a^i\in\mathcal{A}^i}\pi_k^i(a^i)\geq \ell_{\tau,\Bar{\epsilon}}$, where 
 \begin{align*}
     \ell_{\tau,\Bar{\epsilon}}=\frac{\Bar{\epsilon}}{m}+\frac{(1-\Bar{\epsilon})}{(m-1)\exp(2/\tau)+1}.
 \end{align*}
\end{lemma}
The proof of Lemma \ref{le:boundedness_matrix2} is identical to that of Lemma \ref{le:boundedness_matrix}, and therefore is omitted.

\subsection{Analysis of the Policies}
We also use the Lyapunov function $V_R(\cdot,\cdot)$ defined in \eqref{eq:gradient_VR} to analyze the policies. Lemma~\ref{le:properties_Lyapunov_matrix2} provides the properties of $V_R(\cdot,\cdot)$ needed for the proof of Theorem~\ref{thm:matrix_slow}. We next present the negative drift inequality for the policies generated by Algorithm~\ref{algo:matrix_slow}.

\begin{lemma}\label{le:policy_matrix}
    The following inequality holds for all $k\geq 0$:
    \begin{align*}
    \mathbb{E}[V_R(\pi_{k+1}^1,\pi_{k+1}^2)]
    \leq \,&
    (1-\beta_k)\mathbb{E}[V_R(\pi_k^1,\pi_k^2)]
    +\frac{4\beta_k}{\tau}\left(\frac{\tau}{\ell_{\tau,\Bar{\epsilon}}}
	+ m\right)
    \bigg(\sum_{i=1,2}\mathbb{E}[\| q_k^i-R_i\pi_k^{-i}\|_2^2]\bigg)^{1/2}\\
    &+
    8\Bar{\epsilon}\beta_k\left(\frac{\tau }{\ell_{\tau,\Bar{\epsilon}}}
	+m\right)+2L_{\tau,\Bar{\epsilon}}\beta_k^2,
\end{align*}
where $L_{\tau,\Bar{\epsilon}}=\frac{\tau}{\ell_{\tau,\Bar{\epsilon}}}+\frac{m^2}{\tau}$.
\end{lemma}

\proof{Proof of Lemma~\ref{le:policy_matrix}.}
Using the smoothness property of $V_R(\cdot,\cdot)$ in Lemma~\ref{le:properties_Lyapunov_matrix2} (1) and the update equation in Line~3 of Algorithm~\ref{algo:matrix_slow}, we have for any $k\geq 0$ that
\begin{align*}
	V_R(\pi_{k+1}^1,\pi_{k+1}^2)
    \leq \,&
    V_R(\pi_k^1,\pi_k^2)
    +\beta_k\langle \nabla_2V_R(\pi_k^1,\pi_k^2),\sm(q_k^2)-\pi_k^2 \rangle\\
    &+\beta_k\langle \nabla_1V_R(\pi_k^1,\pi_k^2),\sm(q_k^1)-\pi_k^1 \rangle
    +\frac{L_{\tau,\Bar{\epsilon}}\beta_k^2}{2}\sum_{i=1,2}\|\sm(q_k^i)-\pi_k^i\|_2^2\\
    \leq \,&
    V_R(\pi_k^1,\pi_k^2)
    +\beta_k\langle \nabla_2V_R(\pi_k^1,\pi_k^2),\sigma_\tau(R_2\pi_k^1)-\pi_k^2 \rangle\\
    &+\beta_k\langle \nabla_1 V_R(\pi_k^1,\pi_k^2),\sigma_\tau(R_1\pi_k^2)-\pi_k^1 \rangle\\
    &+\beta_k\langle \nabla_2V_R(\pi_k^1,\pi_k^2),\sm(q_k^2)-\sigma_\tau(R_2\pi_k^1) \rangle\\
    &+\beta_k\langle \nabla_1 V_R(\pi_k^1,\pi_k^2),\sm(q_k^1)-\sigma_\tau(R_1\pi_k^2) \rangle
    +2L_{\tau,\Bar{\epsilon}}\beta_k^2\\
    \leq \,&
    (1-\beta_k)V_R(\pi_k^1,\pi_k^2)
    +2L_{\tau,\Bar{\epsilon}}\beta_k^2\\
    &+\frac{2\beta_k}{\tau}\left(\frac{\tau}{\ell_{\tau,\Bar{\epsilon}}}+m\right)
    \sum_{i=1,2}\|q_k^i-R_i\pi_k^{-i}\|_2
    +8\Bar{\epsilon}\beta_k\left(\frac{\tau}{\ell_{\tau,\Bar{\epsilon}}}+m\right),
\end{align*}
where the last line follows from Lemma~\ref{le:properties_Lyapunov_matrix2} (2) and (3), and we used
$\sum_{i=1,2}\|\sm(q_k^i)-\pi_k^i\|_2^2\leq 4$.

Taking expectations on both sides gives
\begin{align*}
    \mathbb{E}[V_R(\pi_{k+1}^1,\pi_{k+1}^2)]
    \leq \,&
    (1-\beta_k)\mathbb{E}[V_R(\pi_k^1,\pi_k^2)]
    +2L_{\tau,\Bar{\epsilon}}\beta_k^2\\
    &+\frac{2\beta_k}{\tau}\left(\frac{\tau}{\ell_{\tau,\Bar{\epsilon}}}+m\right)
    \sum_{i=1,2}\mathbb{E}[\|q_k^i-R_i\pi_k^{-i}\|_2]+8\Bar{\epsilon}\beta_k\left(\frac{\tau}{\ell_{\tau,\Bar{\epsilon}}}+m\right).
\end{align*}
Finally,
\begin{align*}
    \sum_{i=1,2}\mathbb{E}[\|q_k^i-R_i\pi_k^{-i}\|_2]
    \leq\,&
    \sum_{i=1,2}
    \left(\mathbb{E}[\|q_k^i-R_i\pi_k^{-i}\|_2^2]\right)^{1/2}\\
    \leq\,&
    \left(2\sum_{i=1,2}\mathbb{E}[\|q_k^i-R_i\pi_k^{-i}\|_2^2]\right)^{1/2}\\
    \leq\,&
    2\left(\sum_{i=1,2}\mathbb{E}[\|q_k^i-R_i\pi_k^{-i}\|_2^2]\right)^{1/2}.
\end{align*}
Substituting this bound into the previous display gives the desired result.
\hfill\qed
\endproof

\subsection{Analysis of the q-Functions}

Similar to the analysis of Algorithm~\ref{algo:matrix_fast}, for $i\in \{1,2\}$, let
$F^i:\mathbb{R}^{m_i}\times \mathcal{A}^i\times \mathcal{A}^{-i}\to \mathbb{R}^{m_i}$ be defined as
\begin{align*}
    [F^i(q^i,a_0^i,a_0^{-i})](a^i)
    =
    \mathds{1}_{\{a_0^i=a^i\}}
    \left(R_i(a_0^i,a_0^{-i})-q^i(a_0^i)\right),
    \quad \forall\,(q^i,a_0^i,a_0^{-i}) \text{ and }a^i.
\end{align*}
Then, Algorithm~\ref{algo:matrix_slow}, Line~5, can be compactly written as
\begin{align}\label{eq:q_SA_matrix_slow}
    q_{k+1}^i=q_k^i+\alpha_kF^i(q_k^i,A_k^i,A_k^{-i}).
\end{align}
Given a joint policy $\pi=(\pi^1,\pi^2)$, let $\bar{F}_\pi^i:\mathbb{R}^{m_i}\to \mathbb{R}^{m_i}$ be defined as
\begin{align*}
    \bar{F}_\pi^i(q^i)
    =
    \mathbb{E}_{A^i\sim \pi^i(\cdot),A^{-i}\sim \pi^{-i}(\cdot)}
    [F^i(q^i,A^i,A^{-i})]
    =
    \text{diag}(\pi^i)(R_i\pi^{-i}-q^i).
\end{align*}
We next establish the negative drift inequality of the $q$-functions generated by Algorithm~\ref{algo:matrix_slow} with respect to the norm-square Lyapunov function.

\begin{lemma}\label{le:q-function-drift-matrix}
The following inequality holds for all $k\geq 0$:
    \begin{align*}
        \sum_{i=1,2}\mathbb{E}[\|q_{k+1}^i-R_i\pi_{k+1}^{-i}\|_2^2]
        \leq\,&
        \left(1-\ell_{\tau,\Bar{\epsilon}}\alpha_k\right)
        \sum_{i=1,2}\mathbb{E}[\|q_k^i-R_i\pi_k^{-i}\|_2^2]
        +16\alpha_k^2
        +\frac{24m^2}{\alpha_k\ell_{\tau,\Bar{\epsilon}}}\beta_k^2.
    \end{align*}
\end{lemma}

\proof{Proof of Lemma~\ref{le:q-function-drift-matrix}.}
For any $k\geq 0$ and $i\in\{1,2\}$, define
$\Delta_k^i=q_k^i-R_i\pi_k^{-i}$ and
$\widetilde{\Delta}_k^i=q_k^i-R_i\pi_{k+1}^{-i}$. Since
$A_k^i\sim \pi_{k+1}^i(\cdot)$ and $A_k^{-i}\sim \pi_{k+1}^{-i}(\cdot)$, we have
\begin{align*}
    \bar{F}_{\pi_{k+1}}^i(q_k^i)
    =
    \text{diag}(\pi_{k+1}^i)(R_i\pi_{k+1}^{-i}-q_k^i).
\end{align*}
Using \eqref{eq:q_SA_matrix_slow}, we obtain
\begin{align*}
    \mathbb{E}[\|q_{k+1}^i-R_i\pi_{k+1}^{-i}\|_2^2]
    =\,&
    \mathbb{E}[\|\widetilde{\Delta}_k^i+\alpha_kF^i(q_k^i,A_k^i,A_k^{-i})\|_2^2]\\
    =\,&
    \mathbb{E}[\|\widetilde{\Delta}_k^i\|_2^2]
    +2\alpha_k
    \mathbb{E}[\langle \bar{F}_{\pi_{k+1}}^i(q_k^i),\widetilde{\Delta}_k^i\rangle]
    +\alpha_k^2
    \mathbb{E}[\|F^i(q_k^i,A_k^i,A_k^{-i})\|_2^2]\\
    \leq\,&
    (1-2\ell_{\tau,\Bar{\epsilon}}\alpha_k)
    \mathbb{E}[\|\widetilde{\Delta}_k^i\|_2^2]
    +4\alpha_k^2,
\end{align*}
where the last line follows from Lemma~\ref{le:boundedness_matrix2} and
$\mathbb{E}[\|F^i(q_k^i,A_k^i,A_k^{-i})\|_2^2]\leq 4$.

Next, by Line~3 of Algorithm~\ref{algo:matrix_slow},
\begin{align*}
    \widetilde{\Delta}_k^i
    =
    q_k^i-R_i\pi_{k+1}^{-i}
    =
    \Delta_k^i-\beta_kR_i(\sm(q_k^{-i})-\pi_k^{-i}).
\end{align*}
Using $\|x+y\|_2^2\leq (1+\eta)\|x\|_2^2+(1+\eta^{-1})\|y\|_2^2$ with
$\eta=\ell_{\tau,\Bar{\epsilon}}\alpha_k/2$, and using
$\ell_{\tau,\Bar{\epsilon}}\alpha_k\leq 1$, we have
\begin{align*}
    (1-2\ell_{\tau,\Bar{\epsilon}}\alpha_k)\|\widetilde{\Delta}_k^i\|_2^2
    \leq\,&
    (1-\ell_{\tau,\Bar{\epsilon}}\alpha_k)\|\Delta_k^i\|_2^2
    +
    \frac{3\beta_k^2}{\ell_{\tau,\Bar{\epsilon}}\alpha_k}
    \|R_i(\sm(q_k^{-i})-\pi_k^{-i})\|_2^2.
\end{align*}
Therefore,
\begin{align*}
    \mathbb{E}[\|q_{k+1}^i-R_i\pi_{k+1}^{-i}\|_2^2]
    \leq\,&
    (1-\ell_{\tau,\Bar{\epsilon}}\alpha_k)
    \mathbb{E}[\|q_k^i-R_i\pi_k^{-i}\|_2^2]\\
    &+
    \frac{3\beta_k^2}{\ell_{\tau,\Bar{\epsilon}}\alpha_k}
    \mathbb{E}[\|R_i(\sm(q_k^{-i})-\pi_k^{-i})\|_2^2]
    +4\alpha_k^2.
\end{align*}
Since both $\sm(q_k^{-i})$ and $\pi_k^{-i}$ are probability vectors, we have
$\|\sm(q_k^{-i})-\pi_k^{-i}\|_2^2\leq 4$. Hence,
\begin{align*}
    \mathbb{E}[\|R_i(\sm(q_k^{-i})-\pi_k^{-i})\|_2^2]
    \leq
    4\|R_i\|_2^2
    \leq
    4m^2.
\end{align*}
It follows that
\begin{align*}
    \mathbb{E}[\|q_{k+1}^i-R_i\pi_{k+1}^{-i}\|_2^2]
    \leq\,&
    (1-\ell_{\tau,\Bar{\epsilon}}\alpha_k)
    \mathbb{E}[\|q_k^i-R_i\pi_k^{-i}\|_2^2]
    +4\alpha_k^2
    +\frac{12m^2}{\ell_{\tau,\Bar{\epsilon}}\alpha_k}\beta_k^2.
\end{align*}
The final result follows by summing the previous inequality over $i\in\{1,2\}$.
\hfill\qed
\endproof

\subsection{Solving Coupled Lyapunov Drift Inequalities}

Denote $\mathcal{L}_q(k)=\sum_{i=1,2}\mathbb{E}[\|q_k^i-R_i\pi_k^{-i}\|_2^2]$ and $\mathcal{L}_\pi(k)=\mathbb{E}[V_R(\pi_k^1,\pi_k^2)]$ for simplicity of notation. When $\alpha_k\equiv \alpha$ and $\beta_k\equiv \beta$, Lemmas~\ref{le:policy_matrix} and~\ref{le:q-function-drift-matrix} state that
\begin{align}
    \mathcal{L}_q(k+1)
    \leq \,&
    \left(1-\ell_{\tau,\Bar{\epsilon}}\alpha\right)\mathcal{L}_q(k)
    +16\alpha^2+\frac{24m^2}{\alpha\ell_{\tau,\Bar{\epsilon}}}\beta^2,
    \label{recursion:matrix_q}\\
    \mathcal{L}_\pi(k+1)
    \leq \,&
    (1-\beta)\mathcal{L}_\pi(k)
    +\frac{4\beta}{\tau}\left(\frac{\tau}{\ell_{\tau,\Bar{\epsilon}}}+ m\right)\mathcal{L}_q^{1/2}(k)+
    8\Bar{\epsilon}\beta\left(\frac{\tau }{\ell_{\tau,\Bar{\epsilon}}}+m\right)
    +2L_{\tau,\Bar{\epsilon}} \beta^2.
    \label{recursion:matrix_pi}
\end{align}
Iterating \eqref{recursion:matrix_q}, we have for all $k\geq 0$ that
\begin{align*}
    \mathcal{L}_q(k)
    \leq\,&
    \left(1-\ell_{\tau,\Bar{\epsilon}}\alpha\right)^k\mathcal{L}_q(0)
    +\frac{16\alpha}{\ell_{\tau,\Bar{\epsilon}}}
    +\frac{24\beta^2m^2}{\alpha^2\ell_{\tau,\Bar{\epsilon}}^2}.
\end{align*}
Therefore,
\begin{align*}
    \mathcal{L}_q^{1/2}(k)
    \leq\,&
    \left(1-\ell_{\tau,\Bar{\epsilon}}\alpha\right)^{k/2}\mathcal{L}_q^{1/2}(0)
    +\frac{4\sqrt{\alpha}}{\ell_{\tau,\Bar{\epsilon}}^{1/2}}
    +\frac{5\beta m}{\alpha \ell_{\tau,\Bar{\epsilon}}}.
\end{align*}
Substituting this bound into \eqref{recursion:matrix_pi}, we obtain
\begin{align*}
    \mathcal{L}_\pi(k+1)
    \leq \,&
    (1-\beta)\mathcal{L}_\pi(k)
    +8\Bar{\epsilon}\beta\left(\frac{\tau }{\ell_{\tau,\Bar{\epsilon}}}+ m\right)
    +2L_{\tau,\Bar{\epsilon}}\beta^2\\
    &+
    \frac{4\beta}{\tau}\left(\frac{\tau}{\ell_{\tau,\Bar{\epsilon}}}+ m\right)
    \left[
    \left(1-\ell_{\tau,\Bar{\epsilon}}\alpha\right)^{k/2}\mathcal{L}_q^{1/2}(0)
    +\frac{4\sqrt{\alpha}}{\ell_{\tau,\Bar{\epsilon}}^{1/2}}
    +\frac{5\beta m}{\alpha \ell_{\tau,\Bar{\epsilon}}}
    \right].
\end{align*}
Iterating the previous inequality, and using $\beta\leq \ell_{\tau,\Bar{\epsilon}}\alpha/2$, we have
\begin{align*}
    \mathcal{L}_\pi(k)
    \leq \,&
    (1-\beta)^k\mathcal{L}_\pi(0)
    +8\Bar{\epsilon}\left(\frac{\tau }{\ell_{\tau,\Bar{\epsilon}}}+ m\right)
    +2L_{\tau,\Bar{\epsilon}}\beta\\
    &+
    \frac{8}{\tau}\left(\frac{\tau}{\ell_{\tau,\Bar{\epsilon}}}+ m\right)
    \beta k(1-\beta)^k\mathcal{L}_q^{1/2}(0)+
    \frac{4}{\tau}\left(\frac{\tau}{\ell_{\tau,\Bar{\epsilon}}}+ m\right)
    \left[
    \frac{4\sqrt{\alpha}}{\ell_{\tau,\Bar{\epsilon}}^{1/2}}
    +\frac{5\beta m}{\alpha \ell_{\tau,\Bar{\epsilon}}}
    \right]\\
    \leq \,&
    (1-\beta)^k\left(
    \mathcal{L}_\pi(0)
    +
    8\left(\frac{1}{\ell_{\tau,\Bar{\epsilon}}}
    + \frac{m}{\tau}\right)
    \beta k \mathcal{L}_q^{1/2}(0)
    \right)\\
    &+
    8\Bar{\epsilon}\left(\frac{\tau }{\ell_{\tau,\Bar{\epsilon}}}+ m\right)
    +\ell_{\tau,\Bar{\epsilon}} L_{\tau,\Bar{\epsilon}} \alpha+
    4\left(\frac{1}{\ell_{\tau,\Bar{\epsilon}}}
    + \frac{m}{\tau}\right)
    \left[
    \frac{4\sqrt{\alpha}}{\ell_{\tau,\Bar{\epsilon}}^{1/2}}
    +\frac{5\beta m}{\alpha \ell_{\tau,\Bar{\epsilon}}}
    \right],
\end{align*}
where we used $2L_{\tau,\Bar{\epsilon}}\beta\leq \ell_{\tau,\Bar{\epsilon}}L_{\tau,\Bar{\epsilon}}\alpha$.

Since $\ell_{\tau,\Bar{\epsilon}}\geq \Bar{\epsilon}/m$ by Lemma~\ref{le:boundedness_matrix2}, $\Bar{\epsilon}=\tau$, and $L_{\tau,\Bar{\epsilon}}=\tau/\ell_{\tau,\Bar{\epsilon}}+m^2/\tau$, we have
\begin{align*}
    \mathcal{L}_\pi(k)
    \leq\,&
    (1-\beta)^k\left(
    \mathcal{L}_\pi(0)+8\alpha k \mathcal{L}_q^{1/2}(0)
    \right)
    +16\tau m
    +\frac{2m}{\tau}\alpha
    +\frac{32m^{3/2}\sqrt{\alpha}}{\tau^{3/2}}
    +\frac{40\beta m^3}{\alpha \tau^2}.
\end{align*}
Note that $\mathcal{L}_\pi(0)\leq 4+2\tau \log(m)$, $\mathcal{L}_q(0)\leq 2m$, and $\text{NG}(\pi_k^1,\pi_k^2)
    \leq
    \mathcal{L}_\pi(k)+2\tau\log(m)$. Therefore,
\begin{align*}
    \mathbb{E}[\text{NG}(\pi_k^1,\pi_k^2)]
    \leq\,&
    (1-\beta)^k
    \left(
    4+2\tau\log(m)+8\sqrt{2m}\alpha k
    \right)\\
    &+
    16\tau m
    +2\tau\log(m)
    +\frac{2m}{\tau}\alpha
    +\frac{32m^{3/2}\sqrt{\alpha}}{\tau^{3/2}}
    +\frac{40\beta m^3}{\alpha \tau^2}\\
    \leq\,&
    20\sqrt{m}k(1-\beta)^k+
    18\tau m
    +\frac{2m}{\tau}\alpha
    +\frac{32m^{3/2}\sqrt{\alpha}}{\tau^{3/2}}
    +\frac{40\beta m^3}{\alpha \tau^2}.
\end{align*}

\subsection{Proof of Corollary \ref{co:sc_matrix_slow}}\label{pf:co:sc_matrix_slow}

In view of Theorem~\ref{thm:matrix_slow}, we have
\begin{align*}
    \mathbb{E}[\text{NG}(\pi_K^1,\pi_K^2)]
    \leq\,&
    20\sqrt{m}K(1-\beta)^K
    +18\tau m
    +\frac{2m}{\tau}\alpha
    +\frac{32m^{3/2}\sqrt{\alpha}}{\tau^{3/2}}
    +\frac{40\beta m^3}{\alpha \tau^2}.
\end{align*}
Given $\epsilon>0$, choose $\tau=\epsilon/(90m)$, $\alpha=\epsilon^2\tau^3/(160^2m^3)$, and $\beta=\epsilon\alpha\tau^2/(200m^3)$. Then $18\tau m=\epsilon/5$, and direct substitution gives
\begin{align*}
    \frac{2m}{\tau}\alpha
    \leq \frac{\epsilon}{5},\qquad
    \frac{32m^{3/2}\sqrt{\alpha}}{\tau^{3/2}}
    =
    \frac{\epsilon}{5},\qquad
    \frac{40\beta m^3}{\alpha \tau^2}
    =
    \frac{\epsilon}{5}.
\end{align*}
It remains to control the transient term. Since $(1-\beta)^K\leq e^{-\beta K}$, it is sufficient to ensure $20\sqrt{m}K e^{-\beta K}\leq \epsilon/5$. This holds whenever
\begin{align*}
    K
    \geq
    \frac{2}{\beta}
    \log\left(\frac{100\sqrt{m}}{\epsilon\beta}\right).
\end{align*}
Indeed, for such $K$, we have $\beta K\geq 2\log(100\sqrt{m}/(\epsilon\beta))$, and hence $K e^{-\beta K}\leq \epsilon/(100\sqrt{m})$. Therefore, $20\sqrt{m}K e^{-\beta K}\leq \epsilon/5$. Combining the previous bounds gives $\mathbb{E}[\text{NG}(\pi_K^1,\pi_K^2)]\leq \epsilon$.

Finally, substituting $\tau=\epsilon/(90m)$ and $\alpha=\epsilon^2\tau^3/(160^2m^3)$ into $\beta=\epsilon\alpha\tau^2/(200m^3)$ gives $\beta=\Theta(\epsilon^8/m^{11})$. Therefore,
\begin{align*}
    K
    =
    \mathcal{O}\left(
    m^{11}\epsilon^{-8}
    \log\left(\frac{m}{\epsilon}\right)
    \right).
\end{align*}

\subsection{Statements and Proofs of Supporting Lemmas}
Let $\Pi_{\tau,\Bar{\epsilon}}=\{(\pi^1,\pi^2)\in\Delta(\mathcal{A}^1)\times\Delta(\mathcal{A}^2)\mid \min_{a^1\in\mathcal{A}^1}\pi^1(a^1)\geq \ell_{\tau,\Bar{\epsilon}},\min_{a^2\in\mathcal{A}^2}\pi^2(a^2)\geq \ell_{\tau,\Bar{\epsilon}}\}$. Note that Lemma \ref{le:boundedness_matrix2} implies that $(\pi_k^1,\pi_k^2)\in\Pi_{\tau,\Bar{\epsilon}}$ for all $k\geq 0$.
\begin{lemma}\label{le:properties_Lyapunov_matrix2}
	The function $V_R(\cdot,\cdot)$ has the following properties.
	\begin{enumerate}[(1)]
		\item The function $V_R(\mu^1,\mu^2)$ is $L_{\tau,\Bar{\epsilon}}$ -- smooth on $\Pi_{\tau,\Bar{\epsilon}}$, where $L_{\tau,\Bar{\epsilon}}=\frac{\tau}{\ell_{\tau,\Bar{\epsilon}}}+\frac{m^2}{\tau}$.
		\item It holds for any $(\mu^1,\mu^2)\in\Pi_{\tau,\Bar{\epsilon}}$ that
		\begin{align*}
			\langle \nabla_1V_R(\mu^1,\mu^2),\sigma_\tau(R_1\mu^2)-\mu^1 \rangle
            +\langle \nabla_2V_R(\mu^1,\mu^2),\sigma_\tau(R_2\mu^1)-\mu^2 \rangle
			\leq -V_R(\mu^1,\mu^2).
		\end{align*}
		\item For any $q^1\in\mathbb{R}^{m_1}$ and $q^2\in\mathbb{R}^{m_2}$, we have for all $(\mu^1,\mu^2)\in\Pi_{\tau,\Bar{\epsilon}}$ that
		\begin{align*}
            &\langle \nabla_1V_R(\mu^1,\mu^2),\sm(q^1)-\sigma_\tau(R_1\mu^2)\rangle
            +\langle \nabla_2V_R(\mu^1,\mu^2),\sm(q^2)-\sigma_\tau(R_2\mu^1)\rangle\\
            \leq\,&
            8\Bar{\epsilon}\left(\frac{\tau }{\ell_{\tau,\Bar{\epsilon}}}+ m\right)
            +\frac{2}{\tau}\left(\frac{\tau}{\ell_{\tau,\Bar{\epsilon}}}+ m\right)
            \sum_{i=1,2}\| q^i-R_i\mu^{-i}\|_2.
        \end{align*}
	\end{enumerate}
\end{lemma}

\proof{Proof of Lemma~\ref{le:properties_Lyapunov_matrix2}.}
The proof of Lemma~\ref{le:properties_Lyapunov_matrix2} (1) and (2) is identical to that of Lemma~\ref{le:properties_Lyapunov_matrix} (1) and (2), and therefore is omitted. For Lemma~\ref{le:properties_Lyapunov_matrix2} (3), recall that by Danskin's theorem \citep{danskin2012theory},
\begin{align*}
    \nabla_1V_R(\mu^1,\mu^2)=\,&-\tau \nabla \nu(\mu^1)+R_2^\top \sigma_\tau(R_2\mu^1),\\
    \nabla_2V_R(\mu^1,\mu^2)=\,&-\tau \nabla \nu(\mu^2)+R_1^\top \sigma_\tau(R_1\mu^2).
\end{align*}
By the optimality condition of the softmax operator, we have
\begin{align*}
    \left\langle
    R_1\mu^2+\tau \nabla \nu(\sigma_\tau(R_1\mu^2)),
    \sm(q^1)-\sigma_\tau(R_1\mu^2)
    \right\rangle
    =
    0.
\end{align*}
Therefore,
\begin{align*}
	\langle \nabla_1V_R(\mu^1,\mu^2),\sm(q^1)-\sigma_\tau(R_1\mu^2)\rangle
=\,&
\tau\langle  \nabla \nu(\sigma_\tau(R_1 \mu^2))- \nabla \nu(\mu^1),
\sm(q^1)-\sigma_\tau(R_1\mu^2) \rangle\\
&+(\sigma_\tau(R_2 \mu^1)-\mu^2)^\top R_2
(\sm(q^1)-\sigma_\tau(R_1\mu^2)).
\end{align*}
By the Cauchy--Schwarz inequality, we obtain
\begin{align*}
	&\langle \nabla_1V_R(\mu^1,\mu^2),\sm(q^1)-\sigma_\tau(R_1\mu^2)\rangle\\
\leq\,&
\left(
\tau\|\nabla \nu(\sigma_\tau(R_1 \mu^2))- \nabla \nu(\mu^1)\|_2
+\|\sigma_\tau(R_2 \mu^1)-\mu^2\|_2 \|R_2\|_2
\right)\|\sm(q^1)-\sigma_\tau(R_1\mu^2)\|_2 .
\end{align*}
Since $(\mu^1,\mu^2)\in\Pi_{\tau,\Bar{\epsilon}}$, Lemma~\ref{le:entropy-smoothness} implies
\begin{align*}
	&\langle \nabla_1V_R(\mu^1,\mu^2),\sm(q^1)-\sigma_\tau(R_1\mu^2)\rangle\\
\leq\,&
\left(
\frac{\tau}{\ell_{\tau,\Bar{\epsilon}}}
\|\sigma_\tau(R_1 \mu^2)- \mu^1\|_2
+\|\sigma_\tau(R_2 \mu^1)-\mu^2\|_2 \|R_2\|_2
\right)\|\sm(q^1)-\sigma_\tau(R_1\mu^2)\|_2 .
\end{align*}
Since
$\|\sigma_\tau(R_1 \mu^2)- \mu^1\|_2
+\|\sigma_\tau(R_2 \mu^1)-\mu^2\|_2
\leq 2$, we obtain
\begin{align*}
	\langle \nabla_1V_R(\mu^1,\mu^2),\sm(q^1)-\sigma_\tau(R_1\mu^2)\rangle
\leq
2\left(\frac{\tau}{\ell_{\tau,\Bar{\epsilon}}}
+\|R_2\|_2\right)
\|\sm(q^1)-\sigma_\tau(R_1\mu^2)\|_2 .
\end{align*}
Moreover, by Lemma~\ref{le:sm_to_sigma_tau} and the $1/\tau$-Lipschitz continuity of $\sigma_\tau(\cdot)$,
\begin{align*}
	\|\sm(q^1)-\sigma_\tau(R_1\mu^2)\|_2
	\leq\,&
	\|\sm(q^1)-\sigma_\tau(q^1)\|_2
	+\|\sigma_\tau(q^1)-\sigma_\tau(R_1\mu^2)\|_2\\
	\leq\,&
	2\Bar{\epsilon}
	+\frac{1}{\tau}\|q^1-R_1\mu^2\|_2 .
\end{align*}
Combining the previous two bounds yields
\begin{align*}
	\langle \nabla_1V_R(\mu^1,\mu^2),\sm(q^1)-\sigma_\tau(R_1\mu^2)\rangle
\leq\,&
2\left(\frac{\tau}{\ell_{\tau,\Bar{\epsilon}}}
+\|R_2\|_2\right)
\left(
2\Bar{\epsilon}
+\frac{1}{\tau}\|q^1-R_1\mu^2\|_2
\right)\\
=\,&
4\Bar{\epsilon}
\left(\frac{\tau }{\ell_{\tau,\Bar{\epsilon}}}
+\|R_2\|_2\right)
+\frac{2}{\tau}
\left(\frac{\tau}{\ell_{\tau,\Bar{\epsilon}}}
+\|R_2\|_2\right)
\|q^1-R_1\mu^2\|_2\\
\leq\,&
4\Bar{\epsilon}
\left(\frac{\tau }{\ell_{\tau,\Bar{\epsilon}}}
+m\right)
+\frac{2}{\tau}
\left(\frac{\tau}{\ell_{\tau,\Bar{\epsilon}}}
+m\right)
\|q^1-R_1\mu^2\|_2,
\end{align*}
where the last line follows from $\|R_2\|_2\leq m$. Similarly, we have
\begin{align*}
	\langle \nabla_2V_R(\mu^1,\mu^2),\sm(q^2)-\sigma_\tau(R_2\mu^1)\rangle
    \leq\,&
    4\Bar{\epsilon}\left(\frac{\tau }{\ell_{\tau,\Bar{\epsilon}}}+ m\right)
    +\frac{2}{\tau}\left(\frac{\tau}{\ell_{\tau,\Bar{\epsilon}}}+ m\right)
    \| q^2-R_2\mu^1\|_2.
\end{align*}
The claim follows from adding the previous two inequalities.
\hfill\qed
\endproof

\begin{lemma}\label{le:sm_to_sigma_tau}
    Given $i\in \{1,2\}$, for any $q^i\in\mathbb{R}^{m_i}$, we have
    $\|\sm(q^i)-\sigma_\tau(q^i)\|_2\leq 2\Bar{\epsilon}$.
\end{lemma}

\proof{Proof of Lemma~\ref{le:sm_to_sigma_tau}.}
Given $i\in \{1,2\}$, for any $q^i\in\mathbb{R}^{m_i}$, we have
\begin{align*}
    \|\sm(q^i)-\sigma_\tau(q^i)\|_2^2
    =
    \Bar{\epsilon}^2
    \sum_{a^i\in\mathcal{A}^i}
    \left(\frac{1}{m_i}-\sigma_\tau(q^i)(a^i)\right)^2
    \leq
    4\Bar{\epsilon}^2,
\end{align*}
where the inequality follows because both $\text{Unif}(\mathcal{A}^i)$ and $\sigma_\tau(q^i)$ are probability vectors. It follows that
$\|\sm(q^i)-\sigma_\tau(q^i)\|_2\leq 2\Bar{\epsilon}$.
\hfill\qed
\endproof

\section{Details for the Proof of Theorem~\ref{thm:stochastic_game}}\label{sec:analysis}

We begin with a summary of notation.

\subsection{Notation}\label{subsec:notation}

\begin{enumerate}[(1)]
	\item Given a pair of matrices $\{X_i\in\mathbb{R}^{m_i\times m_{-i}}\}_{i\in \{1,2\}}$ and a pair of distributions $\{\mu^i\in\Delta(\mathcal{A}^i)\}_{i\in \{1,2\}}$, we define
    \begin{align}\label{def:Nash_Gap}
		V_{X}(\mu^1,\mu^2)
        =
        \sum_{i=1,2}
        \max_{\hat{\mu}^i\in\Delta(\mathcal{A}^i)}
        \left\{
        (\hat{\mu}^i-\mu^i)^\top X_i\mu^{-i}
        +\tau \nu(\hat{\mu}^i)-\tau\nu(\mu^i)
        \right\},
	\end{align}
	where $\nu(\cdot)$ is the entropy function. Note that $V_X(\cdot,\cdot)$ is similar to $V_R(\cdot,\cdot)$ defined in Appendix~\ref{ap:matrix_policy} for matrix games. However, we do not assume that $X_1+X_2^\top =0$.

	\item Given a pair of value functions $v=(v^1,v^2)$ and a state $s\in\mathcal{S}$, when $X_i=\mathcal{T}^i(v^i)(s)$ for $i\in \{1,2\}$, we write $V_{v,s}(\cdot,\cdot)$ for $V_X(\cdot,\cdot)$.

	\item For any joint policy $(\pi^1,\pi^2)$ and state $s$, given $i\in \{1,2\}$, we define
    $v^i_{*,\pi^{-i}}(s)=\max_{\hat{\pi}^i}v^i_{\hat{\pi}^i,\pi^{-i}}(s)$,
    $v^i_{\pi^i,*}(s)=\min_{\hat{\pi}^{-i}}v^i_{\pi^i,\hat{\pi}^{-i}}(s)$,
    $v^{-i}_{\pi^{-i},*}(s)=\min_{\hat{\pi}^i}v^{-i}_{\pi^{-i},\hat{\pi}^i}(s)$, and
    $v^{-i}_{*,\pi^i}(s)=\max_{\hat{\pi}^{-i}}v^{-i}_{\hat{\pi}^{-i},\pi^i}(s)$.
    Note that $v^1_{*,\pi^2}+v^2_{\pi^2,*}=0$ and $v^1_{\pi^1,*}+v^2_{*,\pi^1}=0$ because of the zero-sum structure.

	\item For $i\in \{1,2\}$, denote by $v_*^i$ the unique fixed point of the equation $\mathcal{B}^i(v^i)=v^i$, where $\mathcal{B}^i(\cdot)$ is the minimax Bellman operator defined in Section~\ref{subsec:Markov_Algorithm}. Note that $v_*^1+v_*^2=0$.

    \item For any $t,k\geq 0$ and $i\in \{1,2\}$, let $\Bar{q}_{t,k}^i\in\mathbb{R}^{nm_i}$ be defined as $\Bar{q}_{t,k}^i(s)=\mathcal{T}^i(v_t^i)(s)\pi_{t,k}^{-i}(s)$ for all $s\in\mathcal{S}$. In addition, let
    \begin{align*}
        \mathcal{L}_{\text{sum}}(t)
        =\,&
        \|v_t^1+v_t^2\|_\infty,\quad
        \mathcal{L}_v(t)
        =
        \sum_{i=1,2}\|v_t^i-v_*^i\|_\infty,\\
        \mathcal{L}_q(t,k)
        =\,&
        \sum_{i=1,2}\sum_{s\in\mathcal{S}}
        \|q_{t,k}^i(s)-\mathcal{T}^i(v_t^i)(s)\pi_{t,k}^{-i}(s)\|_2^2
        =
        \sum_{i=1,2}\|q_{t,k}^i-\Bar{q}_{t,k}^i\|_2^2,\\
        \mathcal{L}_\pi(t,k)
        =\,&
        \max_{s\in\mathcal{S}}
        V_{v_t,s}(\pi_{t,k}^1(s),\pi_{t,k}^2(s)).
    \end{align*}
    These will be the Lyapunov functions used in the analysis.

    \item Given $k_1\leq k_2$, we denote $\beta_{k_1,k_2}=\sum_{k=k_1}^{k_2}\beta_k$ and $\alpha_{k_1,k_2}=\sum_{k=k_1}^{k_2}\alpha_k$.
\end{enumerate}

\subsection{Proof of Lemma~\ref{le:boundedness_proof_outline}}\label{subsec:boundedness}

Let $i\in \{1,2\}$. The proof uses induction arguments.

\begin{enumerate}[(1)]
\item Fixing $t\geq 0$, we first show by induction that, if $\|v_t^i\|_\infty\leq \frac{1}{1-\gamma}$ and $\|q_{t,0}^i\|_\infty\leq \frac{1}{1-\gamma}$, then $\|q_{t,k}^i\|_\infty\leq \frac{1}{1-\gamma}$ for all $k\geq 0$. The base case holds by the assumption $\|q_{t,0}^i\|_\infty\leq \frac{1}{1-\gamma}$. Suppose that $\|q_{t,k}^i\|_\infty\leq \frac{1}{1-\gamma}$ for some $k\geq 0$. Then, by Algorithm~\ref{algo:stochastic_game}, Line~6, we have for all $(s,a^i)$ that
\begin{align}
	|q_{t,k+1}^i(s,a^i)|
	=\,&
    |q_{t,k}^i(s,a^i)+\alpha_k\mathds{1}_{\{(s,a^i)=(S_k,A_k^i)\}}
    (R_i(S_k,A_k^i,A_k^{-i})+\gamma v_t^i(S_{k+1})-q_{t,k}^i(S_k,A_k^i))|\nonumber\\
	\leq \,&
    (1-\alpha_k\mathds{1}_{\{(s,a^i)=(S_k,A_k^i)\}})|q_{t,k}^i(s,a^i)|\nonumber\\
    &+
    \alpha_k\mathds{1}_{\{(s,a^i)=(S_k,A_k^i)\}}
    |R_i(S_k,A_k^i,A_k^{-i})+\gamma v_t^i(S_{k+1})|\nonumber\\
	\leq \,&
    (1-\alpha_k\mathds{1}_{\{(s,a^i)=(S_k,A_k^i)\}})\frac{1}{1-\gamma}
    +\alpha_k\mathds{1}_{\{(s,a^i)=(S_k,A_k^i)\}}
    \left(1+\frac{\gamma}{1-\gamma}\right)\label{eq:boundedness1}\\
	= \,&
    \frac{1}{1-\gamma},\nonumber
\end{align}
where \eqref{eq:boundedness1} follows from the induction hypothesis, the assumption $\|v_t^i\|_\infty\leq \frac{1}{1-\gamma}$, the bound $\max_{s,a^i,a^{-i}}|R_i(s,a^i,a^{-i})|\leq 1$, and the stepsize condition $\alpha_k\in[0,1]$. The induction is complete, and hence $\|q_{t,k}^i\|_\infty\leq \frac{1}{1-\gamma}$ for all $k\geq 0$ whenever $\|v_t^i\|_\infty\leq \frac{1}{1-\gamma}$ and $\|q_{t,0}^i\|_\infty\leq \frac{1}{1-\gamma}$.

We next use induction to show that $\|v_t^i\|_\infty\leq \frac{1}{1-\gamma}$ and $\|q_{t,0}^i\|_\infty\leq \frac{1}{1-\gamma}$ for all $t\geq 0$. The initialization ensures that $\|v_0^i\|_\infty\leq \frac{1}{1-\gamma}$ and $\|q_{0,0}^i\|_\infty\leq \frac{1}{1-\gamma}$. Suppose that $\|v_t^i\|_\infty\leq \frac{1}{1-\gamma}$ and $\|q_{t,0}^i\|_\infty\leq \frac{1}{1-\gamma}$ for some $t\geq  0$. Using the update equation for $v_{t+1}^i$ in Algorithm~\ref{algo:stochastic_game}, Line~8, and the fact that $\|q_{t,k}^i\|_\infty\leq \frac{1}{1-\gamma}$ for all $k\geq 0$, we have for all $s\in\mathcal{S}$ that
\begin{align*}
	|v_{t+1}^i(s)|
    =
    \left|\sum_{a^i\in\mathcal{A}^i}\pi^i_{t,K}(a^i|s)q_{t,K}^i(s,a^i)\right|
    \leq
    \sum_{a^i\in\mathcal{A}^i}\pi^i_{t,K}(a^i|s)\|q_{t,K}^i\|_\infty
    \leq
    \frac{1}{1-\gamma}.
\end{align*}
Thus, $\|v_{t+1}^i\|_\infty\leq \frac{1}{1-\gamma}$. Moreover, Algorithm~\ref{algo:stochastic_game}, Line~9, gives $\|q_{t+1,0}^i\|_\infty=\|q_{t,K}^i\|_\infty\leq \frac{1}{1-\gamma}$. The induction is complete, and hence $\|v_t^i\|_\infty\leq \frac{1}{1-\gamma}$ and $\|q_{t,0}^i\|_\infty\leq \frac{1}{1-\gamma}$ for all $t\geq 0$.

\item We first use induction to show that, given $t\geq 0$, if $\min_{s,a^i}\pi_{t,0}^i(a^i\mid s)\geq \ell_\tau$, then $\min_{s,a^i}\pi_{t,k}^i(a^i\mid s)\geq \ell_\tau$ for all $k\in \{0,1,\cdots,K\}$. The base case holds by the assumption. Suppose that $\min_{s\in\mathcal{S},a^i\in\mathcal{A}^i}\pi_{t,k}^i(a^i\mid s)\geq \ell_\tau$ for some $k\geq 0$. Then, by Algorithm~\ref{algo:stochastic_game}, Line~4, we have
\begin{align*}
    \pi_{t,k+1}^i(a^i\mid s)
    =\,&
    (1-\beta_k)\pi_{t,k}^i(a^i\mid s)+\beta_k\sigma_\tau(q_{t,k}^i(s))(a^i)\\
    \geq\,&
    (1-\beta_k) \ell_\tau+\beta_k\ell_\tau\\
    =\,&
    \ell_\tau,
\end{align*}
where the inequality follows from the induction hypothesis, Part (1), Lemma~\ref{le:softmax_bound}, and the stepsize condition $\beta_k\in[0,1]$. The induction is complete.

We next use induction to show that $\min_{s,a^i}\pi_{t,0}^i(a^i\mid s)\geq \ell_\tau$ for all $t\in \{0,1,\cdots,T\}$. Since $\pi_{0,0}^i$ is initialized as a uniform policy, the base case holds. Suppose that $\min_{s,a^i}\pi_{t,0}^i(a^i\mid s)\geq \ell_\tau$ for some $t\geq 0$. Then $\min_{s,a^i}\pi_{t,k}^i(a^i\mid s)\geq \ell_\tau$ for all $k\in \{0,1,\cdots,K\}$. Since $\pi_{t+1,0}^i=\pi_{t,K}^i$ by Algorithm~\ref{algo:stochastic_game}, Line~9, we have $\min_{s,a^i}\pi_{t+1,0}^i(a^i\mid s)\geq \ell_\tau$. The induction is complete.
\end{enumerate}
\hfill\qed
\endproof

\subsection{Proof of Lemma \ref{le:Nash_Combine}}\label{ap:bound_Nash_stochastic}
Our ultimate goal is to bound the Nash gap 
\begin{align}\label{explicit_Nash}
	\text{NG}(\pi_{T,K}^1,\pi_{T,K}^2)=\sum_{i=1,2}\left(\max_{\pi^i}U^i(\pi^i,\pi_{T,K}^{-i})-U^i(\pi_{T,K}^i,\pi_{T,K}^{-i})\right)
\end{align}
as a function of the Lyapunov functions.
We first bound the Nash gap using the value functions of the output policies from Algorithm \ref{algo:stochastic_game}. 

\begin{lemma}\label{le:Nash_to_v}
	It holds that
	\begin{align}\label{eq:Nash_to_v_policy}
		\sum_{i=1,2}\left(\max_{\pi^i}U^i(\pi^i,\pi_{T,K}^{-i})-U^i(\pi_{T,K}^i,\pi_{T,K}^{-i})\right)
        \leq
        \sum_{i=1,2}\left\|v^i_{*,\pi_{T,K}^{-i}}-v^i_{\pi_{T,K}^i,\pi_{T,K}^{-i}}\right\|_\infty.
	\end{align}
\end{lemma}

\proof{Proof of Lemma~\ref{le:Nash_to_v}.}
Using the definition of the utility function, we have
\begin{align*}
	\sum_{i=1,2}\left(\max_{\pi^i}U^i(\pi^i,\pi_{T,K}^{-i})-U^i(\pi_{T,K}^i,\pi_{T,K}^{-i})\right)
	= \,&
    \sum_{i=1,2}\left(\max_{\pi^i}\mathbb{E}_{S\sim p_o}\left[v^i_{\pi^i,\pi_{T,K}^{-i}}(S)-v^i_{\pi_{T,K}^i,\pi_{T,K}^{-i}}(S)\right]\right)\\
	\leq \,&
    \sum_{i=1,2}\mathbb{E}_{S\sim p_o}\left[\max_{\pi^i}v^i_{\pi^i,\pi_{T,K}^{-i}}(S)-v^i_{\pi_{T,K}^i,\pi_{T,K}^{-i}}(S)\right]\\
	= \,&
    \sum_{i=1,2}\mathbb{E}_{S\sim p_o}\left[v^i_{*,\pi_{T,K}^{-i}}(S)-v^i_{\pi_{T,K}^i,\pi_{T,K}^{-i}}(S)\right]\\
	\leq \,&
    \sum_{i=1,2}\left\|v^i_{*,\pi_{T,K}^{-i}}-v^i_{\pi_{T,K}^i,\pi_{T,K}^{-i}}\right\|_\infty.
\end{align*}
\hfill\qed
\endproof

The next lemma bounds the right-hand side of  (\ref{eq:Nash_to_v_policy}) using the iterates from Algorithm \ref{algo:stochastic_game}.

\begin{lemma}\label{le:Nash_Gap}
	It holds for $i\in \{1,2\}$ that
	\begin{align*}
		\left\|v^i_{*,\pi_{T,K}^{-i}}-v^i_{\pi_{T,K}^i,\pi_{T,K}^{-i}}\right\|_\infty
        \leq
        \frac{2}{1-\gamma}
        \left(
        2\mathcal{L}_{\text{sum}}(T)
        +\mathcal{L}_v(T)
        +\mathcal{L}_\pi(T,K)
        +2\tau \log(m)
        \right).
	\end{align*}
\end{lemma}

\proof{Proof of Lemma~\ref{le:Nash_Gap}.}
For any $s\in\mathcal{S}$ and $i\in \{1,2\}$, we have
\begin{align*}
	0
    \leq \,&
    \left|v^i_{*,\pi_{T,K}^{-i}}(s)-v^i_{\pi_{T,K}^i,\pi_{T,K}^{-i}}(s)\right|\\
	=\,&
    v^i_{*,\pi_{T,K}^{-i}}(s)-v^i_{\pi_{T,K}^i,\pi_{T,K}^{-i}}(s)\\
	\leq \,&
    v^i_{*,\pi_{T,K}^{-i}}(s)-v^i_{\pi_{T,K}^i,*}(s)\\
	=\,&
    -v^{-i}_{\pi_{T,K}^{-i},*}(s)-v^i_{\pi_{T,K}^i,*}(s)\\
	=\,&
    v^i_*(s)-v^{-i}_{\pi_{T,K}^{-i},*}(s)
    +
    v^{-i}_*(s)-v^i_{\pi_{T,K}^i,*}(s)\\
	\leq \,&
    \sum_{j=1,2}\left\|v^{-j}_*-v^{-j}_{\pi_{T,K}^{-j},*}\right\|_\infty.
\end{align*}
Since the right-hand side does not depend on $s$, we have, for $i\in \{1,2\}$,
\begin{align}
	\left\|v^i_{*,\pi_{T,K}^{-i}}-v^i_{\pi_{T,K}^i,\pi_{T,K}^{-i}}\right\|_\infty
    \leq
    \sum_{j=1,2}\left\|v^{-j}_*-v^{-j}_{\pi_{T,K}^{-j},*}\right\|_\infty.
    \label{eq:last_policy_bound}
\end{align}
It remains to bound the right-hand side of \eqref{eq:last_policy_bound}. For any $s\in\mathcal{S}$ and $i\in \{1,2\}$, we have
\begin{align}
	0
    \leq\,&
    v^{-i}_*(s)-v^{-i}_{\pi_{T,K}^{-i},*}(s)\nonumber\\
	=\,&
    v^i_{*,\pi_{T,K}^{-i}}(s)-v^i_*(s)\nonumber\\
	=\,&
    \max_{\mu^i\in\Delta(\mathcal{A}^i)}
    (\mu^i)^\top \mathcal{T}^i(v^i_{*,\pi_{T,K}^{-i}})(s)\pi_{T,K}^{-i}(s)
    -
    \max_{\mu^i\in\Delta(\mathcal{A}^i)}
    \min_{\mu^{-i}\in\Delta(\mathcal{A}^{-i})}
    (\mu^i)^\top \mathcal{T}^i(v_*^i)(s)\mu^{-i}\nonumber\\
	\leq \,&
    \left|
    \max_{\mu^i}
    (\mu^i)^\top \mathcal{T}^i(v^i_{*,\pi_{T,K}^{-i}})(s)\pi_{T,K}^{-i}(s)
    -
    \max_{\mu^i}
    (\mu^i)^\top \mathcal{T}^i(v^i_{*})(s)\pi_{T,K}^{-i}(s)
    \right|\nonumber\\
	&+
    \left|
    \max_{\mu^i}
    (\mu^i)^\top \mathcal{T}^i(v^i_{*})(s)\pi_{T,K}^{-i}(s)
    -
    \max_{\mu^i}
    (\mu^i)^\top \mathcal{T}^i(v^i_T)(s)\pi_{T,K}^{-i}(s)
    \right|\nonumber\\
	&+
    \max_{\mu^i}
    (\mu^i)^\top \mathcal{T}^i(v^i_T)(s)\pi_{T,K}^{-i}(s)
    -
    \max_{\mu^i}
    \min_{\mu^{-i}}
    (\mu^i)^\top \mathcal{T}^i(v^i_T)(s)\mu^{-i}\nonumber\\
	&+
    \left|
    \max_{\mu^i}
    \min_{\mu^{-i}}
    (\mu^i)^\top \mathcal{T}^i(v^i_T)(s)\mu^{-i}
    -
    \max_{\mu^i}
    \min_{\mu^{-i}}
    (\mu^i)^\top \mathcal{T}^i(v_*^i)(s)\mu^{-i}
    \right|.
    \label{eq:connect1}
\end{align}
We next bound the four terms on the right-hand side of \eqref{eq:connect1}.

For the first term, using the definition of $\mathcal{T}^i(\cdot)$, we have
\begin{align*}
	&\left|
    \max_{\mu^i}
    (\mu^i)^\top \mathcal{T}^i(v^i_{*,\pi_{T,K}^{-i}})(s)\pi_{T,K}^{-i}(s)
    -
    \max_{\mu^i}
    (\mu^i)^\top \mathcal{T}^i(v^i_{*})(s)\pi_{T,K}^{-i}(s)
    \right|\\
	\leq\,&
    \max_{\mu^i}
    \left|
    (\mu^i)^\top
    \left(\mathcal{T}^i(v^i_{*,\pi_{T,K}^{-i}})(s)-\mathcal{T}^i(v^i_{*})(s)\right)
    \pi_{T,K}^{-i}(s)
    \right|\\
	\leq\,&
    \gamma
    \left\|v^i_{*}-v^i_{*,\pi_{T,K}^{-i}}\right\|_\infty.
\end{align*}
Similarly, the second term is bounded by
\begin{align*}
	\left|
    \max_{\mu^i}
    (\mu^i)^\top \mathcal{T}^i(v^i_{*})(s)\pi_{T,K}^{-i}(s)
    -
    \max_{\mu^i}
    (\mu^i)^\top \mathcal{T}^i(v^i_T)(s)\pi_{T,K}^{-i}(s)
    \right|
	\leq
    \gamma \left\|v^i_*-v^i_T\right\|_\infty.
\end{align*}

We next consider the third term. We decompose it as
\begin{align}
	&\max_{\mu^i}
    (\mu^i)^\top \mathcal{T}^i(v^i_T)(s)\pi_{T,K}^{-i}(s)
    -
    \max_{\mu^i}
    \min_{\mu^{-i}}
    (\mu^i)^\top \mathcal{T}^i(v^i_T)(s)\mu^{-i}\nonumber\\
	\leq\,&
    \left|
    \max_{\mu^i}
    (\mu^i)^\top \mathcal{T}^i(v_T^i)(s)\pi_{T,K}^{-i}(s)
    -
    \min_{\mu^{-i}}
    (\pi_{T,K}^i(s))^\top\mathcal{T}^i(v_T^i)(s)\mu^{-i}
    \right|\nonumber\\
	\leq\,&
    \left|
    \max_{\mu^{-i}}
    (\mu^{-i})^\top \mathcal{T}^{-i}(v_T^{-i})(s)\pi_{T,K}^i(s)
    +
    \min_{\mu^{-i}}
    (\mu^{-i})^\top \mathcal{T}^i(v_T^i)(s)^\top \pi_{T,K}^i(s)
    \right|\nonumber\\
	&+
    \left|
    \sum_{j=1,2}
    \max_{\mu^j}
    (\mu^j)^\top \mathcal{T}^j(v_T^j)(s)\pi_{T,K}^{-j}(s)
    \right|.
    \label{eq_decompose:le:Nash_Gap}
\end{align}
For the first term on the right-hand side of \eqref{eq_decompose:le:Nash_Gap}, we have
\begin{align*}
	&\left|
    \max_{\mu^{-i}}
    (\mu^{-i})^\top \mathcal{T}^{-i}(v_T^{-i})(s)\pi_{T,K}^i(s)
    +
    \min_{\mu^{-i}}
    (\mu^{-i})^\top \mathcal{T}^i(v_T^i)(s)^\top \pi_{T,K}^i(s)
    \right|\\
	=\,&
    \left|
    \max_{\mu^{-i}}
    (\mu^{-i})^\top \mathcal{T}^{-i}(v_T^{-i})(s)\pi_{T,K}^i(s)
    -
    \max_{\mu^{-i}}
    (\mu^{-i})^\top [-\mathcal{T}^i(v_T^i)(s)]^\top \pi_{T,K}^i(s)
    \right|\\
	\leq\,&
    \max_{\mu^{-i}}
    \left|
    (\mu^{-i})^\top
    \left(\mathcal{T}^{-i}(v_T^{-i})(s)+\mathcal{T}^i(v_T^i)(s)^\top\right)
    \pi_{T,K}^i(s)
    \right|\\
	\leq\,&
    \gamma \left\|v_T^{-i}+v_T^i\right\|_\infty.
\end{align*}
For the second term on the right-hand side of \eqref{eq_decompose:le:Nash_Gap}, using the Lyapunov function $V_{v_T,s}(\cdot,\cdot)$, we have
\begin{align*}
	\left|
    \sum_{j=1,2}
    \max_{\mu^j}
    (\mu^j)^\top \mathcal{T}^j(v_T^j)(s)\pi_{T,K}^{-j}(s)
    \right|
	=\,&
    \sum_{j=1,2}
    \max_{\mu^j}
    (\mu^j-\pi_{T,K}^j(s))^\top
    \mathcal{T}^j(v_T^j)(s)\pi_{T,K}^{-j}(s)\\
    &+
    \left|
    \sum_{j=1,2}
    (\pi_{T,K}^j(s))^\top
    \mathcal{T}^j(v_T^j)(s)\pi_{T,K}^{-j}(s)
    \right|\\
	\leq\,&
    V_{v_T,s}(\pi_{T,K}^1(s),\pi_{T,K}^2(s))
    +2\tau \log(m)
    +
    \gamma\|v_T^1+v_T^2\|_\infty.
\end{align*}
Using the previous two bounds in \eqref{eq_decompose:le:Nash_Gap}, we obtain
\begin{align}
	&\max_{\mu^i}
    (\mu^i)^\top \mathcal{T}^i(v^i_T)(s)\pi_{T,K}^{-i}(s)
    -
    \max_{\mu^i}
    \min_{\mu^{-i}}
    (\mu^i)^\top \mathcal{T}^i(v^i_T)(s)\mu^{-i}\nonumber\\
	\leq\,&
    V_{v_T,s}(\pi_{T,K}^1(s),\pi_{T,K}^2(s))
    +2\gamma\|v_T^1+v_T^2\|_\infty
    +2\tau \log(m).
    \label{bound:3rd_Term}
\end{align}

For the fourth term, using the definition of $\mathcal{T}^i(\cdot)$ and the Lipschitz property of the matrix-game value, we have
\begin{align*}
	&\left|
    \max_{\mu^i}
    \min_{\mu^{-i}}
    (\mu^i)^\top \mathcal{T}^i(v^i_T)(s)\mu^{-i}
    -
    \max_{\mu^i}
    \min_{\mu^{-i}}
    (\mu^i)^\top \mathcal{T}^i(v_*^i)(s)\mu^{-i}
    \right|\\
	\leq\,&
    \max_{a^i,a^{-i}}
    \left|
    \mathcal{T}^i(v^i_T)(s,a^i,a^{-i})
    -
    \mathcal{T}^i(v_*^i)(s,a^i,a^{-i})
    \right|\\
	\leq\,&
    \gamma \|v_T^i-v_*^i\|_\infty.
\end{align*}
Combining the bounds for the four terms in \eqref{eq:connect1}, we obtain
\begin{align*}
	\left\|v^{-i}_*-v^{-i}_{\pi_{T,K}^{-i},*}\right\|_\infty
    \leq\,&
    \gamma\left\|v^i_{*,\pi_{T,K}^{-i}}-v^i_{*}\right\|_\infty
    +2\gamma\|v_T^1+v_T^2\|_\infty
    +2\gamma\|v^i_T-v^i_{*}\|_\infty\\
	&+
    \max_{s\in\mathcal{S}}V_{v_T,s}(\pi_{T,K}^1(s),\pi_{T,K}^2(s))
    +2\tau \log(m)\\
	\leq\,&
    \gamma\|v^{-i}_*-v^{-i}_{\pi_{T,K}^{-i},*}\|_\infty
    +2\mathcal{L}_{\text{sum}}(T)
    +2\|v^i_T-v^i_{*}\|_\infty+
    \mathcal{L}_\pi(T,K)
    +2\tau \log(m).
\end{align*}
Rearranging terms gives
\begin{align*}
	\left\|v^{-i}_*-v^{-i}_{\pi_{T,K}^{-i},*}\right\|_\infty
	\leq
    \frac{1}{1-\gamma}
    \left(
    2\mathcal{L}_{\text{sum}}(T)
    +2\|v^i_T-v^i_{*}\|_\infty
    +\mathcal{L}_\pi(T,K)
    +2\tau \log(m)
    \right).
\end{align*}
Summing both sides over $i\in \{1,2\}$, we have
\begin{align*}
	\sum_{i=1,2}
    \left\|v^{-i}_*-v^{-i}_{\pi_{T,K}^{-i},*}\right\|_\infty
	\leq
    \frac{2}{1-\gamma}
    \left(
    2\mathcal{L}_{\text{sum}}(T)
    +\mathcal{L}_v(T)
    +\mathcal{L}_\pi(T,K)
    +2\tau \log(m)
    \right).
\end{align*}
Using the previous inequality in \eqref{eq:last_policy_bound}, we obtain the desired result.
\hfill\qed
\endproof
The proof of Lemma \ref{le:Nash_Combine} follows by combining Lemmas \ref{le:Nash_to_v} and \ref{le:Nash_Gap} in \eqref{explicit_Nash}.

\subsection{Analysis of the Outer Loop}

\subsubsection{Proof of Lemma~\ref{le:outer-loop}}\label{subsec:outer-loop}

For $i\in \{1,2\}$, using the outer-loop update equation in Algorithm~\ref{algo:stochastic_game}, Line~8, and the fact that $\mathcal{B}^i(v_*^i)=v_*^i$, we have for any $t\geq 0$ and $s\in\mathcal{S}$ that
\begin{align*}
	v_{t+1}^i(s)-v_*^i(s)
    =\,&
    \pi_{t,K}^i(s)^\top q_{t,K}^i(s)-v_*^i(s)\\
    =\,&
    \mathcal{B}^i(v^i_t)(s)-\mathcal{B}^i(v_*^i)(s)
    +\pi_{t,K}^i(s)^\top q_{t,K}^i(s)-\mathcal{B}^i(v^i_t)(s).
\end{align*}
Since the minimax Bellman operator $\mathcal{B}^i(\cdot)$ is a $\gamma$ -- contraction mapping in $\|\cdot\|_\infty$, we have
\begin{align}
	\left|v_{t+1}^i(s)-v_*^i(s)\right|
    \leq\,&
    \gamma\left\|v^i_t-v_*^i\right\|_\infty
    +
    \left|\pi_{t,K}^i(s)^\top q_{t,K}^i(s)-\mathcal{B}^i(v^i_t)(s)\right|.
    \label{eq1:prop:outer}
\end{align}
It remains to bound the second term on the right-hand side of \eqref{eq1:prop:outer}. Using the definition of $\mathcal{B}^i(\cdot)$, we have
\begin{align*}
	&\left|\pi_{t,K}^i(s)^\top q_{t,K}^i(s)-\mathcal{B}^i(v^i_t)(s)\right|\\
    =\,&
    \left|
    \pi_{t,K}^i(s)^\top q_{t,K}^i(s)
    -
    \max_{\mu^i\in\Delta(\mathcal{A}^i)}
    \min_{\mu^{-i}\in\Delta(\mathcal{A}^{-i})}
    (\mu^i)^\top \mathcal{T}^i(v_t^i)(s)\mu^{-i}
    \right|\\
	\leq \,&
    \left|
    \max_{\mu^i\in\Delta(\mathcal{A}^i)}
    (\mu^i)^\top \mathcal{T}^i(v^i_t)(s)\pi_{t,K}^{-i}(s)
    -
    \pi_{t,K}^i(s)^\top q_{t,K}^i(s)
    \right|\\
	&+
    \left|
    \max_{\mu^i\in\Delta(\mathcal{A}^i)}
    (\mu^i)^\top \mathcal{T}^i(v^i_t)(s)\pi_{t,K}^{-i}(s)
    -
    \max_{\mu^i\in\Delta(\mathcal{A}^i)}
    \min_{\mu^{-i}\in\Delta(\mathcal{A}^{-i})}
    (\mu^i)^\top \mathcal{T}^i(v_t^i)(s)\mu^{-i}
    \right|\\
    \leq\,&
    \max_{\mu^i\in\Delta(\mathcal{A}^i)}
    (\mu^i-\pi_{t,K}^i(s))^\top \mathcal{T}^i(v^i_t)(s)\pi_{t,K}^{-i}(s)\\
	&+
    \left|
    (\pi_{t,K}^i(s))^\top
    \left(\mathcal{T}^i(v^i_t)(s)\pi_{t,K}^{-i}(s)-q_{t,K}^i(s)\right)
    \right|\\
	&+
    \left|
    \max_{\mu^i\in\Delta(\mathcal{A}^i)}
    (\mu^i)^\top \mathcal{T}^i(v^i_t)(s)\pi_{t,K}^{-i}(s)
    -
    \max_{\mu^i\in\Delta(\mathcal{A}^i)}
    \min_{\mu^{-i}\in\Delta(\mathcal{A}^{-i})}
    (\mu^i)^\top \mathcal{T}^i(v_t^i)(s)\mu^{-i}
    \right|\\
    \leq\,&
    \left\|\mathcal{T}^i(v^i_t)(s)\pi_{t,K}^{-i}(s)-q_{t,K}^i(s)\right\|_\infty
    +2V_{v_t,s}(\pi_{t,K}^1(s),\pi_{t,K}^2(s))\\
    &+
    2\gamma\|v_t^1+v_t^2\|_\infty
    +3\tau \log(m),
\end{align*}
where the last line follows from the same argument as in \eqref{bound:3rd_Term}, with $T$ replaced by $t$. Using the previous inequality in \eqref{eq1:prop:outer}, we obtain
\begin{align*}
    \left\|v^i_{t+1}-v_*^i\right\|_\infty
    \leq\,&
    \gamma\left\|v^i_t-v_*^i\right\|_\infty
    +
    \max_{s\in\mathcal{S}}
    \left\|\mathcal{T}^i(v^i_t)(s)\pi_{t,K}^{-i}(s)-q_{t,K}^i(s)\right\|_\infty\\
    &+
    2\max_{s\in\mathcal{S}}V_{v_t,s}(\pi_{t,K}^1(s),\pi_{t,K}^2(s))
    +2\gamma\|v_t^1+v_t^2\|_\infty
    +3\tau \log(m).
\end{align*}
Summing both sides over $i\in \{1,2\}$ gives
\begin{align*}
    \mathcal{L}_v(t+1)
    \leq\,&
    \gamma \mathcal{L}_v(t)
    +4\mathcal{L}_{\text{sum}}(t)
    +4\mathcal{L}_\pi(t,K)
    +6\tau \log(m)\\
    &+
    \sum_{i=1,2}
    \max_{s\in\mathcal{S}}
    \left\|\mathcal{T}^i(v^i_t)(s)\pi_{t,K}^{-i}(s)-q_{t,K}^i(s)\right\|_\infty.
\end{align*}
To bound the last term, observe that
\begin{align}
    \sum_{i=1,2}
    \max_{s\in\mathcal{S}}
    \left\|\mathcal{T}^i(v^i_t)(s)\pi_{t,K}^{-i}(s)-q_{t,K}^i(s)\right\|_\infty
    =\,&
    \sum_{i=1,2}\left\|\Bar{q}_{t,K}^i-q_{t,K}^i\right\|_\infty\nonumber\\
    \leq\,&
    \sum_{i=1,2}\left\|\Bar{q}_{t,K}^i-q_{t,K}^i\right\|_2\nonumber\\
    \leq\,&
    \left(
    2\sum_{i=1,2}\left\|\Bar{q}_{t,K}^i-q_{t,K}^i\right\|_2^2
    \right)^{1/2}\nonumber\\
    \leq\,&
    2\mathcal{L}_q^{1/2}(t,K).
    \label{eq:Jensen}
\end{align}
Therefore,
\begin{align*}
    \mathcal{L}_v(t+1)
    \leq\,&
    \gamma \mathcal{L}_v(t)
    +4\mathcal{L}_{\text{sum}}(t)
    +2\mathcal{L}_q^{1/2}(t,K)
    +4\mathcal{L}_\pi(t,K)
    +6\tau \log(m).
\end{align*}
This completes the proof.
\hfill\qed
\endproof

\subsubsection{Proof of Lemma~\ref{le:outer-sum}}\label{pf:le:outer-sum}

Using the outer-loop update equation in Algorithm~\ref{algo:stochastic_game}, Line~8, we have for any $t\geq 0$ and $s\in\mathcal{S}$ that
\begin{align*}
	\left|v_{t+1}^1(s)+v_{t+1}^2(s)\right|
	=\,&
    \left|\sum_{i=1,2}\pi_{t,K}^i(s)^\top q_{t,K}^i(s)\right|\\
	\leq \,&
    \left|\sum_{i=1,2}\pi_{t,K}^i(s)^\top
    (q_{t,K}^i(s)-\mathcal{T}^i(v_t^i)(s)\pi_{t,K}^{-i}(s))\right|\\
    &+
    \left|
    \sum_{i=1,2}
    (\pi_{t,K}^i(s))^\top
    \mathcal{T}^i(v_t^i)(s)\pi_{t,K}^{-i}(s)
    \right|\\
	\leq \,&
    \sum_{i=1,2}\max_{s\in\mathcal{S}}
    \| q_{t,K}^i(s)-\mathcal{T}^i(v_t^i)(s)\pi_{t,K}^{-i}(s)\|_\infty\\
    &+
    \max_{(s,a^i,a^{-i})}
    \left|
    \mathcal{T}^i(v_t^i)(s,a^i,a^{-i})
    +
    \mathcal{T}^{-i}(v_t^{-i})(s,a^i,a^{-i})
    \right|\\
	\leq \,&
    \sum_{i=1,2}\max_{s\in\mathcal{S}}
    \| q_{t,K}^i(s)-\mathcal{T}^i(v_t^i)(s)\pi_{t,K}^{-i}(s)\|_\infty
    +\gamma\|v_t^1+v_t^2\|_\infty,
\end{align*}
where the last line follows from the definition of $\mathcal{T}^i(\cdot)$. Since the right-hand side does not depend on $s$, we have
\begin{align*}
	\|v_{t+1}^1+v_{t+1}^2\|_\infty
    \leq
    \gamma\|v_t^1+v_t^2\|_\infty
    +
    \sum_{i=1,2}\max_{s\in\mathcal{S}}
    \| q_{t,K}^i(s)-\mathcal{T}^i(v_t^i)(s)\pi_{t,K}^{-i}(s)\|_\infty.
\end{align*}
The result follows from using \eqref{eq:Jensen} to bound the last term on the right-hand side and then using $\mathcal{L}_{\text{sum}}(t)$ and $\mathcal{L}_q(t,K)$ to simplify the notation.
\hfill\qed
\endproof

\subsection{Analysis of the Inner Loop}\label{subsec:inner_loop_stochastic}

For ease of presentation, we write down only the inner loop of Algorithm~\ref{algo:stochastic_game} in Algorithm~\ref{algorithm:inner-loop}, where we omit the outer-loop index $t$. Similarly, we write $\mathcal{L}_q(k)$ for $\mathcal{L}_q(t,k)$ and $\mathcal{L}_\pi(k)$ for $\mathcal{L}_\pi(t,k)$. All results derived for the $q$-functions and policies of Algorithm~\ref{algorithm:inner-loop} can be combined with the outer-loop analysis of Algorithm~\ref{algo:stochastic_game} using a conditioning argument and the Markov property.

\begin{algorithm}[ht]
\caption{Inner Loop of Algorithm~\ref{algo:stochastic_game}}\label{algorithm:inner-loop}
	\begin{algorithmic}[1]
		\STATE \textbf{Input:} Integer $K$, initializations $q_0^i$ and $\pi_0^i$, and a value function $v^i$ from the outer loop. Note that $\|q_0^i\|_\infty\leq \frac{1}{1-\gamma}$, $\|v^i\|_\infty\leq \frac{1}{1-\gamma}$, and $\min_{s,a^i}\pi_0^i(a^i\mid s)\geq\ell_\tau$ by Lemma~\ref{le:boundedness_proof_outline}.
		\FOR{$k=0,1,\cdots,K-1$}
		\STATE $\pi_{k+1}^i(s)=\pi_k^i(s)+\beta_k(\sigma_\tau(q_k^i(s))-\pi_k^i(s))$ for all $s\in\mathcal{S}$
		\STATE Sample $A_k^i\sim \pi_{k+1}^i(\cdot\mid S_k)$, receive reward $R_i(S_k,A_k^i,A_k^{-i})$, and observe $S_{k+1}\sim p(\cdot\mid S_k,A_k^i,A_k^{-i})$
		\STATE $q_{k+1}^i(s,a^i)=q_k^i(s,a^i)+\alpha_k\mathds{1}_{\{(S_k,A_k^i)=(s,a^i)\}} \left(R_i(S_k,A_k^i,A_k^{-i})+\gamma v^i(S_{k+1})-q_k^i(S_k,A_k^i)\right)$ for all $(s,a^i)\in\mathcal{S}\times \mathcal{A}^i$
		\ENDFOR
	\end{algorithmic}
\end{algorithm}

\subsubsection{Proof of Lemma~\ref{le:properties_Lyapunov_main}}\label{pf:le:properties_Lyapunov_main}

To begin with, by Danskin's theorem \citep{danskin2012theory}, we have
\begin{align}\label{eq:V_gradient}
    \nabla_1V_X(\mu^1,\mu^2)
    =\,&
    -(X_1+X_2^\top )\mu^2
    -\tau \nabla \nu(\mu^1)
    +X_2^\top \sigma_\tau(X_2\mu^1).
\end{align}
A similar result holds for $\nabla_2V_X(\mu^1,\mu^2)$.

\begin{enumerate}[(1)]
    \item It is clear that the function $V_X(\cdot,\cdot)$ is non-negative. The strong convexity follows from the following two observations.
    \begin{enumerate}[(i)]
        \item The negative entropy $-\nu(\cdot)$ is $1$ -- strongly convex with respect to $\|\cdot\|_2$ \cite[Example 5.27]{beck2017first}.
        \item Given $i\in \{1,2\}$, the function
        $\max_{\hat{\mu}^{-i}\in\Delta(\mathcal{A}^{-i})}\{(\hat{\mu}^{-i})^\top X_{-i}\mu^i+\tau \nu(\hat{\mu}^{-i})\}$,
        as a function of $\mu^i$, is the maximum of linear functions in $\mu^i$, and therefore is convex.
    \end{enumerate}
    It follows that, for any $i\in \{1,2\}$, the function $V_X(\mu^1,\mu^2)$ is $\tau$ -- strongly convex in $\mu^i$ with respect to $\|\cdot\|_2$, uniformly over $\mu^{-i}$.

    \item For any $(\mu^1,\mu^2),(\Bar{\mu}^1,\Bar{\mu}^2)\in\Pi_\tau$, we have by \eqref{eq:V_gradient} that
    \begin{align}
        &\left\|\nabla_1V_X(\mu^1,\mu^2)-\nabla_1V_X(\Bar{\mu}^1,\Bar{\mu}^2)\right\|_2\nonumber\\
        =\,&
        \left\|
        (X_1+X_2^\top)(\mu^2-\Bar{\mu}^2)
        +\tau(\nabla \nu(\mu^1)-\nabla \nu(\Bar{\mu}^1))
        +X_2^\top(\sigma_\tau(X_2\Bar{\mu}^1)-\sigma_\tau(X_2\mu^1))
        \right\|_2\nonumber\\
        \leq\,&
        \|X_1+X_2^\top\|_2\|\mu^2-\Bar{\mu}^2\|_2
        +
        \left(\frac{\tau}{\ell_\tau}+\frac{\|X_2\|_2^2}{\tau}\right)
        \left\|\Bar{\mu}^1-\mu^1\right\|_2,
        \label{eq:smoothness_1nequality}
    \end{align}
    where \eqref{eq:smoothness_1nequality} follows from Lemma~\ref{le:entropy-smoothness} and the Lipschitz continuity of $\sigma_\tau(\cdot)$ \citep{gao2017properties}. Similarly,
    \begin{align*}
        \left\|\nabla_2V_X(\mu^1,\mu^2)-\nabla_2V_X(\Bar{\mu}^1,\Bar{\mu}^2)\right\|_2
        \leq
        \|X_2+X_1^\top\|_2\|\mu^1-\Bar{\mu}^1\|_2
        +
        \left(\frac{\tau}{\ell_\tau}+\frac{\|X_1\|_2^2}{\tau}\right)
        \left\|\Bar{\mu}^2-\mu^2\right\|_2.
    \end{align*}
    Using the previous two inequalities, we have
    \begin{align*}
        &\left\|\nabla V_X(\mu^1,\mu^2)-\nabla V_X(\Bar{\mu}^1,\Bar{\mu}^2)\right\|_2^2\\
        \leq\,&
        \sum_{i=1,2}
        \left[
        2\left(\frac{\tau}{\ell_\tau}+\frac{\|X_{-i}\|_2^2}{\tau}\right)^2
        \left\|\Bar{\mu}^i-\mu^i\right\|_2^2
        +
        2\|X_i+X_{-i}^\top\|_2^2
        \|\mu^{-i}-\Bar{\mu}^{-i}\|_2^2
        \right]\\
        \leq\,&
        2\left[
        \left(\frac{\tau}{\ell_\tau}
        +\frac{\max(\|X_1\|_2^2,\|X_2\|_2^2)}{\tau}\right)^2
        +\|X_1+X_2^\top\|_2^2
        \right]
        \sum_{i=1,2}\|\Bar{\mu}^i-\mu^i\|_2^2.
    \end{align*}
    Therefore, $V_X(\cdot,\cdot)$ is an $\tilde{L}_\tau$ -- smooth function on $\Pi_\tau$ \citep{beck2017first}, where
    \begin{align*}
        \tilde{L}_\tau
        =
        2\left(
        \frac{\tau}{\ell_\tau}
        +\frac{\max(\|X_1\|_2^2,\|X_2\|_2^2)}{\tau}
        +\|X_1+X_2^\top\|_2
        \right).
    \end{align*}

    \item By the optimality condition of the softmax map, we have
    \begin{align*}
        \left\langle
        X_1\mu^2+\tau\nabla\nu(\sigma_\tau(X_1\mu^2)),
        \sigma_\tau(X_1\mu^2)-\mu^1
        \right\rangle
        =
        0.
    \end{align*}
    Using \eqref{eq:V_gradient}, we have
    \begin{align*}
        \langle \nabla_1V_X(\mu^1,\mu^2),\sigma_\tau(X_1\mu^2)-\mu^1 \rangle
        =\,&
        \tau\langle
        \nabla \nu(\sigma_\tau(X_1 \mu^2))- \nabla \nu(\mu^1),
        \sigma_\tau(X_1\mu^2)-\mu^1
        \rangle\\
        &+
        ( \sigma_\tau(X_2 \mu^1)-\mu^2)^\top X_2( \sigma_\tau(X_1\mu^2)-\mu^1 ).
    \end{align*}
    By the concavity of $\nu(\cdot)$ and the same optimality condition,
    \begin{align*}
        &\langle
        \nabla \nu(\sigma_\tau(X_1 \mu^2))- \nabla \nu(\mu^1),
        \sigma_\tau(X_1\mu^2)-\mu^1
        \rangle\\
        \leq\,&
        \frac{1}{\tau}
        \left[
        (\mu^1)^\top X_1 \mu^2+\tau \nu(\mu^1)
        -
        \max_{\hat{\mu}^1\in\Delta(\mathcal{A}^1)}
        \left\{
        (\hat{\mu}^1)^\top X_1\mu^2+\tau \nu(\hat{\mu}^1)
        \right\}
        \right].
    \end{align*}
    Therefore,
    \begin{align*}
        \langle \nabla_1V_X(\mu^1,\mu^2),\sigma_\tau(X_1\mu^2)-\mu^1 \rangle
        \leq\,&
        (\mu^1)^\top X_1 \mu^2+\tau \nu(\mu^1)
        -
        \max_{\hat{\mu}^1\in\Delta(\mathcal{A}^1)}
        \left\{
        (\hat{\mu}^1)^\top X_1\mu^2+\tau \nu(\hat{\mu}^1)
        \right\}\\
        &+
        ( \sigma_\tau(X_2 \mu^1)-\mu^2)^\top X_2( \sigma_\tau(X_1\mu^2)-\mu^1 ).
    \end{align*}
    Similarly,
    \begin{align*}
        \langle \nabla_2V_X(\mu^1,\mu^2),\sigma_\tau(X_2\mu^1)-\mu^2 \rangle
        \leq\,&
        (\mu^2)^\top X_2 \mu^1+\tau \nu(\mu^2)
        -
        \max_{\hat{\mu}^2\in\Delta(\mathcal{A}^2)}
        \left\{
        (\hat{\mu}^2)^\top X_2\mu^1+\tau \nu(\hat{\mu}^2)
        \right\}\\
        &+
        ( \sigma_\tau(X_1 \mu^2)-\mu^1)^\top X_1( \sigma_\tau(X_2\mu^1)-\mu^2 ).
    \end{align*}
    Adding the previous two inequalities gives
    \begin{align}
        &\langle \nabla_1V_X(\mu^1,\mu^2),\sigma_\tau(X_1\mu^2)-\mu^1 \rangle
        +
        \langle \nabla_2V_X(\mu^1,\mu^2),\sigma_\tau(X_2\mu^1)-\mu^2 \rangle\nonumber\\
        \leq\,&
        -V_X(\mu^1,\mu^2)
        +
        ( \sigma_\tau(X_1 \mu^2)-\mu^1)^\top
        (X_1+X_2^\top)
        ( \sigma_\tau(X_2\mu^1)-\mu^2 )\nonumber\\
        \leq\,&
        -V_X(\mu^1,\mu^2)
        +
        2\| \sigma_\tau(X_1 \mu^2)-\mu^1\|_2
        \|X_1+X_2^\top \|_2,
        \label{eq:gradient_V_1}
    \end{align}
    where the last line follows from $\| \sigma_\tau(X_2\mu^1)-\mu^2 \|_2\leq \| \sigma_\tau(X_2\mu^1)\|_1+\|\mu^2 \|_1\leq 2$.
    Using Part (1) together with the quadratic growth property of strongly convex functions, we have
    \begin{align*}
        \| \sigma_\tau(X_1 \mu^2)-\mu^1\|_2
        \leq
        \frac{\sqrt{2}}{\sqrt{\tau}}V_X(\mu^1,\mu^2)^{1/2}.
    \end{align*}
    It follows from \eqref{eq:gradient_V_1} that
    \begin{align*}
        &\langle \nabla_1V_X(\mu^1,\mu^2),\sigma_\tau(X_1\mu^2)-\mu^1 \rangle
        +
        \langle \nabla_2V_X(\mu^1,\mu^2),\sigma_\tau(X_2\mu^1)-\mu^2 \rangle\\
        \leq\,&
        -V_X(\mu^1,\mu^2)
        +
        \frac{2\sqrt{2}}{\sqrt{\tau}}
        V_X(\mu^1,\mu^2)^{1/2}\|X_1+X_2^\top \|_2\\
        \leq\,&
        -\frac{7}{8}V_X(\mu^1,\mu^2)
        +
        \frac{16}{\tau}\|X_1+X_2^\top \|_2^2.
    \end{align*}

    \item For any $u^1\in\mathbb{R}^{m_1}$, using \eqref{eq:V_gradient} and the optimality condition of the softmax map in inner-product form, we have
    \begin{align*}
        &\langle \nabla_1V_X(\mu^1,\mu^2),\sigma_\tau(u^1)-\sigma_\tau(X_1\mu^2)\rangle\\
        =\,&
        \tau\langle
        \nabla \nu(\sigma_\tau(X_1 \mu^2))- \nabla \nu(\mu^1),
        \sigma_\tau(u^1)-\sigma_\tau(X_1\mu^2)
        \rangle\\
        &+
        ( \sigma_\tau(X_2 \mu^1)-\mu^2)^\top X_2( \sigma_\tau(u^1)-\sigma_\tau(X_1\mu^2) )\\
        \leq\,&
        \left(
        \tau\|\nabla \nu(\sigma_\tau(X_1 \mu^2))- \nabla \nu(\mu^1)\|_2
        +\|\sigma_\tau(X_2 \mu^1)-\mu^2\|_2 \|X_2\|_2
        \right)
        \| \sigma_\tau(u^1)-\sigma_\tau(X_1\mu^2)\|_2\\
        \leq\,&
        \left(
        \frac{\tau}{\ell_\tau}\|\sigma_\tau(X_1 \mu^2)- \mu^1\|_2
        +\|\sigma_\tau(X_2 \mu^1)-\mu^2\|_2 \|X_2\|_2
        \right)
        \frac{1}{\tau}\| u^1-X_1\mu^2\|_2\\
        \leq\,&
        \frac{\sqrt{2}}{\sqrt{\tau}}
        \left(\frac{1}{\ell_\tau}+\frac{\|X_2\|_2}{\tau}\right)
        V_X(\mu^1,\mu^2)^{1/2}
        \| u^1-X_1\mu^2\|_2\\
        \leq\,&
        \frac{1}{16}V_X(\mu^1,\mu^2)
        +
        \frac{8}{\tau}
        \left(\frac{1}{\ell_\tau}+\frac{\|X_2\|_2}{\tau}\right)^2
        \| u^1-X_1\mu^2\|_2^2.
    \end{align*}
    Similarly, for any $u^2\in\mathbb{R}^{m_2}$,
    \begin{align*}
        \langle \nabla_2V_X(\mu^1,\mu^2),\sigma_\tau(u^2)-\sigma_\tau(X_2\mu^1) \rangle
        \leq\,&
        \frac{1}{16}V_X(\mu^1,\mu^2)
        +
        \frac{8}{\tau}
        \left(\frac{1}{\ell_\tau}+\frac{\|X_1\|_2}{\tau}\right)^2
        \| u^2-X_2\mu^1\|_2^2.
    \end{align*}
    Adding the previous two inequalities gives
    \begin{align*}
        &\langle \nabla_1V_X(\mu^1,\mu^2),\sigma_\tau(u^1)-\sigma_\tau(X_1\mu^2)\rangle
        +
        \langle \nabla_2V_X(\mu^1,\mu^2),\sigma_\tau(u^2)-\sigma_\tau(X_2\mu^1) \rangle\\
        \leq\,&
        \frac{1}{8}V_X(\mu^1,\mu^2)
        +
        \frac{8}{\tau}
        \left(
        \frac{1}{\ell_\tau}
        +\frac{\max(\|X_1\|_2,\|X_2\|_2)}{\tau}
        \right)^2
        \sum_{i=1,2}\| u^i-X_i\mu^{-i}\|_2^2.
    \end{align*}
\end{enumerate}
\hfill\qed
\endproof

\subsubsection{Proof of Lemma~\ref{le:policy_drift}}\label{pf:le:policy_drift}

We will use $V_{v,s}(\cdot,\cdot)$ (see Appendix~\ref{subsec:notation}) as the Lyapunov function to study the evolution of $(\pi_k^1(s),\pi_k^2(s))$.
To begin with, we identify the smoothness parameter of $V_{v,s}(\cdot,\cdot)$. Using Lemma~\ref{le:properties_Lyapunov_main} (1) and the definition of $V_{v,s}(\cdot,\cdot)$, we have
\begin{align*}
	\Tilde{L}_\tau
    =\,&
    2\left(\frac{\tau}{\ell_\tau}
    +\frac{\max(\|X_1\|_2^2,\|X_2\|_2^2)}{\tau}
    +\|X_1+X_2^\top\|_2\right)\\
	=\,&
    2\left(\frac{\tau}{\ell_\tau}
    +\frac{\max(\|\mathcal{T}^1(v^1)(s)\|_2^2,\|\mathcal{T}^2(v^2)(s)\|_2^2)}{\tau}
    +\|\mathcal{T}^1(v^1)(s)+\mathcal{T}^2(v^2)(s)^\top\|_2\right)\\
	\leq \,&
    2\left(\frac{\tau}{\ell_\tau}
    +\frac{m^2}{\tau(1-\gamma)^2}
    +\frac{2m}{1-\gamma}\right)
    :=L_\tau,
\end{align*}
where the inequality follows from $|\mathcal{T}^i(v^i)(s,a^i,a^{-i})|\leq \frac{1}{1-\gamma}$ for all $(s,a^i,a^{-i})$ and $i\in \{1,2\}$. Therefore, $V_{v,s}(\cdot,\cdot)$ is an $L_\tau$ -- smooth function on $\Pi_\tau$.

Using the smoothness of $V_{v,s}(\cdot,\cdot)$, for any $s\in\mathcal{S}$, we have by the policy update equation in Algorithm~\ref{algorithm:inner-loop}, Line~3, that
\begin{align*}
	&V_{v,s}(\pi_{k+1}^1(s),\pi_{k+1}^2(s))\\
    \leq \,&
    V_{v,s}(\pi_k^1(s),\pi_k^2(s))
    +\beta_k\langle \nabla_2V_{v,s}(\pi_k^1(s),\pi_k^2(s)),\sigma_\tau(q_k^2(s))-\pi_k^2(s) \rangle\\
    &+\beta_k\langle \nabla_1V_{v,s}(\pi_k^1(s),\pi_k^2(s)),\sigma_\tau(q_k^1(s))-\pi_k^1(s) \rangle
    +\frac{L_\tau\beta_k^2}{2}\sum_{i=1,2}\|\sigma_\tau(q_k^i(s))-\pi_k^i(s)\|_2^2\\
    \leq \,&
    V_{v,s}(\pi_k^1(s),\pi_k^2(s))
    +\beta_k\langle \nabla_2V_{v,s}(\pi_k^1(s),\pi_k^2(s)),\sigma_\tau(\mathcal{T}^2(v^2)(s)\pi_k^1(s))-\pi_k^2(s) \rangle\\
    &+\beta_k\langle \nabla_1 V_{v,s}(\pi_k^1(s),\pi_k^2(s)),\sigma_\tau(\mathcal{T}^1(v^1)(s)\pi_k^2(s))-\pi_k^1(s) \rangle\\
    &+\beta_k\langle \nabla_2V_{v,s}(\pi_k^1(s),\pi_k^2(s)),\sigma_\tau(q_k^2(s))-\sigma_\tau(\mathcal{T}^2(v^2)(s)\pi_k^1(s)) \rangle\\
    &+\beta_k\langle \nabla_1 V_{v,s}(\pi_k^1(s),\pi_k^2(s)),\sigma_\tau(q_k^1(s))-\sigma_\tau(\mathcal{T}^1(v^1)(s)\pi_k^2(s)) \rangle
    +2L_\tau\beta_k^2\\
    \leq \,&
    \left(1-\frac{3\beta_k}{4}\right)V_{v,s}(\pi_k^1(s),\pi_k^2(s))
    +\frac{16\beta_k}{\tau}\|\mathcal{T}^1(v^1)(s)+\mathcal{T}^2(v^2)(s)^\top \|_2^2\\
    &+\frac{8\beta_k}{\tau}
    \left(\frac{1}{\ell_\tau}
    +\frac{\max_{i\in \{1,2\}}\|\mathcal{T}^i(v^i)(s)\|_2}{\tau}\right)^2
    \sum_{i=1,2}\| q_k^i(s)-\mathcal{T}^i(v^i)(s)\pi_k^{-i}(s)\|_2^2
    +2L_\tau\beta_k^2,
\end{align*}
where the last line follows from Lemma~\ref{le:properties_Lyapunov_main} (3) and (4).

Since $\max_{i\in \{1,2\}}\|\mathcal{T}^i(v^i)(s)\|_2 \leq \frac{m}{1-\gamma}$ and
\begin{align*}
    \|\mathcal{T}^1(v^1)(s)+\mathcal{T}^2(v^2)(s)^\top \|_2^2
    \leq
    m^2\|v^1+v^2\|_\infty^2,
\end{align*}
we have
\begin{align*}
    &V_{v,s}(\pi_{k+1}^1(s),\pi_{k+1}^2(s))\\
    \leq \,&
    \left(1-\frac{3\beta_k}{4}\right)V_{v,s}(\pi_k^1(s),\pi_k^2(s))
    +\frac{16\beta_k m^2}{\tau}\|v^1+v^2\|_\infty^2\\
    &+\frac{8\beta_k}{\tau}
    \left(\frac{1}{\ell_\tau}+\frac{m}{\tau(1-\gamma)}\right)^2
    \sum_{i=1,2}\| q_k^i(s)-\mathcal{T}^i(v^i)(s)\pi_k^{-i}(s)\|_2^2
    +2L_\tau\beta_k^2\\
    \leq \,&
    \left(1-\frac{3\beta_k}{4}\right)
    \max_{s\in\mathcal{S}}V_{v,s}(\pi_k^1(s),\pi_k^2(s))
    +\frac{16\beta_k m^2}{\tau}\|v^1+v^2\|_\infty^2\\
    &+\frac{8\beta_k}{\tau}
    \left(\frac{1}{\ell_\tau}+\frac{m}{\tau(1-\gamma)}\right)^2
    \sum_{i=1,2}\sum_{s\in\mathcal{S}}\| q_k^i(s)-\mathcal{T}^i(v^i)(s)\pi_k^{-i}(s)\|_2^2
    +2L_\tau\beta_k^2\\
    =\,&
    \left(1-\frac{3\beta_k}{4}\right)\mathcal{L}_\pi(k)
    +\frac{16\beta_k m^2}{\tau}\|v^1+v^2\|_\infty^2\\
    &+\frac{8\beta_k}{\tau}
    \left(\frac{1}{\ell_\tau}+\frac{m}{\tau(1-\gamma)}\right)^2
    \mathcal{L}_q(k)
    +2L_\tau\beta_k^2.
\end{align*}
Since the right-hand side does not depend on $s$, we have
\begin{align*}
    \mathcal{L}_\pi(k+1)
    \leq \,&
    \left(1-\frac{3\beta_k}{4}\right)\mathcal{L}_\pi(k)
    +\frac{16\beta_k m^2}{\tau}\|v^1+v^2\|_\infty^2+\frac{8\beta_k}{\tau}
    \left(\frac{1}{\ell_\tau}+\frac{m}{\tau(1-\gamma)}\right)^2
    \mathcal{L}_q(k)
    +2L_\tau\beta_k^2\\
    \leq \,&
    \left(1-\frac{3\beta_k}{4}\right)\mathcal{L}_\pi(k)
    +\frac{16\beta_k m^2}{\tau}\|v^1+v^2\|_\infty^2+\frac{32m^2\beta_k}{\tau^3\ell_\tau^2(1-\gamma)^2}\mathcal{L}_q(k)
    +2L_\tau\beta_k^2,
\end{align*}
where the last line follows from $\tau\leq 1/(1-\gamma)$.
\hfill\qed
\endproof

\subsubsection{Proof of Lemma~\ref{le:operators}}\label{pf:le:operators}

\begin{enumerate}[(1)]
	\item For any $(q_1^i,q_2^i)$ and $(s_0,a_0^i,a_0^{-i},s_1)$, we have
	\begin{align*}
		&\|F^i(q_1^i,s_0,a_0^i,a_0^{-i},s_1)-F^i(q_2^i,s_0,a_0^i,a_0^{-i},s_1)\|_2^2\\
		=\,&
        \sum_{(s,a^i)}
        \left(
        [F^i(q_1^i,s_0,a_0^i,a_0^{-i},s_1)](s,a^i)
        -
        [F^i(q_2^i,s_0,a_0^i,a_0^{-i},s_1)](s,a^i)
        \right)^2\\
		=\,&
        \left(q_1^i(s_0,a_0^i)-q_2^i(s_0,a_0^i)\right)^2\\
		\leq\,&
        \|q_1^i-q_2^i\|_2^2.
	\end{align*}

	\item For any $(s_0,a_0^i,a_0^{-i},s_1)$, we have
	\begin{align*}
		\|F^i(0,s_0,a_0^i,a_0^{-i},s_1)\|_2^2
		=\,&
        \sum_{(s,a^i)}
        \left([F^i(0,s_0,a_0^i,a_0^{-i},s_1)](s,a^i)\right)^2\\
		=\,&
        \left(R_i(s_0,a_0^i,a_0^{-i})+\gamma v^i(s_1)\right)^2\\
		\leq\,&
        \frac{1}{(1-\gamma)^2},
	\end{align*}
	where the last line follows from $\|v^i\|_\infty\leq 1/(1-\gamma)$ and $|R_i(s_0,a_0^i,a_0^{-i})|\leq 1$.

	\item We first write the operator $\Bar{F}_k^i(\cdot)$ explicitly. Using the definition of $\mathcal{T}^i(\cdot)$, we have
	\begin{align*}
		\Bar{F}_k^i(q^i)(s)
        =
        \mu_k(s)\text{diag}(\pi_k^i(s))
        \left(\mathcal{T}^i(v^i)(s)\pi_k^{-i}(s)-q^i(s)\right),
        \quad \forall\,s\in\mathcal{S}.
	\end{align*}
	Since $\mu_k(s)\geq \mu_{\min}>0$ by Lemma~\ref{le:exploration} (4), and $\text{diag}(\pi_k^i(s))$ has strictly positive diagonal entries by Lemma~\ref{le:boundedness_proof_outline}, the equation $\Bar{F}_k^i(q^i)=0$ has a unique solution $\Bar{q}_k^i\in\mathbb{R}^{nm_i}$, given by
	\begin{align*}
		\Bar{q}_k^i(s)
        =
        \mathcal{T}^i(v^i)(s)\pi_k^{-i}(s),
        \quad \forall\,s\in\mathcal{S}.
	\end{align*}

	\item Using the expression of $\Bar{F}_k^i(\cdot)$, we have for any $q_1^i,q_2^i\in\mathbb{R}^{nm_i}$ that
	\begin{align*}
		(q_1^i-q_2^i)^\top (\Bar{F}_k^i(q_1^i)-\Bar{F}_k^i(q_2^i))
        =\,&
        -\sum_{s,a^i}
        \mu_k(s)\pi_k^i(a^i|s)
        (q_1^i(s,a^i)-q_2^i(s,a^i))^2\\
		\leq\,&
        -\min_{s,a^i}\mu_k(s)\pi_k^i(a^i|s)
        \|q_1^i-q_2^i\|_2^2\\
		\leq\,&
        -\mu_{\min}\ell_\tau\|q_1^i-q_2^i\|_2^2
        \tag{Lemma~\ref{le:boundedness_proof_outline} and Lemma~\ref{le:exploration}}\\
        =\,&
        -c_\tau\|q_1^i-q_2^i\|_2^2.
	\end{align*}
\end{enumerate}
This completes the proof.
\hfill\qed
\endproof

\subsubsection{Proof of Lemma~\ref{le:exploration}}\label{pf:le:exploration}
Let $\Pi_{\det}$ denote the set of deterministic stationary policy pairs. Since the state and action spaces are finite, $\Pi_{\det}$ is finite and $\Pi$ is compact.

For each $\pi\in\Pi_{\det}$, Assumption~\ref{as:MC} implies that the transition matrix $P_\pi$ of the induced Markov chain $\{S_k\}$ is irreducible and aperiodic. Hence, there exists a positive integer $r(\pi)$ such that $P_\pi^r$ has strictly positive entries for all $r\geq r(\pi)$; see, e.g., \cite[Proposition~1.7]{levin2017markov}. Define $r_*
    :=
    \max_{\pi\in\Pi_{\det}} r(\pi)$.
Then $P_\pi^{r_*}(s,s')>0$ for all $\pi\in\Pi_{\det}$ and all $s,s'\in\mathcal{S}$.

We next show that the same positivity holds uniformly over all stationary policy pairs. Fix any $\pi\in\Pi$. For each state $s\in\mathcal{S}$ and each player $i\in\{1,2\}$, let $a^i(s)\in\mathcal{A}^i$ be an action such that $\pi^i(a^i(s)\mid s)>0$.
Let $d=(d^1,d^2)$ be the deterministic stationary policy pair defined by $d^i(s)=a^i(s)$. Since $d\in\Pi_{\det}$, we have $P_{d}^{r_*}(s,s')>0$ for all $s,s'\in\mathcal{S}$. Therefore, for any $s,s'\in\mathcal{S}$, there exists at least one path $ s=s_0,s_1,\ldots,s_{r_*}=s'$
such that
\begin{align*}
    \prod_{\ell=0}^{r_*-1}
    P(s_{\ell+1}\mid s_\ell,d^1(s_\ell),d^2(s_\ell))
    >0.
\end{align*}
Along the same path, under the stationary policy pair $\pi$, the probability of choosing the action pair
$(d^1(s_\ell),d^2(s_\ell))$ at state $s_\ell$ is positive for every $\ell$. Hence $P_\pi^{r_*}(s,s')>0$ for all $s,s'\in\mathcal{S}$.

Since $P_\pi^{r_*}(s,s')$ is continuous in $\pi$ and $\Pi$ is compact, we have
\begin{align}\label{eq:uniform_positive_power}
    p_*
    :=
    \inf_{\pi\in\Pi}
    \min_{s,s'\in\mathcal{S}}
    P_\pi^{r_*}(s,s')>0.
\end{align}
The constants $r_*$ and $p_*$ depend only on the transition kernel and the finite state-action spaces, and not on $\ell_\tau$.

We now prove the four claims.

\begin{enumerate}[(1)]
    \item Since \eqref{eq:uniform_positive_power} implies that, for every $\pi\in\Pi$, the transition matrix $P_\pi$ of the induced Markov chain $\{S_k\}$ has a strictly positive power, $P_\pi$ is irreducible and aperiodic \cite[Proposition~1.7]{levin2017markov}. Therefore, there exists a unique stationary distribution $\mu_\pi$ for each $\pi\in\Pi$.
    \item Fix $\pi\in\Pi$. By \eqref{eq:uniform_positive_power} and the fact that $\mu_\pi(s')\leq 1$ for every $s'\in\mathcal{S}$, we have the minorization condition
\begin{align}\label{eq:minorization_mu}
    P_\pi^{r_*}(s,s')
    \geq
    p_* \mu_\pi(s'),
    \quad
    \forall s,s'\in\mathcal{S}.
\end{align}
Equivalently, for every $s\in\mathcal{S}$, $P_\pi^{r_*}
    =
    p_* \mathbf{1}\mu_\pi^\top
    +
    (1-p_*)Q_\pi$,
where $\mathbf{1}$ is the all-ones vector and $Q_\pi$ is a transition kernel. Now let $\xi$ be such that $\xi^\top \mathbf{1}=0$. Then, we have
\begin{align}\label{eq:doeblin_contraction}
    \|\xi^\top P_\pi^{r_*}\|_1
    =
    \|p_* \xi^\top \mathbf{1}\mu_\pi^\top
    +
    (1-p_*)\xi^\top Q_\pi\|_1
    =
    (1-p_*)\|\xi^\top Q_\pi\|_1\leq (1-p_*)\|\xi\|_1.
\end{align}

Let $k=qr_*+r$ with $q=\lfloor k/r_*\rfloor$ and $0\leq r<r_*$. For any $s\in\mathcal{S}$, let $\xi_s:=\delta_s-\mu_\pi$, where $\delta_s$ denotes the binary vector with its $s$-th entry being one and zero everywhere else.
Since $\mu_\pi^\top P_\pi=\mu_\pi^\top$, we have 
\begin{align*}
    \|P_\pi^k(s,\cdot)-\mu_\pi(\cdot)\|_{\text{TV}}
    =\,&\frac{1}{2}
    \|\xi_s^\top (P_\pi^{r_*})^q P_\pi^r\|_1\\
    \leq \,&\frac{1}{2}(1-p_*)^q\|\xi_s^\top P_\pi^r\|_1\\
    \leq \,&\frac{1}{2}(1-p_*)^q\|\xi_s\|_1\\
     \leq \,&(1-p_*)^q.
\end{align*}
Define $\rho_*:=\exp(-p_*/(2r_*))\in(0,1).$
Using $(1-p_*)^q\leq 2\rho_*^k$ for all $k\geq 0$, we obtain
\begin{align*}
    \sup_{\pi\in \Pi}
    \max_{s\in\mathcal{S}}
    \|P_\pi^k(s,\cdot)-\mu_\pi(\cdot)\|_{\text{TV}}
    \leq
    2\rho_*^k,
    \quad \forall\, k\geq 0.
\end{align*}
Consequently, $\sup_{\pi\in \Pi}\max_{s\in\mathcal{S}}\|P_\pi^k(s,\cdot)-\mu_\pi(\cdot)\|_{\text{TV}}\leq \eta$ holds whenever $2\rho_*^k\leq \eta$, that is,
\begin{align*}
    k
    \geq
    \frac{\log(2/\eta)}{\log(1/\rho_*)}.
\end{align*}
\item Fix
$\pi=(\pi^1,\pi^2),\bar{\pi}=(\bar{\pi}^1,\bar{\pi}^2)\in\Pi$ and write
$P=P_\pi$, $\bar P=P_{\bar{\pi}}$, $\mu=\mu_\pi$, and $\bar{\mu}=\mu_{\bar{\pi}}$. Define
\begin{align*}
    \Delta(\pi,\bar{\pi})
    :=
    \max_{s\in\mathcal{S}}\|\pi^1(s)-\bar{\pi}^1(s)\|_1
    +
    \max_{s\in\mathcal{S}}\|\pi^2(s)-\bar{\pi}^2(s)\|_1.
\end{align*}
For each $s\in\mathcal{S}$, we have
\begin{align*}
    \|P(s,\cdot)-\bar P(s,\cdot)\|_1
    \leq
    \|\pi^1(s)-\bar{\pi}^1(s)\|_1
    +
    \|\pi^2(s)-\bar{\pi}^2(s)\|_1
    \leq
    \Delta(\pi,\bar{\pi}).
\end{align*}
Therefore,
\begin{align}\label{eq:P_difference_bound}
    \max_{s\in\mathcal{S}}\|P(s,\cdot)-\bar P(s,\cdot)\|_1
    \leq
    \Delta(\pi,\bar{\pi}).
\end{align}

Let $x:= (P-\bar P)^\top \mu$. Since both $P$ and $\bar P$ are transition matrices, $x^\top \mathbf{1}=0$. Moreover, by \eqref{eq:P_difference_bound},
\begin{align}\label{eq:x_bound}
    \|x\|_1
    =
    \|\mu^\top (P-\bar P)\|_1
    \leq
    \max_{s\in\mathcal{S}}\|P(s,\cdot)-\bar P(s,\cdot)\|_1
    \leq
    \Delta(\pi,\bar{\pi}).
\end{align}
Since $\bar P$ is irreducible and aperiodic by Part (1) of this lemma, we have $\lim_{N\rightarrow\infty}\mu^\top \bar P^N\to\bar{\mu}^\top $ \cite[Theorem~4.9]{levin2017markov}. Hence
\begin{align*}
    \mu-\bar{\mu}
    =
    \lim_{N\to\infty}\mu(I-\bar P^N)
    =
    \lim_{N\to\infty}
    \sum_{j=0}^{N-1}\mu(I-\bar P)\bar P^j.
\end{align*}
Since $\mu^\top  P=\mu^\top$, we have $\mu^\top(I-\bar P)=\mu^\top(P-\bar P)=x$. Therefore,
\begin{align*}
    \|\mu-\bar{\mu}\|_1
    &\leq
    \sum_{j=0}^{\infty}\|x\bar P^j\|_1\\
    &\leq
    \left(\sum_{j=0}^{\infty}(1-p_*)^{\lfloor j/r_*\rfloor}\right)\Delta(\pi,\bar{\pi})\tag{Inequalities \eqref{eq:doeblin_contraction} and \eqref{eq:x_bound}}\\
    &\leq 
    \frac{r_*}{p_*}\Delta(\pi,\bar{\pi}).
\end{align*}
Thus Part~(3) holds with $L_p=r_*/p_*$,
which is independent of $\ell_\tau$.
\item Finally, for any $\pi\in\Pi$, using stationarity and \eqref{eq:uniform_positive_power}, we have for every $s'\in\mathcal{S}$,
\begin{align*}
    \mu_\pi(s')
    =
    \sum_{s\in\mathcal{S}}\mu_\pi(s)P_\pi^{r_*}(s,s')
    \geq
    \sum_{s\in\mathcal{S}}\mu_\pi(s)p_*
    =
    p_*.
\end{align*}
Taking the infimum over $\pi\in\Pi$ and the minimum over $s'\in\mathcal{S}$ gives
\begin{align*}
    \mu_{\min}
    =
    \inf_{\pi\in \Pi}\min_{s'\in\mathcal{S}}\mu_\pi(s')
    \geq
    p_*
    >0.
\end{align*}
\end{enumerate}

\subsubsection{Proof of Lemma \ref{le:q-function-drift}}\label{pf:le:q-function-drift}

Using $\|\cdot\|_2^2$ as a Lyapunov function, we have by the update equation \eqref{sa:reformulation_sketch} that
\begin{align}\label{eq:q_Lyapunov_decomposition}
	\mathbb{E}[\|q_{k+1}^i-\bar{q}_{k+1}^i\|_2^2]
	=\,&\mathbb{E}[\|q_{k+1}^i-q_k^i+q_k^i-\bar{q}_k^i+\bar{q}_k^i-\bar{q}_{k+1}^i\|_2^2]\nonumber\\
	=\,&\mathbb{E}[\|q_k^i-\bar{q}_k^i\|_2^2]
    +\mathbb{E}[\|q_{k+1}^i-q_k^i\|_2^2]
    +\mathbb{E}[\|\bar{q}_k^i-\bar{q}_{k+1}^i\|_2^2]\nonumber\\
    &+2\alpha_k\mathbb{E}[(q_k^i-\bar{q}_k^i)^\top \bar{F}_k^i(q_k^i)]\nonumber\\
	&+2\alpha_k\mathbb{E}[(F^i(q_k^i,S_k,A_k^i,A_k^{-i},S_{k+1})-\bar{F}_k^i(q_k^i))^\top (q_k^i-\bar{q}_k^i)]\nonumber\\
    &+2\mathbb{E}[(\bar{q}_k^i-\bar{q}_{k+1}^i)^\top (q_{k+1}^i-q_k^i)]\nonumber\\
    &+2\mathbb{E}[(q_k^i-\bar{q}_k^i)^\top (\bar{q}_k^i-\bar{q}_{k+1}^i)]\nonumber\\
    \leq \,&(1-2\alpha_k c_\tau)\mathbb{E}[\|q_k^i-\bar{q}_k^i\|_2^2]
    +\mathbb{E}[\|q_{k+1}^i-q_k^i\|_2^2]
    +\mathbb{E}[\|\bar{q}_k^i-\bar{q}_{k+1}^i\|_2^2]\nonumber\\
    &+2\mathbb{E}[(\bar{q}_k^i-\bar{q}_{k+1}^i)^\top (q_{k+1}^i-q_k^i)]\nonumber\\
    &+2\mathbb{E}[(q_k^i-\bar{q}_k^i)^\top (\bar{q}_k^i-\bar{q}_{k+1}^i)]\nonumber\\
    &+2\alpha_k\mathbb{E}[(F^i(q_k^i,S_k,A_k^i,A_k^{-i},S_{k+1})-\bar{F}_k^i(q_k^i))^\top (q_k^i-\bar{q}_k^i)],
\end{align}
where the last line follows from Lemma~\ref{le:operators} (4).

The terms $\mathbb{E}[\|q_{k+1}^i-q_k^i\|_2^2]$, $\mathbb{E}[\|\bar{q}_k^i-\bar{q}_{k+1}^i\|_2^2]$, $\mathbb{E}[(\bar{q}_k^i-\bar{q}_{k+1}^i)^\top (q_{k+1}^i-q_k^i)]$, and $\mathbb{E}[(q_k^i-\bar{q}_k^i)^\top (\bar{q}_k^i-\bar{q}_{k+1}^i)]$ on the right-hand side of \eqref{eq:q_Lyapunov_decomposition} are bounded in the following lemma, whose proof is presented in Appendix~\ref{pf:le:other_terms}.

\begin{lemma}\label{le:other_terms}
	The following inequalities hold for all $k\geq 0$.
	\begin{enumerate}[(1)]
		\item $\mathbb{E}[\|q_{k+1}^i-q_k^i\|_2^2]\leq \frac{4nm\alpha_k^2}{(1-\gamma)^2}$.
		\item $\mathbb{E}[\|\bar{q}_k^i-\bar{q}_{k+1}^i\|_2^2]\leq \frac{4nm\beta_k^2}{(1-\gamma)^2}$.
		\item $\mathbb{E}[\langle q_{k+1}^i-q_k^i,\bar{q}_k^i-\bar{q}_{k+1}^i\rangle]\leq \frac{4nm\alpha_k\beta_k}{(1-\gamma)^2}$.
		\item
		\begin{align*}
		\mathbb{E}[\langle q_k^i-\bar{q}_k^i,\bar{q}_k^i-\bar{q}_{k+1}^i\rangle]
		\leq
        \frac{17nm^2\beta_k}{\tau(1-\gamma)^2}\mathbb{E}[\| q_k^i-\bar{q}_k^i\|_2^2]
        +\frac{\beta_k}{16}\mathbb{E}[\mathcal{L}_\pi(k)].
        \end{align*}
	\end{enumerate}
\end{lemma}

We next consider the last term on the right-hand side of \eqref{eq:q_Lyapunov_decomposition}, which involves the difference between the operator $F^i(q_k^i,S_k,A_k^i,A_k^{-i},S_{k+1})$ and its expected version $\bar{F}_k^i(q_k^i)$, and hence can be viewed as the stochastic error due to sampling. Since the sample at time $k$ is generated using $\pi_{k+1}$ while $\bar{F}_k^i$ is defined using the stationary distribution and action distribution associated with $\pi_k$, the term $F^i(q_k^i,S_k,A_k^i,A_k^{-i},S_{k+1})-\bar{F}_k^i(q_k^i)$ contains both the Markovian sampling error and the one-step policy-shift error. Lemma~\ref{le:noise} controls both effects using the slow variation of the policies and the Lipschitz sensitivity of the stationary distribution.

The fact that the Markov chain $\{(S_k,A_k^i,A_k^{-i},S_{k+1})\}$ is time-inhomogeneous presents a challenge in our analysis. To overcome this challenge, observe that: (1) the policy, and hence the transition probability matrix of the induced Markov chain, is changing slowly compared to the $q$-function; see Algorithm~\ref{algorithm:inner-loop}, Line~3, and (2) the stationary distribution as a function of the policy is Lipschitz; see Lemma~\ref{le:exploration} (3). These two observations together enable us to develop a refined conditioning argument to handle the time-inhomogeneous Markovian noise. The result is presented in the following lemma, whose proof is presented in Appendix~\ref{pf:le:noise}.

\begin{lemma}\label{le:noise}
	The following inequality holds for all $k\geq z_k$:
	\begin{align*}
		\mathbb{E}[(F^i(q_k^i,S_k,A_k^i,A_k^{-i},S_{k+1})-\bar{F}_k^i(q_k^i))^\top (q_k^i-\bar{q}_k^i)]
        \leq
        \frac{17z_k\alpha_{k-z_k,k-1}}{(1-\gamma)^2},
	\end{align*}
    where we recall that $\alpha_{k_1,k_2}=\sum_{k=k_1}^{k_2}\alpha_k$.
\end{lemma}

Using the upper bounds we obtained for all the terms on the right-hand side of \eqref{eq:q_Lyapunov_decomposition}, we have the one-step Lyapunov drift inequality for $q_k^i$. Specifically, for $i\in \{1,2\}$, we have from \eqref{eq:q_Lyapunov_decomposition}, Lemma~\ref{le:other_terms}, and Lemma~\ref{le:noise} that
\begin{align*}
	\mathbb{E}[\|q_{k+1}^i-\bar{q}_{k+1}^i\|_2^2]
	\leq \,&(1-2\alpha_kc_\tau)\mathbb{E}[\|q_k^i-\bar{q}_k^i\|_2^2]
    +\frac{4nm}{(1-\gamma)^2}(\alpha_k^2+2\alpha_k\beta_k+\beta_k^2)\\
	&+\frac{34nm^2\beta_k}{\tau(1-\gamma)^2}\mathbb{E}[\| q_k^i-\bar{q}_k^i\|_2^2]
    +\frac{\beta_k}{8}\mathbb{E}[\mathcal{L}_\pi(k)]
    +\frac{34z_k\alpha_k\alpha_{k-z_k,k-1}}{(1-\gamma)^2}\\
	\leq \,&\left(1-2\alpha_kc_\tau+\frac{34nm^2\beta_k}{\tau(1-\gamma)^2}\right)
    \mathbb{E}[\|q_k^i-\bar{q}_k^i\|_2^2]
    +\frac{\beta_k}{8}\mathbb{E}[\mathcal{L}_\pi(k)]\\
	&+\frac{50nm}{(1-\gamma)^2}z_k\alpha_k\alpha_{k-z_k,k-1},
\end{align*}
where the second inequality follows from $\beta_k=c_{\alpha,\beta}\alpha_k$ with $c_{\alpha,\beta}\leq 1$, the monotonicity of $\{\alpha_k\}$, and $z_k\geq 1$, which imply $\alpha_k^2+2\alpha_k\beta_k+\beta_k^2
    \leq 4\alpha_k^2
    \leq 4z_k\alpha_k\alpha_{k-z_k,k-1}$.
Since
\begin{align*}
   c_{\alpha,\beta}\leq  \frac{c_\tau\tau(1-\gamma)^2}{34nm^2},
   \tag{Condition~\ref{con:stepsize_stochastic_game}}
\end{align*}
we have
\begin{align*}
    \mathbb{E}[\|q_{k+1}^i-\bar{q}_{k+1}^i\|_2^2]
    \leq \,&
    \left(1-\alpha_kc_\tau\right)\mathbb{E}[\|q_k^i-\bar{q}_k^i\|_2^2]
    +\frac{\beta_k}{8}\mathbb{E}[\mathcal{L}_\pi(k)]
    +\frac{50nm}{(1-\gamma)^2}z_k\alpha_k\alpha_{k-z_k,k-1}.
\end{align*}
Summing the previous inequality over $i=1,2$, we obtain
\begin{align*}
    \mathbb{E}[\mathcal{L}_q(k+1)]
    \leq\,&
    \left(1-\alpha_kc_\tau\right)\mathbb{E}[\mathcal{L}_q(k)]
    +\frac{\beta_k}{4}\mathbb{E}[\mathcal{L}_\pi(k)]
    +\frac{100nm}{(1-\gamma)^2}z_k\alpha_k\alpha_{k-z_k,k-1}.
\end{align*}
This completes the proof.
\hfill\qed
\endproof

\subsection{Solving Coupled Lyapunov Drift Inequalities}\label{sec:strategy}

We first restate the Lyapunov drift inequalities from previous sections. Recall our notation
$\mathcal{L}_q(t,k)=\sum_{i=1,2}\|q_{t,k}^i-\bar{q}_{t,k}^i\|_2^2$,
$\mathcal{L}_\pi(t,k)=\max_{s\in\mathcal{S}}V_{v_t,s}(\pi_{t,k}^1(s),\pi_{t,k}^2(s))$,
$\mathcal{L}_{\text{sum}}(t)=\|v_t^1+v_t^2\|_\infty$, and
$\mathcal{L}_v(t)=\sum_{i=1,2}\|v_t^i-v_*^i\|_\infty$. Let $\mathcal{F}_t$ be the history of Algorithm~\ref{algo:stochastic_game} right before the $t$-th outer-loop iteration. Note that $v_t^1$ and $v_t^2$ are both measurable with respect to $\mathcal{F}_t$. In what follows, for ease of presentation, we write $\mathbb{E}_t[\cdot]$ for $\mathbb{E}[\cdot\mid \mathcal{F}_t]$.

\begin{itemize}
	\item \textbf{Lemma~\ref{le:outer-loop}:} It holds for all $t\geq 0$ that
	\begin{align}\label{eq:Lyapunov_v}
        \mathcal{L}_v(t+1)
        \leq\,&
        \gamma \mathcal{L}_v(t)
        +4\mathcal{L}_{\text{sum}}(t)
        +2\mathcal{L}_q^{1/2}(t,K)
        +4\mathcal{L}_\pi(t,K)
        +6\tau \log(m).
	\end{align}

	\item \textbf{Lemma~\ref{le:outer-sum}:} It holds for all $t\geq 0$ that
	\begin{align}\label{eq:Lyapunov_v+}
        \mathcal{L}_{\text{sum}}(t+1)
        \leq
        \gamma\mathcal{L}_{\text{sum}}(t)
        +2\mathcal{L}_q^{1/2}(t,K).
	\end{align}

	\item \textbf{Lemma~\ref{le:policy_drift}:} It holds for all $t,k\geq 0$ that
	\begin{align}\label{eq:Lyapunov_pi}
		\mathbb{E}_t[\mathcal{L}_\pi(t,k+1)]
        \leq\,&
        \left(1-\frac{3\beta_k}{4}\right)\mathbb{E}_t[\mathcal{L}_\pi(t,k)]
        +\frac{16m^2\beta_k}{\tau}\mathcal{L}_{\text{sum}}(t)^2\nonumber\\
        &+
        \frac{32m^2\beta_k}{\tau^3\ell_\tau^2(1-\gamma)^2}
        \mathbb{E}_t[\mathcal{L}_q(t,k)]
        +2L_\tau\beta_k^2.
	\end{align}

	\item \textbf{Lemma~\ref{le:q-function-drift}:} It holds for all $t\geq 0$ and $k\geq z_k$ that
	\begin{align}\label{eq:Lyapunov_q}
		\mathbb{E}_t[\mathcal{L}_q(t,k+1)]
        \leq\,&
        \left(1-\alpha_kc_\tau\right)\mathbb{E}_t[\mathcal{L}_q(t,k)]
        +\frac{\beta_k}{4}\mathbb{E}_t[\mathcal{L}_\pi(t,k)]\\
        &+
        \frac{100nm}{(1-\gamma)^2}z_k\alpha_k\alpha_{k-z_k,k-1}.\nonumber
	\end{align}
\end{itemize}

Adding \eqref{eq:Lyapunov_pi} and \eqref{eq:Lyapunov_q}, and using
$c_{\alpha,\beta}\leq \min\{L_\tau^{-1/2},c_\tau \tau^3\ell_\tau^2(1-\gamma)^2/(128m^2),c_\tau\}$ from Condition~\ref{con:stepsize_stochastic_game}, we obtain
\begin{align}
    \mathbb{E}_t[\mathcal{L}_\pi(t,k+1)+\mathcal{L}_q(t,k+1)]
    \leq\,&
    \left(1-\frac{\beta_k}{2}\right)
    \mathbb{E}_t[\mathcal{L}_\pi(t,k)+\mathcal{L}_q(t,k)]\nonumber\\
	&+
    \frac{16m^2 \beta_k}{\tau}\mathcal{L}_{\text{sum}}(t)^2
    +\frac{102nm}{(1-\gamma)^2}z_k\alpha_k\alpha_{k-z_k,k-1}.
    \label{eq:before_stepsize}
\end{align}

\subsubsection{Constant Stepsize}

When using constant stepsizes, i.e., $\alpha_k\equiv \alpha$ and $\beta_k\equiv \beta=c_{\alpha,\beta}\alpha$, iterating \eqref{eq:before_stepsize} from $z_\beta$ to $k$ gives
\begin{align}
	\mathbb{E}_t[\mathcal{L}_\pi(t,k)+\mathcal{L}_q(t,k)]
	\leq\,&
    \left(1-\frac{\beta}{2}\right)^{k-z_\beta}
    (\mathcal{L}_\pi(t,0)+\mathcal{L}_q(t,0))\nonumber\\
	&+
    \frac{32m^2}{\tau}\mathcal{L}_{\text{sum}}(t)^2
    +\frac{204nm}{(1-\gamma)^2c_{\alpha,\beta}}z_\beta^2\alpha.
    \label{eq:polish1}
\end{align}
We next bound $\mathcal{L}_\pi(t,0)+\mathcal{L}_q(t,0)$. For $i\in \{1,2\}$, we have
\begin{align*}
    \mathcal{L}_\pi(t,0)
    =\,&
    \max_{s}V_{v_t,s}(\pi_{t,0}^1(s),\pi_{t,0}^2(s))\\
    =\,&
    \max_{s}
    \sum_{i=1,2}
    \max_{\mu^i}
    \left\{
    (\mu^i-\pi_{t,0}^i(s))^\top \mathcal{T}^i(v_t^i)(s)\pi_{t,0}^{-i}(s)
    +\tau \nu(\mu^i)-\tau\nu(\pi_{t,0}^i(s))
    \right\}\\
    \leq\,&
    2\sum_{i=1,2}\max_{s,a^i,a^{-i}}|\mathcal{T}^i(v_t^i)(s,a^i,a^{-i})|
    +2\tau \log(m)\\
    \leq\,&
    \frac{4}{1-\gamma}+2\tau \log(m),
\end{align*}
and
\begin{align*}
    \mathcal{L}_q(t,0)
    =
    \sum_{i=1,2}\|q_{t,0}^i-\bar{q}_{t,0}^i\|_2^2
    \leq
    \frac{8nm}{(1-\gamma)^2}.
    \tag{Lemma~\ref{le:boundedness_proof_outline}}
\end{align*}
It follows that
\begin{align*}
    \mathcal{L}_\pi(t,0)+\mathcal{L}_q(t,0)
    \leq
    \frac{4}{1-\gamma}+2\tau \log(m)+\frac{8nm}{(1-\gamma)^2}
    =
    L_{\text{in}}.
\end{align*}
Using this bound in \eqref{eq:polish1}, we have
\begin{align}\label{eq:cici}
    \mathbb{E}_t[\mathcal{L}_\pi(t,k)+\mathcal{L}_q(t,k)]
	\leq\,&
    L_{\text{in}}\left(1-\frac{\beta}{2}\right)^{k-z_\beta}
    +\frac{32m^2}{\tau}\mathcal{L}_{\text{sum}}(t)^2
    +\frac{204nm}{(1-\gamma)^2c_{\alpha,\beta}}z_\beta^2\alpha.
\end{align}
In particular,
\begin{align*}
	\mathbb{E}_t[\mathcal{L}_\pi(t,k)]
	\leq\,&
    L_{\text{in}}\left(1-\frac{\beta}{2}\right)^{k-z_\beta}
    +\frac{32m^2}{\tau}\mathcal{L}_{\text{sum}}(t)^2
    +\frac{204nm}{(1-\gamma)^2c_{\alpha,\beta}}z_\beta^2\alpha.
\end{align*}
Substituting this inequality into \eqref{eq:Lyapunov_q}, we have
\begin{align*}
    \mathbb{E}_t[\mathcal{L}_q(t,k+1)]
    \leq\,&
    \left(1-\alpha c_\tau\right)\mathbb{E}_t[\mathcal{L}_q(t,k)]
    +\frac{151nm}{(1-\gamma)^2}z_\beta^2\alpha^2\\
    &+
    \frac{\beta L_{\text{in}}}{4}
    \left(1-\frac{\beta}{2}\right)^{k-z_\beta}
    +\frac{8m^2\beta}{\tau}\mathcal{L}_{\text{sum}}(t)^2.
\end{align*}
Iterating the previous inequality and using $c_{\alpha,\beta}\leq c_\tau$ gives
\begin{align*}
	\mathbb{E}_t[\mathcal{L}_q(t,k)]
	\leq\,&
    L_{\text{in}}\left(1-c_\tau\alpha\right)^{k-z_\beta}
    +\frac{\beta L_{\text{in}}(k-z_\beta)}{4}
    \left(1-\frac{\beta}{2}\right)^{k-z_\beta-1}\\
	&+
    \frac{8m^2c_{\alpha,\beta}}{c_\tau\tau}\mathcal{L}_{\text{sum}}(t)^2
    +\frac{151nm}{(1-\gamma)^2c_\tau}z_\beta^2\alpha.
\end{align*}
Thus, by Jensen's inequality,
\begin{align*}
    \mathbb{E}_t[\mathcal{L}_q(t,k)^{1/2}]
	\leq\,&
    L_{\text{in}}^{1/2}\left(1-c_\tau\alpha\right)^{\frac{k-z_\beta}{2}}
    +\frac{\beta^{1/2} L_{\text{in}}^{1/2}(k-z_\beta)^{1/2}}{2}
    \left(1-\frac{\beta}{2}\right)^{\frac{k-z_\beta-1}{2}}\\
	&+
    \frac{3mc_{\alpha,\beta}^{1/2}}{c_\tau^{1/2}\tau^{1/2}}\mathcal{L}_{\text{sum}}(t)
    +\frac{13n^{1/2}m^{1/2}}{(1-\gamma)c_\tau^{1/2}}z_\beta\alpha^{1/2}.
\end{align*}
Substituting the previous bound into \eqref{eq:Lyapunov_v+} and then taking total expectation, we have
\begin{align*}
    \mathbb{E}[\mathcal{L}_{\text{sum}}(t+1)]
    \leq\,&
    \gamma\mathbb{E}[\mathcal{L}_{\text{sum}}(t)]
    +2L_{\text{in}}^{1/2}\left(1-c_\tau\alpha\right)^{\frac{K-z_\beta}{2}}\\
    &+
    \beta^{1/2}L_{\text{in}}^{1/2}(K-z_\beta)^{1/2}
    \left(1-\frac{\beta}{2}\right)^{\frac{K-z_\beta-1}{2}}\\
	&+
    \frac{6mc_{\alpha,\beta}^{1/2}}{c_\tau^{1/2}\tau^{1/2}}
    \mathbb{E}[\mathcal{L}_{\text{sum}}(t)]
    +\frac{26n^{1/2}m^{1/2}}{(1-\gamma)c_\tau^{1/2}}z_\beta\alpha^{1/2}\\
    \leq\,&
    \left(\frac{1+\gamma}{2}\right)\mathbb{E}[\mathcal{L}_{\text{sum}}(t)]
    +2L_{\text{in}}^{1/2}\left(1-c_\tau\alpha\right)^{\frac{K-z_\beta}{2}}\\
    &+
    \beta^{1/2}L_{\text{in}}^{1/2}(K-z_\beta)^{1/2}
    \left(1-\frac{\beta}{2}\right)^{\frac{K-z_\beta-1}{2}}
    +\frac{26n^{1/2}m^{1/2}}{(1-\gamma)c_\tau^{1/2}}z_\beta\alpha^{1/2},
\end{align*}
where the last line follows from
$c_{\alpha,\beta}\leq c_\tau\tau(1-\gamma)^2/(144m^2)$; see Condition~\ref{con:stepsize_stochastic_game}. Since $\|v_0^1+v_0^2\|_\infty\leq 2/(1-\gamma)$, iterating the previous inequality gives
\begin{align}\label{eq:cicipp}
    \mathbb{E}[\mathcal{L}_{\text{sum}}(t)]
    \leq\,&
    \frac{2}{1-\gamma}\left(\frac{1+\gamma}{2}\right)^t
    +\frac{4L_{\text{in}}^{1/2}\left(1-c_\tau\alpha\right)^{\frac{K-z_\beta}{2}}}{1-\gamma}\nonumber\\
    &+
    \frac{2\beta^{1/2}L_{\text{in}}^{1/2}(K-z_\beta)^{1/2}}{1-\gamma}
    \left(1-\frac{\beta}{2}\right)^{\frac{K-z_\beta-1}{2}}
    +\frac{52n^{1/2}m^{1/2}}{(1-\gamma)^2c_\tau^{1/2}}z_\beta\alpha^{1/2}\nonumber\\
    \leq\,&
    \frac{2}{1-\gamma}\left(\frac{1+\gamma}{2}\right)^t
    +\frac{6L_{\text{in}}^{1/2}(K-z_\beta)^{1/2}}{1-\gamma}
    \left(1-\frac{\beta}{2}\right)^{\frac{K-z_\beta-1}{2}}
    +\frac{52n^{1/2}m^{1/2}}{(1-\gamma)^2c_\tau^{1/2}}z_\beta\alpha^{1/2}.
\end{align}
Now we have obtained finite-sample bounds for $\mathcal{L}_q(t,k)$, $\mathcal{L}_\pi(t,k)$, and $\mathcal{L}_{\text{sum}}(t)$. The next step is to use them in \eqref{eq:Lyapunov_v} to obtain a finite-sample bound for $\mathcal{L}_v(t)$. Specifically, using \eqref{eq:Lyapunov_v}, \eqref{eq:cici}, and \eqref{eq:cicipp}, we have
\begin{align*}
    \mathbb{E}[\mathcal{L}_v(t+1)]
    \leq\,&
    \gamma \mathbb{E}[\mathcal{L}_v(t)]
    +4\mathbb{E}[\mathcal{L}_{\text{sum}}(t)]
    +2\mathbb{E}[\mathcal{L}_q^{1/2}(t,K)]
    +4\mathbb{E}[\mathcal{L}_\pi(t,K)]
    +6\tau \log(m)\\
    \leq\,&
    \gamma \mathbb{E}[\mathcal{L}_v(t)]
    +\frac{266m^2}{\tau(1-\gamma)^2}
    \left(\frac{1+\gamma}{2}\right)^t\\
    &+
    \frac{805m^2L_{\text{in}}(K-z_\beta)^{1/2}}{(1-\gamma)^2\tau}
    \left(1-\frac{\beta}{2}\right)^{\frac{K-z_\beta-1}{2}}\\
    &+
    \frac{1223nm}{(1-\gamma)^2c_{\alpha,\beta}}z_\beta^2\alpha^{1/2}
    +6\tau \log(m).
\end{align*}
Iterating the previous inequality from $0$ to $T-1$ and using $\mathcal{L}_v(0)\leq 4/(1-\gamma)$, we obtain
\begin{align*}
    \mathbb{E}[\mathcal{L}_v(T)]
    \leq\,&
    \frac{270m^2T}{\tau(1-\gamma)^2}
    \left(\frac{1+\gamma}{2}\right)^{T-1}+
    \frac{805m^2L_{\text{in}}(K-z_\beta)^{1/2}}{\tau(1-\gamma)^3}
    \left(1-\frac{\beta}{2}\right)^{\frac{K-z_\beta-1}{2}}\\
    &+
    \frac{1223nm}{(1-\gamma)^3c_{\alpha,\beta}}z_\beta^2\alpha^{1/2}
    +\frac{6\tau \log(m)}{1-\gamma}.
\end{align*}

Our next step is to use the bounds obtained for $\mathcal{L}_q(t,k)$, $\mathcal{L}_\pi(t,k)$, $\mathcal{L}_{v}(t)$, and $\mathcal{L}_{\text{sum}}(t)$ in Lemma~\ref{le:Nash_Combine}. For simplicity of presentation, we use $a\lesssim b$ to mean that there exists a numerical constant $c$ such that $a\leq cb$. Using the previous inequality, \eqref{eq:cici}, and \eqref{eq:cicipp}, we have
\begin{align*}
    \mathbb{E}[\text{NG}(\pi_{T,K}^1,\pi_{T,K}^2)]
    \leq\,&
    \frac{8}{1-\gamma}\mathbb{E}[\mathcal{L}_{\text{sum}}(T)]
    +\frac{4}{1-\gamma}\mathbb{E}[\mathcal{L}_v(T)]
    +\frac{4}{1-\gamma}\mathbb{E}[\mathcal{L}_\pi(T,K)]
    +\frac{8\tau \log(m)}{1-\gamma}\\
    \lesssim\,&
    \frac{m^2T}{\tau(1-\gamma)^3}
    \left(\frac{1+\gamma}{2}\right)^{T-1}
    +\frac{m^2L_{\text{in}}(K-z_\beta)^{1/2}}{\tau(1-\gamma)^4}
    \left(1-\frac{\beta}{2}\right)^{\frac{K-z_\beta-1}{2}}\\
    &+
    \frac{nm}{(1-\gamma)^4c_{\alpha,\beta}}z_\beta^2\alpha^{1/2}
    +\frac{\tau \log(m)}{(1-\gamma)^2}.
\end{align*}
The proof of Theorem~\ref{thm:stochastic_game} (1) is complete.

\subsubsection{Diminishing Stepsizes}

Consider using harmonically diminishing stepsizes, i.e., $\alpha_k=\frac{\alpha}{k+h}$, $\beta_k=\frac{\beta}{k+h}$, and $\beta=c_{\alpha,\beta}\alpha$. Iterating \eqref{eq:before_stepsize}, we have for all $k\geq k_0:=\min\{k'\mid k'\geq z_{k'}\}$ that
\begin{align*}
	\mathbb{E}_t[\mathcal{L}_\pi(t,k)+\mathcal{L}_q(t,k)]
	\leq \,&
    L_{\text{in}}
    \underbrace{\prod_{m=k_0}^{k-1}\left(1-\frac{\beta_m}{2}\right)}_{\hat{\mathcal{E}}_1}
    +\frac{204nm}{(1-\gamma)^2}
    \underbrace{\sum_{n=k_0}^{k-1}z_n^2\alpha_n^2
    \prod_{m=n+1}^{k-1}\left(1-\frac{\beta_m}{2}\right)}_{\hat{\mathcal{E}}_2}\\
	&+\frac{16m^2}{\tau}\mathcal{L}_{\text{sum}}(t)^2
    \underbrace{\sum_{n=k_0}^{k-1}\beta_n
    \prod_{m=n+1}^{k-1}\left(1-\frac{\beta_m}{2}\right)}_{\hat{\mathcal{E}}_3}.
\end{align*}
Next, we evaluate the terms $\{\hat{\mathcal{E}}_j\}_{1\leq j\leq 3}$. Terms of this form have been well studied in the existing literature \citep{srikant2019finite,lan2020first,chen2021finite}. Specifically, using the same line of analysis as in \cite[Appendix A.2]{chen2021finite} and $\beta=4$, we have
\begin{align*}
	\hat{\mathcal{E}}_1\leq
	\frac{k_0+h}{k+h},\quad
	\hat{\mathcal{E}}_2\leq
	\frac{64ez_k^2}{(k+h)c_{\alpha,\beta}^2},\quad
	\text{and}\quad
	\hat{\mathcal{E}}_3\leq 2.
\end{align*}
It follows that
\begin{align*}
	\mathbb{E}_t[\mathcal{L}_\pi(t,k)+\mathcal{L}_q(t,k)]
	\leq\,&
    L_{\text{in}}\frac{k_0+h}{k+h}
    +\frac{3264enm}{(1-\gamma)^2c_{\alpha,\beta}}z_k^2\alpha_k
    +\frac{32m^2}{\tau}\mathcal{L}_{\text{sum}}(t)^2,
\end{align*}
which implies
\begin{align}
    \mathbb{E}_t[\mathcal{L}_\pi(t,k)]
	\leq\,&
    L_{\text{in}}\frac{k_0+h}{k+h}
    +\frac{3264enm}{(1-\gamma)^2c_{\alpha,\beta}}z_k^2\alpha_k
    +\frac{32m^2}{\tau}\mathcal{L}_{\text{sum}}(t)^2.
    \label{eq:ci_di}
\end{align}
Using the previous inequality on $\mathbb{E}_t[\mathcal{L}_\pi(t,k)]$ in \eqref{eq:Lyapunov_q}, we have
\begin{align*}
    \mathbb{E}_t[\mathcal{L}_q(t,k+1)]
    \leq \,&
    \left(1-\alpha_kc_\tau\right)\mathbb{E}_t[\mathcal{L}_q(t,k)]
    +\frac{100nm}{(1-\gamma)^2}z_k\alpha_k\alpha_{k-z_k,k-1}\\
    &+\frac{L_{\text{in}} c_{\alpha,\beta}\alpha_k}{4}\frac{k_0+h}{k+h}
    +\frac{816enm}{(1-\gamma)^2}z_k^2\alpha_k^2
    +\frac{8m^2\beta_k}{\tau}\mathcal{L}_{\text{sum}}(t)^2\\
    \leq \,&
    \left(1-\alpha_kc_\tau\right)\mathbb{E}_t[\mathcal{L}_q(t,k)]
    +\frac{1017eL_{\text{in}}nm}{(1-\gamma)^2\alpha_{k_0}}z_k^2\alpha_k^2
    +\frac{8m^2\beta_k}{\tau}\mathcal{L}_{\text{sum}}(t)^2.
\end{align*}
Here the last inequality uses the monotonicity of $\{\alpha_k\}$, the definition of $k_0$, and the fact that $z_k=\mathcal{O}(\log k)$; the constants are numerical and chosen conservatively.

Iterating the previous inequality starting from $k_0$, and using $\alpha c_\tau\geq 1$ from Condition~\ref{con:stepsize_stochastic_game}, we have
\begin{align*}
    \mathbb{E}_t[\mathcal{L}_q(t,k)]
    \leq\,&
    L_{\text{in}}\frac{k_0+h}{k+h}
    +\frac{4068e^2L_{\text{in}}nm}{(1-\gamma)^2c_\tau\alpha_{k_0}}z_k^2\alpha_k
    +\frac{8m^2c_{\alpha,\beta}}{c_\tau\tau}\mathcal{L}_{\text{sum}}(t)^2.
\end{align*}
Therefore, by Jensen's inequality,
\begin{align}
    \mathbb{E}_t[\mathcal{L}_q(t,k)^{1/2}]
    \leq
    L_{\text{in}}^{1/2}\left(\frac{k_0+h}{k+h}\right)^{1/2}
    +\frac{64eL_{\text{in}}^{1/2}n^{1/2}m^{1/2}}{(1-\gamma)c_\tau^{1/2}\alpha_{k_0}^{1/2}}
    z_k\alpha_k^{1/2}+
    \frac{3m c_{\alpha,\beta}^{1/2}}{c_\tau^{1/2}\tau^{1/2}}\mathcal{L}_{\text{sum}}(t).
    \label{eeeq}
\end{align}
Taking total expectation on both sides of the previous inequality and then using the result in \eqref{eq:Lyapunov_v+}, we obtain
\begin{align*}
    \mathbb{E}[\mathcal{L}_{\text{sum}}(t+1)]
    \leq\,&
    \gamma\mathbb{E}[\mathcal{L}_{\text{sum}}(t)]
    +2L_{\text{in}}^{1/2}\left(\frac{k_0+h}{K+h}\right)^{1/2}\\
    &+
    \frac{128eL_{\text{in}}^{1/2}n^{1/2}m^{1/2}}{(1-\gamma)c_\tau^{1/2}\alpha_{k_0}^{1/2}}
    z_K\alpha_K^{1/2}
    +\frac{6m c_{\alpha,\beta}^{1/2}}{c_\tau^{1/2}\tau^{1/2}}
    \mathbb{E}[\mathcal{L}_{\text{sum}}(t)]\\
    \leq\,&
    \left(\frac{\gamma+1}{2}\right)\mathbb{E}[\mathcal{L}_{\text{sum}}(t)]
    +\frac{130eL_{\text{in}}^{1/2}n^{1/2}m^{1/2}}{(1-\gamma)c_\tau^{1/2}\alpha_{k_0}^{1/2}}
    z_K\alpha_K^{1/2},
\end{align*}
where the last line follows from $c_{\alpha,\beta}\leq \frac{c_\tau\tau(1-\gamma)^2}{144m^2}$; see Condition~\ref{con:stepsize_stochastic_game}. Iterating the previous inequality starting from $0$, we have
\begin{align}
    \mathbb{E}[\mathcal{L}_{\text{sum}}(t)]
    \leq\,&
    \frac{2}{1-\gamma}\left(\frac{1+\gamma}{2}\right)^t
    +\frac{260eL_{\text{in}}^{1/2}n^{1/2}m^{1/2}}{(1-\gamma)^2c_\tau^{1/2}\alpha_{k_0}^{1/2}}
    z_K\alpha_K^{1/2}.
    \label{eq:zls}
\end{align}

The next step is to bound $\mathcal{L}_v(t)$. Recall from \eqref{eq:Lyapunov_v} that
\begin{align*}
    \mathbb{E}_t[\mathcal{L}_v(t+1)]
    \leq\,&
    \gamma \mathcal{L}_v(t)
    +4\mathcal{L}_{\text{sum}}(t)
    +2\mathbb{E}_t[\mathcal{L}_q^{1/2}(t,K)]
    +4\mathbb{E}_t[\mathcal{L}_\pi(t,K)]
    +6\tau \log(m).
\end{align*}
Using \eqref{eq:ci_di}, \eqref{eeeq}, and \eqref{eq:zls} in the previous inequality, we obtain
\begin{align*}
    \mathbb{E}[\mathcal{L}_v(t+1)]
    \leq\,&
    \gamma \mathbb{E}[\mathcal{L}_v(t)]
    +4\mathbb{E}[\mathcal{L}_{\text{sum}}(t)]
    +2\mathbb{E}[\mathcal{L}_q^{1/2}(t,K)]
    +4\mathbb{E}[\mathcal{L}_\pi(t,K)]
    +6\tau \log(m)\\
    \leq\,&
    \gamma \mathbb{E}[\mathcal{L}_v(t)]
    +\frac{130eL_{\text{in}}^{1/2}n^{1/2}m^{1/2}}{(1-\gamma)c_\tau^{1/2}\alpha_{k_0}^{1/2}}
    z_K\alpha_K^{1/2}\\
    &+
    \frac{4L_{\text{in}}\alpha_K}{\alpha_{k_0}}
    +\frac{13056enm}{(1-\gamma)^2c_{\alpha,\beta}}z_K^2\alpha_K
    +6\tau \log(m)\\
    &+
    \frac{522m^2}{\tau(1-\gamma)^2}\left(\frac{1+\gamma}{2}\right)^t
    +\frac{67860eL_{\text{in}}^{1/2}n^{1/2}m^{5/2}}{(1-\gamma)^3 \tau c_\tau^{1/2}\alpha_{k_0}^{1/2}}
    z_K\alpha_K^{1/2}\\
    \leq\,&
    \gamma \mathbb{E}[\mathcal{L}_v(t)]
    +\frac{522m^2}{\tau(1-\gamma)^2}\left(\frac{1+\gamma}{2}\right)^t+
    \frac{15056eL_{\text{in}}nm}{(1-\gamma)^2\alpha_{k_0}^{1/2}c_{\alpha,\beta}}
    z_K^2\alpha_K^{1/2}
    +6\tau \log(m).
\end{align*}
Iterating the previous inequality from $0$ to $T-1$ and using $\mathcal{L}_v(0)\leq \frac{4}{1-\gamma}$, we have
\begin{align*}
	\mathbb{E}[\mathcal{L}_v(T)]
    \leq
    \frac{526m^2T}{\tau(1-\gamma)^2}
    \left(\frac{1+\gamma}{2}\right)^{T-1}+
    \frac{15056eL_{\text{in}}nm}{(1-\gamma)^3\alpha_{k_0}^{1/2}c_{\alpha,\beta}}
    z_K^2\alpha_K^{1/2}
    +\frac{6\tau \log(m)}{1-\gamma}.
\end{align*}
Finally, using the previous inequality, \eqref{eq:ci_di}, and \eqref{eq:zls} in Lemma~\ref{le:Nash_Combine}, we obtain
\begin{align*}
    \mathbb{E}[\text{NG}(\pi_{T,K}^1,\pi_{T,K}^2)]
    \lesssim
    \frac{m^2 T}{\tau(1-\gamma)^3}
    \left(\frac{1+\gamma}{2}\right)^{T-1}
    +\frac{L_{\text{in}}nm}{(1-\gamma)^4\alpha_{k_0}^{1/2}c_{\alpha,\beta}}
    z_K^2\alpha_K^{1/2}+
    \frac{\tau \log(m)}{(1-\gamma)^2}.
\end{align*}
The proof of Theorem~\ref{thm:stochastic_game} (2) is complete.

\subsection{Proofs of Supporting Lemmas}\label{pflemmas_stochastic}

\subsubsection{Proof of Lemma \ref{le:other_terms}}\label{pf:le:other_terms}

\begin{enumerate}[(1)]
	\item For any $k\geq 0$, using the update equation \eqref{sa:reformulation_sketch} and Lemma \ref{le:operators} (1), we have
	\begin{align*}
		\|q_{k+1}^i-q_k^i\|_2^2
		=\,&\alpha_k^2\|F^i(q_k^i,S_k,A_k^i,A_k^{-i},S_{k+1})\|_2^2\\
		=\,&\alpha_k^2\|F^i(q_k^i,S_k,A_k^i,A_k^{-i},S_{k+1})-F^i(0,S_k,A_k^i,A_k^{-i},S_{k+1})\\
		&+F^i(0,S_k,A_k^i,A_k^{-i},S_{k+1})\|_2^2\\
		\leq \,&\alpha_k^2\left(\|q_k^i\|_2+\frac{1}{1-\gamma}\right)^2\\
		\leq \,&\alpha_k^2\left(\frac{\sqrt{nm}}{1-\gamma}+\frac{1}{1-\gamma}\right)^2
        \tag{$\|q_k^i\|_\infty\leq \frac{1}{1-\gamma}$ by Lemma \ref{le:boundedness_proof_outline}}\\
		\leq \,&\frac{4nm\alpha_k^2}{(1-\gamma)^2}.
	\end{align*}
	The result follows by taking expectation on both sides of the previous inequality.

	\item For any $k\geq 0$, using the definition of $\Bar{q}_k^i$ in Appendix \ref{subsec:notation}, we have
	\begin{align*}
		\|\bar{q}_k^i-\bar{q}_{k+1}^i\|_2^2
        =\,&
        \sum_{s}\|\mathcal{T}^i(v^i)(s)(\pi_{k+1}^{-i}(s)-\pi_k^{-i}(s))\|_2^2\\
		=\,&
        \beta_k^2\sum_{s}\|\mathcal{T}^i(v^i)(s)(\sigma_\tau(q_k^{-i}(s))-\pi_k^{-i}(s))\|_2^2\\
		\leq \,&
        \beta_k^2\sum_{s}
        \left(
        \|\mathcal{T}^i(v^i)(s)\sigma_\tau(q_k^{-i}(s))\|_2
        +
        \|\mathcal{T}^i(v^i)(s)\pi_k^{-i}(s)\|_2
        \right)^2\\
		\leq \,&\frac{4nm\beta_k^2}{(1-\gamma)^2}.
	\end{align*}
	The result follows by taking expectation on both sides of the previous inequality.

	\item For any $k\geq 0$, we have
	\begin{align*}
		\langle q_{k+1}^i-q_k^i,\bar{q}_k^i-\bar{q}_{k+1}^i\rangle
		\leq \|q_{k+1}^i-q_k^i\|_2\|\bar{q}_k^i-\bar{q}_{k+1}^i\|_2
		\leq  \frac{4nm\alpha_k\beta_k}{(1-\gamma)^2},
	\end{align*}
	where the last inequality follows from Parts (1) and (2). The result follows by taking expectation on both sides.

	\item For any $k\geq 0$, we have
    \begin{align}
	    &\langle q_k^i-\bar{q}_k^i,\bar{q}_k^i-\bar{q}_{k+1}^i\rangle\nonumber\\
	    =\,&
        \beta_k\sum_{s}
        \langle q_k^i(s)-\bar{q}_k^i(s),
        \mathcal{T}^i(v^i)(s)(\sigma_\tau(q_k^{-i}(s))-\pi_k^{-i}(s))\rangle\nonumber\\
	    \leq \,&
        \frac{c_1\beta_k}{2}\| q_k^i-\bar{q}_k^i\|_2^2
        +
        \frac{\beta_k}{2c_1}
        \sum_{s}\|\mathcal{T}^i(v^i)(s)(\sigma_\tau(q_k^{-i}(s))-\pi_k^{-i}(s))\|_2^2,
        \label{eq:otherterms_2}
    \end{align}
    where $c_1>0$ is arbitrary. We next bound the second term on the right-hand side. For any $s\in\mathcal{S}$, we have
    \begin{align*}
	    &\|\mathcal{T}^i(v^i)(s)(\sigma_\tau(q_k^{-i}(s))-\pi_k^{-i}(s))\|_2\\
	    =\,&
        \|\mathcal{T}^i(v^i)(s)
        (\sigma_\tau(q_k^{-i}(s))-\sigma_\tau(\bar{q}_k^{-i}(s))
        +\sigma_\tau(\mathcal{T}^{-i}(v^{-i})(s)\pi_k^i(s))-\pi_k^{-i}(s))\|_2\\
	    \leq \,&
        \underbrace{\|\mathcal{T}^i(v^i)(s)(\sigma_\tau(q_k^{-i}(s))-\sigma_\tau(\bar{q}_k^{-i}(s)))\|_2}_{B_1}\\
	    &+
        \underbrace{\|\mathcal{T}^i(v^i)(s)(\sigma_\tau(\mathcal{T}^{-i}(v^{-i})(s)\pi_k^i(s))-\pi_k^{-i}(s))\|_2}_{B_2}.
    \end{align*}
    Since the softmax operator $\sigma_\tau(\cdot)$ is $\frac{1}{\tau}$ -- Lipschitz continuous with respect to $\|\cdot\|_2$ \cite[Proposition 4]{gao2017properties}, we have
    \begin{align*}
	    B_1
        \leq\,&
        \|\mathcal{T}^i(v^i)(s)\|_2
        \|\sigma_\tau(q_k^{-i}(s))-\sigma_\tau(\bar{q}_k^{-i}(s))\|_2\\
	    \leq\,&
        \frac{m}{\tau(1-\gamma)}
        \|q_k^{-i}(s)-\bar{q}_k^{-i}(s)\|_2.
    \end{align*}
    We next analyze $B_2$. Using Lemma \ref{le:properties_Lyapunov_main} (1) and the quadratic growth property of strongly convex functions, we have
    \begin{align*}
	    B_2
        =\,&
        \|\mathcal{T}^i(v^i)(s)(\sigma_\tau(\mathcal{T}^{-i}(v^{-i})(s)\pi_k^i(s))-\pi_k^{-i}(s))\|_2\\
	    \leq \,&
        \|\mathcal{T}^i(v^i)(s)\|_2
        \|\sigma_\tau(\mathcal{T}^{-i}(v^{-i})(s)\pi_k^i(s))-\pi_k^{-i}(s)\|_2\\
	    \leq \,&
        \frac{\sqrt{2}m}{\sqrt{\tau}(1-\gamma)}
        V_{v,s}(\pi_k^1(s),\pi_k^2(s))^{1/2}.
    \end{align*}
    Combining the upper bounds for $B_1$ and $B_2$, we obtain
    \begin{align*}
	    &\sum_{s}\|\mathcal{T}^i(v^i)(s)(\sigma_\tau(q_k^{-i}(s))-\pi_k^{-i}(s))\|_2^2\\
	    \leq \,&
        \sum_{s}(B_1+B_2)^2\\
	    \leq \,&
        2\sum_{s}(B_1^2+B_2^2)\\
	    \leq \,&
        2\sum_{s}
        \left(
        \frac{m^2}{\tau^2(1-\gamma)^2}
        \|q_k^{-i}(s)-\bar{q}_k^{-i}(s)\|_2^2
        +
        \frac{2m^2}{\tau(1-\gamma)^2}
        V_{v,s}(\pi_k^1(s),\pi_k^2(s))
        \right)\\
	    \leq \,&
        \frac{2m^2}{\tau^2(1-\gamma)^2}
        \|q_k^{-i}-\bar{q}_k^{-i}\|_2^2
        +
        \frac{4nm^2}{\tau(1-\gamma)^2}\mathcal{L}_\pi(k).
    \end{align*}
    Returning to \eqref{eq:otherterms_2}, we have
    \begin{align*}
	    &\langle q_k^i-\bar{q}_k^i,\bar{q}_k^i-\bar{q}_{k+1}^i\rangle\\
	    \leq\,&
        \frac{c_1\beta_k}{2}\| q_k^i-\bar{q}_k^i\|_2^2
        +
        \frac{m^2\beta_k}{c_1\tau^2(1-\gamma)^2}
        \|q_k^{-i}-\bar{q}_k^{-i}\|_2^2
        +
        \frac{2nm^2\beta_k}{c_1\tau(1-\gamma)^2}\mathcal{L}_\pi(k).
    \end{align*}
    Choosing $c_1=\frac{32nm^2}{\tau(1-\gamma)^2}$ and taking expectations on both sides gives
    \begin{align*}
	    &\mathbb{E}[\langle q_k^i-\bar{q}_k^i,\bar{q}_k^i-\bar{q}_{k+1}^i\rangle]\\
	    \leq\,&
        \frac{16nm^2\beta_k}{\tau(1-\gamma)^2}
        \mathbb{E}[\| q_k^i-\bar{q}_k^i\|_2^2]
        +
        \frac{\beta_k}{32n\tau}
        \mathbb{E}[\|q_k^{-i}-\bar{q}_k^{-i}\|_2^2]
        +
        \frac{\beta_k}{16}\mathbb{E}[\mathcal{L}_\pi(k)].
    \end{align*}
    In particular, after summing over $i=1,2$, we obtain
    \begin{align*}
	    \sum_{i=1,2}
        \mathbb{E}[\langle q_k^i-\bar{q}_k^i,\bar{q}_k^i-\bar{q}_{k+1}^i\rangle]
	    \leq\,&
        \frac{17nm^2\beta_k}{\tau(1-\gamma)^2}
        \sum_{i=1,2}\mathbb{E}[\| q_k^i-\bar{q}_k^i\|_2^2]
        +
        \frac{\beta_k}{8}\mathbb{E}[\mathcal{L}_\pi(k)].
    \end{align*}
\end{enumerate}
This completes the proof.
\hfill\qed
\endproof
\subsubsection{Proof of Lemma~\ref{le:noise}}\label{pf:le:noise}

For any $k\geq z_k$, we decompose
\begin{align}
	&\mathbb{E}[
    (F^i(q_k^i,S_k,A_k^i,A_k^{-i},S_{k+1})-\bar{F}_k^i(q_k^i))^\top
    (q_k^i-\bar{q}_k^i)]\nonumber\\
	= \,&
    \underbrace{\mathbb{E}[
    (F^i(q_{k-z_k}^i,S_k,A_k^i,A_k^{-i},S_{k+1})
    -\bar{F}_{k-z_k}^i(q_{k-z_k}^i))^\top
    (q_{k-z_k}^i-\bar{q}_{k-z_k}^i)]}_{N_1}\nonumber\\
	&+
    \underbrace{\mathbb{E}[
    (F^i(q_{k-z_k}^i,S_k,A_k^i,A_k^{-i},S_{k+1})
    -\bar{F}_{k-z_k}^i(q_{k-z_k}^i))^\top
    (q_k^i-q_{k-z_k}^i)]}_{N_2}\nonumber\\
	&+
    \underbrace{\mathbb{E}[
    (F^i(q_{k-z_k}^i,S_k,A_k^i,A_k^{-i},S_{k+1})
    -\bar{F}_{k-z_k}^i(q_{k-z_k}^i))^\top
    (\bar{q}_{k-z_k}^i-\bar{q}_k^i)]}_{N_3}\nonumber\\
	&+
    \underbrace{\mathbb{E}[
    (F^i(q_k^i,S_k,A_k^i,A_k^{-i},S_{k+1})
    -F^i(q_{k-z_k}^i,S_k,A_k^i,A_k^{-i},S_{k+1}))^\top
    (q_k^i-\bar{q}_k^i)]}_{N_4}\nonumber\\
	&+
    \underbrace{\mathbb{E}[
    (\bar{F}_{k-z_k}^i(q_{k-z_k}^i)-\bar{F}_k^i(q_k^i))^\top
    (q_k^i-\bar{q}_k^i)]}_{N_5}.
    \label{eq:N1_N5}
\end{align}
To bound $N_1$ to $N_5$, we first use the following lemma.

\begin{lemma}\label{le:difference}
	For any positive integers $k_1\leq k_2$, we have
    $\|q_{k_2}^i-q_{k_1}^i\|_\infty\leq  \frac{2\alpha_{k_1,k_2-1}}{1-\gamma}$ and
    $\max_{s\in\mathcal{S}}\|\pi_{k_2}^i(s)-\pi_{k_1}^i(s)\|_1\leq  2\beta_{k_1,k_2-1}$.
\end{lemma}

\proof{Proof of Lemma~\ref{le:difference}.}
For any $k\in [k_1,k_2-1]$, we have by \eqref{sa:reformulation_sketch} that
\begin{align*}
	\|q_{k+1}^i-q_k^i\|_\infty
	=
    \alpha_k \|F^i(q_k^i,S_k,A_k^i,A_k^{-i},S_{k+1})\|_\infty
	\leq
    \frac{2\alpha_k}{1-\gamma}.
\end{align*}
It follows that $\|q_{k_2}^i-q_{k_1}^i\|_\infty\leq \frac{2\alpha_{k_1,k_2-1}}{1-\gamma}$. Similarly, for any $k\in [k_1,k_2-1]$ and $s\in\mathcal{S}$,
\begin{align*}
	\|\pi_{k+1}^i(s)-\pi_k^i(s)\|_1
	=
    \beta_k\|\sigma_\tau(q_k^i(s))-\pi_k^i(s)\|_1
	\leq
    2\beta_k.
\end{align*}
Thus, $\max_{s\in\mathcal{S}}\|\pi_{k_2}^i(s)-\pi_{k_1}^i(s)\|_1\leq 2\beta_{k_1,k_2-1}$.
\hfill\qed
\endproof

Let $\mathcal{F}_k$ be the $\sigma$-algebra generated by
$\{S_0,A_0^i,A_0^{-i},\cdots,S_{k-1},A_{k-1}^i,A_{k-1}^{-i},S_k\}$.

\textbf{The Term $N_1$.}
Using the tower property of conditional expectations, we have
\begin{align}
	N_1
	=\,&
    \mathbb{E}\big[
    \big(\mathbb{E}[F^i(q_{k-z_k}^i,S_k,A_k^i,A_k^{-i},S_{k+1})\mid \mathcal{F}_{k-z_k}]
    -\bar{F}_{k-z_k}^i(q_{k-z_k}^i)\big)^\top
    (q_{k-z_k}^i-\bar{q}_{k-z_k}^i)
    \big]\nonumber\\
	\leq\,&
    \frac{2}{1-\gamma}
    \mathbb{E}\big[
    \|\mathbb{E}[F^i(q_{k-z_k}^i,S_k,A_k^i,A_k^{-i},S_{k+1})\mid \mathcal{F}_{k-z_k}]
    -\bar{F}_{k-z_k}^i(q_{k-z_k}^i)\|_1
    \big]\nonumber\\
    \leq\,&
    \frac{2}{1-\gamma}
    \mathbb{E}\big[
    \|\bar{F}_k^i(q_{k-z_k}^i)-\bar{F}_{k-z_k}^i(q_{k-z_k}^i)\|_1
    \big]\nonumber\\
    &+
    \frac{2}{1-\gamma}
    \mathbb{E}\big[
    \|\mathbb{E}[F^i(q_{k-z_k}^i,S_k,A_k^i,A_k^{-i},S_{k+1})\mid \mathcal{F}_{k-z_k}]
    -\bar{F}_k^i(q_{k-z_k}^i)\|_1
    \big].
    \label{eq:N1_decomposition}
\end{align}
The first term on the right-hand side of \eqref{eq:N1_decomposition} is bounded by
\begin{align}
    \|\bar{F}_k^i(q_{k-z_k}^i)-\bar{F}_{k-z_k}^i(q_{k-z_k}^i)\|_1
    \leq
    \frac{8nmL_p\beta_{k-z_k,k-1}}{1-\gamma}.
    \label{eq:1020}
\end{align}
Indeed, this follows from the total-variation characterization, Lemma~\ref{le:exploration} (3), and Lemma~\ref{le:difference}.

It remains to bound the second term on the right-hand side of \eqref{eq:N1_decomposition}. Since the sample at time $k$ is generated using $\pi_{k+1}$, while $\bar F_k^i$ is defined using the stationary distribution and action distribution associated with $\pi_k$, this term contains both a state-distribution mixing error and a one-step policy-shift error. Using the definition of conditional expectations, we have
\begin{align}
	&\|\mathbb{E}[F^i(q_{k-z_k}^i,S_k,A_k^i,A_k^{-i},S_{k+1})\mid \mathcal{F}_{k-z_k}]
    -\bar{F}_k^i(q_{k-z_k}^i)\|_1\nonumber\\
	\leq\,&
    \frac{2}{1-\gamma}
    \left\{
    \left\|\prod_{j=k+1}^{k+z_k}P_{\pi_{j-z_k}}-P_{\pi_k}^{z_k}\right\|_\infty
    +2\rho_*^{z_k}
    \right\}
    +
    \frac{4\beta_k}{1-\gamma}.
    \label{eq:probability}
\end{align}
The last term $\frac{4\beta_k}{1-\gamma}$ accounts for the one-step difference between sampling actions from $\pi_{k+1}$ and using the action distribution $\pi_k$ in $\bar F_k^i$.

Moreover,
\begin{align*}
	\left\|\prod_{j=k+1}^{k+z_k}P_{\pi_{j-z_k}}-P_{\pi_k}^{z_k}\right\|_\infty
	\leq
    \sum_{\ell=1}^{z_k}\|P_{\pi_{k-\ell+1}}-P_{\pi_k}\|_\infty
	\leq
    4z_k\beta_{k-z_k,k-1},
\end{align*}
where the last inequality follows from Lemma~\ref{le:difference}. Using the definition of $z_k$, we have $2\rho_*^{z_k}\leq \beta_k$. Hence, using $z_k\geq 1$ and $\beta_k\leq \beta_{k-z_k,k-1}$, we obtain
\begin{align*}
    &\|\mathbb{E}[F^i(q_{k-z_k}^i,S_k,A_k^i,A_k^{-i},S_{k+1})\mid \mathcal{F}_{k-z_k}]
    -\bar{F}_k^i(q_{k-z_k}^i)\|_1\\
    \leq\,&
    \frac{2}{1-\gamma}
    \left(4z_k\beta_{k-z_k,k-1}+\beta_k\right)
    +
    \frac{4\beta_k}{1-\gamma}\\
    \leq\,&
    \frac{14z_k\beta_{k-z_k,k-1}}{1-\gamma}.
\end{align*}
Using the previous inequality and \eqref{eq:1020} in \eqref{eq:N1_decomposition}, we obtain
\begin{align*}
    N_1
    \leq
    \frac{16L_pnm\beta_{k-z_k,k-1}}{(1-\gamma)^2}
    +
    \frac{28z_k\beta_{k-z_k,k-1}}{(1-\gamma)^2}
    \leq
    \frac{44L_pnmz_k\beta_{k-z_k,k-1}}{(1-\gamma)^2}.
\end{align*}

\textbf{The Term $N_2$.}
For any $k\geq z_k$, by Lemma~\ref{le:difference},
\begin{align}
	N_2
	\leq\,&
    \mathbb{E}[
    \|F^i(q_{k-z_k}^i,S_k,A_k^i,A_k^{-i},S_{k+1})
    -\bar{F}_{k-z_k}^i(q_{k-z_k}^i)\|_1
    \|q_k^i-q_{k-z_k}^i\|_\infty]\nonumber\\
    \leq\,&
    \frac{2\alpha_{k-z_k,k-1}}{1-\gamma}
    \mathbb{E}[
    \|F^i(q_{k-z_k}^i,S_k,A_k^i,A_k^{-i},S_{k+1})\|_1
    +\|\bar{F}_{k-z_k}^i(q_{k-z_k}^i)\|_1].
    \label{eq:N2_decomposition}
\end{align}
Using the definition of $F^i(\cdot)$ and Lemma~\ref{le:boundedness_proof_outline}, we have
\begin{align}
    \|F^i(q_{k-z_k}^i,S_k,A_k^i,A_k^{-i},S_{k+1})\|_1
    \leq
    \frac{2}{1-\gamma}.
    \label{eq123_N2}
\end{align}
Moreover, by Jensen's inequality,
\begin{align}
    \|\bar{F}_{k-z_k}^i(q_{k-z_k}^i)\|_1
    \leq
    \frac{2}{1-\gamma}.
    \label{eq124_N2}
\end{align}
Using \eqref{eq123_N2} and \eqref{eq124_N2} in \eqref{eq:N2_decomposition}, we have
\begin{align*}
    N_2\leq \frac{8\alpha_{k-z_k,k-1}}{(1-\gamma)^2}.
\end{align*}

\textbf{The Term $N_3$.}
For any $k\geq z_k$,
\begin{align*}
	N_3
    \leq\,&
    \mathbb{E}[
    \| F^i(q_{k-z_k}^i,S_k,A_k^i,A_k^{-i},S_{k+1})
    -\bar{F}_{k-z_k}^i(q_{k-z_k}^i)\|_1
    \|\bar{q}_{k-z_k}^i-\bar{q}_k^i\|_\infty]\\
    \leq\,&
    \frac{4}{1-\gamma}\mathbb{E}[\|\bar{q}_{k-z_k}^i-\bar{q}_k^i\|_\infty],
\end{align*}
where the last line follows from \eqref{eq123_N2} and \eqref{eq124_N2}. Observe that
\begin{align*}
	\|\bar{q}_{k-z_k}^i-\bar{q}_k^i\|_\infty
    =\,&
    \max_{s\in\mathcal{S}}
    \|\mathcal{T}^i(v^i)(s)(\pi_k^{-i}(s)-\pi_{k-z_k}^{-i}(s))\|_\infty\\
    \leq\,&
    \max_{s\in\mathcal{S}}\|\mathcal{T}^i(v^i)(s)\|_{1,\infty}
    \|\pi_k^{-i}(s)-\pi_{k-z_k}^{-i}(s)\|_1\\
	\leq\,&
    \frac{2\beta_{k-z_k,k-1}}{1-\gamma},
\end{align*}
where the last line follows from Lemma~\ref{le:difference} and
$\|\mathcal{T}^i(v^i)(s)\|_{1,\infty}\leq \frac{1}{1-\gamma}$. Therefore,
\begin{align*}
    N_3\leq \frac{8\beta_{k-z_k,k-1}}{(1-\gamma)^2}.
\end{align*}

\textbf{The Term $N_4$.}
For any $k\geq z_k$,
\begin{align*}
	N_4
	\leq \,&
    \mathbb{E}[
    \|F^i(q_k^i,S_k,A_k^i,A_k^{-i},S_{k+1})
    -F^i(q_{k-z_k}^i,S_k,A_k^i,A_k^{-i},S_{k+1})\|_1
    \|q_k^i-\bar{q}_k^i\|_\infty]\\
	\leq \,&
    \frac{2}{1-\gamma}
    \mathbb{E}[
    \|F^i(q_k^i,S_k,A_k^i,A_k^{-i},S_{k+1})
    -F^i(q_{k-z_k}^i,S_k,A_k^i,A_k^{-i},S_{k+1})\|_1]\\
    \leq\,&
    \frac{4\alpha_{k-z_k,k-1}}{(1-\gamma)^2},
\end{align*}
where the last inequality follows from Lemma~\ref{le:difference}.

\textbf{The Term $N_5$.}
For any $k\geq z_k$, we have
\begin{align}
	N_5
	\leq\,&
    \frac{2}{1-\gamma}
    \mathbb{E}[\| \bar{F}_k^i(q_k^i)-\bar{F}_{k-z_k}^i(q_{k-z_k}^i)\|_1]\nonumber\\
    \leq\,&
    \frac{2}{1-\gamma}
    \mathbb{E}[
    \| \bar{F}_k^i(q_k^i)-\bar{F}_{k-z_k}^i(q_k^i)\|_1
    +
    \|\bar{F}_{k-z_k}^i(q_k^i)-\bar{F}_{k-z_k}^i(q_{k-z_k}^i)\|_1]\nonumber\\
    \leq\,&
    \frac{16L_pnm\beta_{k-z_k,k-1}}{(1-\gamma)^2}
    +
    \frac{2}{1-\gamma}
    \mathbb{E}[
    \|\bar{F}_{k-z_k}^i(q_k^i)-\bar{F}_{k-z_k}^i(q_{k-z_k}^i)\|_1],
    \label{eq:N5_decomposition}
\end{align}
where the last line follows from the same analysis used to obtain \eqref{eq:1020}. For the second term on the right-hand side of \eqref{eq:N5_decomposition}, using the definition of $\bar{F}_{k-z_k}^i(\cdot)$, we have
\begin{align*}
    \|\bar{F}_{k-z_k}^i(q_k^i)-\bar{F}_{k-z_k}^i(q_{k-z_k}^i)\|_1
    \leq
    \|q_k^i-q_{k-z_k}^i\|_\infty
    \leq
    \frac{2\alpha_{k-z_k,k-1}}{1-\gamma}.
\end{align*}
Using this in \eqref{eq:N5_decomposition}, we obtain
\begin{align*}
    N_5
    \leq
    \frac{16L_pnm\beta_{k-z_k,k-1}}{(1-\gamma)^2}
    +
    \frac{4\alpha_{k-z_k,k-1}}{(1-\gamma)^2}.
\end{align*}

Combining the upper bounds for $N_1$ to $N_5$ in \eqref{eq:N1_N5}, we obtain
\begin{align*}
    &\mathbb{E}[
    (F^i(q_k^i,S_k,A_k^i,A_k^{-i},S_{k+1})-\bar{F}_k^i(q_k^i))^\top
    (q_k^i-\bar{q}_k^i)]\\
    \leq\,&
    \frac{44L_pnmz_k\beta_{k-z_k,k-1}}{(1-\gamma)^2}
    +\frac{8\alpha_{k-z_k,k-1}}{(1-\gamma)^2}
    +\frac{8\beta_{k-z_k,k-1}}{(1-\gamma)^2}\\
    &+
    \frac{4\alpha_{k-z_k,k-1}}{(1-\gamma)^2}
    +\frac{16L_pnm\beta_{k-z_k,k-1}}{(1-\gamma)^2}
    +\frac{4\alpha_{k-z_k,k-1}}{(1-\gamma)^2}\\
    \leq\,&
    \frac{76L_pnmz_k\beta_{k-z_k,k-1}}{(1-\gamma)^2}
    +\frac{16\alpha_{k-z_k,k-1}}{(1-\gamma)^2}.
\end{align*}
Since $c_{\alpha,\beta}=\beta_k/\alpha_k\leq 1/(76L_pnm)$ (cf. Condition \ref{con:stepsize_stochastic_game}), we have $\beta_{k-z_k,k-1}\leq c_{\alpha,\beta}\alpha_{k-z_k,k-1}$, which implies
\begin{align*}
    \mathbb{E}[
    (F^i(q_k^i,S_k,A_k^i,A_k^{-i},S_{k+1})-\bar{F}_k^i(q_k^i))^\top
    (q_k^i-\bar{q}_k^i)]
    \leq
    \frac{17z_k\alpha_{k-z_k,k-1}}{(1-\gamma)^2}.
\end{align*}
This completes the proof.
\hfill\qed
\endproof

\section{Proof of Theorem \ref{thm:stochastic_game_slow}}\label{pf:thm:stochastic_game_slow}

The high-level idea of proving Theorem~\ref{thm:stochastic_game_slow} is similar to that of Theorem~\ref{thm:stochastic_game}. However, due to the algorithmic modification, namely, replacing $\sigma_\tau(\cdot)$ with $\sm(\cdot)$, the negative drift inequalities for the policies and the $q$-functions are substantially different from those in the proof of Theorem~\ref{thm:stochastic_game}.

\subsection{Boundedness of the Iterates}

\begin{lemma}\label{le:boundedness_stochastic_slow}
	For all $t,k$ and $i\in \{1,2\}$, we have $\|v_t^i\|_\infty\leq 1/(1-\gamma)$, $\|q_{t,k}^i\|_\infty\leq 1/(1-\gamma)$, and $\min_{s\in\mathcal{S},a^i\in\mathcal{A}^i}\pi_{t,k}^i(a^i\mid s)\geq \Bar{\epsilon}/m$.
\end{lemma}

The proof of Lemma~\ref{le:boundedness_stochastic_slow} is identical to that of Lemma~\ref{le:boundedness_proof_outline}, and therefore is omitted.

\subsection{Bounding the Nash Gap}

We restate Lemma~\ref{le:Nash_Combine} to make the proof self-contained.

\begin{lemma}\label{le:Nash_Combine_slow}
    It holds that
    \begin{align*}
        \text{NG}(\pi_{T,K}^1,\pi_{T,K}^2)
        \leq \,&
        \frac{4}{1-\gamma}
        \left(
        2\mathcal{L}_{\text{sum}}(T)
        +\mathcal{L}_v(T)
        +\mathcal{L}_\pi(T,K)
        +2\tau \log(m)
        \right).
    \end{align*}
\end{lemma}

\subsection{Analysis of the Outer Loop}

We restate Lemma~\ref{le:outer-loop} and Lemma~\ref{le:outer-sum} to make the proof self-contained.

\begin{lemma}\label{le:outer-loop-slow}
	The following inequality holds for all $t\geq 0$:
	\begin{align*}
        \mathcal{L}_v(t+1)
        \leq\,&
        \gamma \mathcal{L}_v(t)
        +4\mathcal{L}_{\text{sum}}(t)
        +2\mathcal{L}_q^{1/2}(t,K)
        +4\mathcal{L}_\pi(t,K)
        +6\tau \log(m).
	\end{align*}
\end{lemma}

\begin{lemma}\label{le:outer-sum-slow}
	The following inequality holds for all $t\geq 0$:
    $\mathcal{L}_{\text{sum}}(t+1)\leq \gamma\mathcal{L}_{\text{sum}}(t)+2\mathcal{L}_q^{1/2}(t,K)$.
\end{lemma}

\subsection{Analysis of the Inner Loop}

For ease of presentation, we write down only the inner loop of Algorithm~\ref{algo:stochastic_game_slow} in the following, where we omit the subscript $t$. Similarly, we write $\mathcal{L}_q(k)$ for $\mathcal{L}_q(t,k)$ and $\mathcal{L}_\pi(k)$ for $\mathcal{L}_\pi(t,k)$.

\begin{algorithm}[ht]\caption{Inner Loop of Algorithm \ref{algo:stochastic_game_slow}}\label{algorithm:inner-loop-slow}
	\begin{algorithmic}[1]
		\STATE \textbf{Input:} Integer $K$, initializations $q_0^i$ and $\pi_0^i$, and a joint $v$-function $v^i$ from the outer loop. Note that we have $\|q_0^i\|_\infty\leq \frac{1}{1-\gamma}$, $\|v^i\|_\infty\leq \frac{1}{1-\gamma}$, and $\min_{s,a^i}\pi_0^i(a^i\mid s)\geq \frac{\Bar{\epsilon}}{m}$ due to Lemma \ref{le:boundedness_stochastic_slow}. 
		\FOR{$k=0,1,\cdots,K-1$}
		\STATE $\pi_{k+1}^i(s)=\pi_k^i(s)+\beta_k(\sm(q_k^i(s))-\pi_k^i(s))$ for all $s\in\mathcal{S}$
		\STATE Sample
		$A_k^i\sim \pi_{k+1}^i(\cdot\mid S_k)$, and observe $S_{k+1}\sim p(\cdot\mid S_k,A_k^i,A_k^{-i})$
		\STATE $q_{k+1}^i(S_k,A_k^i)=q_k^i(S_k,A_k^i)+\alpha_k \left(\mathcal{R}_i(S_k,A_k^i,A_k^{-i})+\gamma v^i(S_{k+1})-q_k^i(S_k,A_k^i)\right)$
		\ENDFOR
	\end{algorithmic}
\end{algorithm}

\subsubsection{Analysis of the Policies}

\begin{lemma}\label{le:policy_drift_slow}
When choosing $\Bar{\epsilon}=\tau\leq 1$, we have for all $k\geq 0$ that
\begin{align*}
	\mathbb{E}[\mathcal{L}_\pi(k+1)]
	\leq\,&
    (1-\beta_k)\mathbb{E}[\mathcal{L}_\pi(k)]
    +4\beta_k\|v^1+v^2\|_\infty
    +\frac{8m\beta_k}{(1-\gamma)\tau}\mathbb{E}[\mathcal{L}_q(k)^{1/2}]\\
    &+\frac{12m\tau\beta_k}{1-\gamma}
    +2\Bar{L}_\tau\beta_k^2,
\end{align*}
where $\Bar{L}_\tau=\frac{5m^2}{\tau(1-\gamma)^2}$.
\end{lemma}

\proof{Proof of Lemma~\ref{le:policy_drift_slow}.}
We will use $V_{v,s}(\cdot,\cdot)$ as a Lyapunov function to study the evolution of $(\pi_k^1(s),\pi_k^2(s))$. A sequence of properties regarding $V_X(\cdot,\cdot)$ for arbitrary matrices $X_1,X_2$ is presented in Lemma~\ref{le:properties_Lyapunov_main_slow}. To begin with, we identify the smoothness parameter of $V_{v,s}(\cdot,\cdot)$. Using Lemma~\ref{le:properties_Lyapunov_main_slow} (1) and the definition of $V_{v,s}(\cdot,\cdot)$, we have
\begin{align*}
	L_{\tau,\Bar{\epsilon}}
    =\,&
    2\left(
    \frac{m \tau}{\Bar{\epsilon}}
    +\frac{\max(\|X_1\|_2^2,\|X_2\|_2^2)}{\tau}
    +\|X_1+X_2^\top\|_2
    \right)\\
	=\,&
    2\left(
    m
    +\frac{\max(\|\mathcal{T}^1(v^1)(s)\|_2^2,\|\mathcal{T}^2(v^2)(s)\|_2^2)}{\tau}
    +\|\mathcal{T}^1(v^1)(s)+\mathcal{T}^2(v^2)(s)^\top\|_2
    \right)
    \tag{$\Bar{\epsilon}=\tau$}\\
	\leq \,&
    2\left(
    m
    +\frac{m^2}{\tau(1-\gamma)^2}
    +\frac{2m}{1-\gamma}
    \right)\\
	\leq \,&
    \frac{5m^2}{\tau(1-\gamma)^2}
    =\Bar{L}_{\tau},
\end{align*}
where the first inequality follows from $|\mathcal{T}^i(v^i)(s,a^i,a^{-i})|\leq \frac{1}{1-\gamma}$ for all $(s,a^i,a^{-i})$ and $i\in \{1,2\}$, and the last inequality follows from $m\geq 2$ and $\tau\leq 1/(1-\gamma)$. Therefore, $V_{v,s}(\cdot,\cdot)$ is a $\Bar{L}_\tau$ -- smooth function on $\Pi_{\Bar{\epsilon}}$, where $\Pi_{\Bar{\epsilon}}=\{(\mu^1,\mu^2)\mid \min_{a^1}\mu^1(a^1)\geq \Bar{\epsilon}/m,\min_{a^2}\mu^2(a^2)\geq \Bar{\epsilon}/m\}$.

Using the smoothness of $V_{v,s}(\cdot,\cdot)$, for any $s\in\mathcal{S}$, we have by the policy update equation in Algorithm~\ref{algorithm:inner-loop-slow}, Line~3, that
\begin{align*}
	&V_{v,s}(\pi_{k+1}^1(s),\pi_{k+1}^2(s))\\
    \leq\,&
    V_{v,s}(\pi_k^1(s),\pi_k^2(s))
    +\beta_k\langle \nabla_2V_{v,s}(\pi_k^1(s),\pi_k^2(s)),\sm(q_k^2(s))-\pi_k^2(s) \rangle\\
    &+\beta_k\langle \nabla_1V_{v,s}(\pi_k^1(s),\pi_k^2(s)),\sm(q_k^1(s))-\pi_k^1(s) \rangle
    +\frac{\Bar{L}_\tau\beta_k^2}{2}
    \sum_{i=1,2}\|\sm(q_k^i(s))-\pi_k^i(s)\|_2^2\\
    \leq\,&
    V_{v,s}(\pi_k^1(s),\pi_k^2(s))
    +\beta_k\langle \nabla_2V_{v,s}(\pi_k^1(s),\pi_k^2(s)),
    \sigma_\tau(\mathcal{T}^2(v^2)(s)\pi_k^1(s))-\pi_k^2(s) \rangle\\
    &+\beta_k\langle \nabla_1 V_{v,s}(\pi_k^1(s),\pi_k^2(s)),
    \sigma_\tau(\mathcal{T}^1(v^1)(s)\pi_k^2(s))-\pi_k^1(s) \rangle\\
    &+\beta_k\langle \nabla_2V_{v,s}(\pi_k^1(s),\pi_k^2(s)),
    \sm(q_k^2(s))-\sigma_\tau(\mathcal{T}^2(v^2)(s)\pi_k^1(s)) \rangle\\
    &+\beta_k\langle \nabla_1 V_{v,s}(\pi_k^1(s),\pi_k^2(s)),
    \sm(q_k^1(s))-\sigma_\tau(\mathcal{T}^1(v^1)(s)\pi_k^2(s)) \rangle
    +2\Bar{L}_\tau\beta_k^2\\
    \leq\,&
    (1-\beta_k)V_{v,s}(\pi_k^1(s),\pi_k^2(s))
    +4\beta_k\max_{s,a^1,a^2}\left|
    \sum_{i=1,2}\mathcal{T}^i(v^i)(s,a^i,a^{-i})
    \right|\\
    &+2\beta_k\left(
    m+
    \frac{\max_{i\in\{1,2\}}\|\mathcal{T}^i(v^i)(s)\|_2}{\tau}
    \right)
    \sum_{i=1,2}\|q_k^i(s)-\mathcal{T}^i(v^i)(s)\pi_k^{-i}(s)\|_2\\
    &+4\beta_k\Bar{\epsilon}\left(
    m+\sum_{i=1,2}\|\mathcal{T}^i(v^i)(s)\|_2
    \right)
    +2\Bar{L}_\tau\beta_k^2,
\end{align*}
where the last line follows from Lemma~\ref{le:properties_Lyapunov_main_slow} (2) and (3), and $\Bar{\epsilon}=\tau$.

Since
\begin{align*}
    \max_{s,a^i,a^{-i}}\left|
    \sum_{i=1,2}\mathcal{T}^i(v^i)(s,a^i,a^{-i})
    \right|
    \leq
    \|v^1+v^2\|_\infty,\quad
    \|\mathcal{T}^i(v^i)(s)\|_2\leq \frac{m}{1-\gamma},
\end{align*}
and $\tau\leq 1/(1-\gamma)$, we obtain
\begin{align*}
    \max_{s\in\mathcal{S}}V_{v,s}(\pi_{k+1}^1(s),\pi_{k+1}^2(s))
    \leq\,&
    (1-\beta_k)\max_{s\in\mathcal{S}}V_{v,s}(\pi_k^1(s),\pi_k^2(s))
    +4\beta_k\|v^1+v^2\|_\infty\\
    &+\frac{4m\beta_k}{\tau(1-\gamma)}
    \sum_{i=1,2}\|q_k^i-\Bar{q}_k^i\|_2
    +2\Bar{L}_\tau\beta_k^2
    +\frac{12m\tau \beta_k}{1-\gamma}.
\end{align*}
Taking expectation on both sides and using the pathwise inequality
\begin{align*}
    \sum_{i=1,2}\|q_k^i-\Bar{q}_k^i\|_2
    \leq
    2\left(\sum_{i=1,2}\|q_k^i-\Bar{q}_k^i\|_2^2\right)^{1/2}
    =
    2\mathcal{L}_q(k)^{1/2},
\end{align*}
gives the desired result.
\hfill\qed
\endproof

\subsubsection{Analysis of the $q$-Functions}

For $i\in \{1,2\}$, fixing $v^i\in\mathbb{R}^{n}$, let $F^i:\mathbb{R}^{n m_i}\times \mathcal{S}\times \mathcal{A}^i\times \mathcal{A}^{-i}\times \mathcal{S}\to \mathbb{R}^{n m_i}$ be defined as
\begin{align*}
	[F^i(q^i,s_0,a_0^i,a_0^{-i},s_1)](s,a^i)
	=
    \mathds{1}_{\{(s,a^i)=(s_0,a_0^i)\}}
    \left(R_i(s_0,a_0^i,a_0^{-i})+\gamma v^i(s_1)-q^i(s_0,a_0^i)\right)
\end{align*}
for all $(q^i,s_0,a_0^i,a_0^{-i},s_1)$ and $(s,a^i)$. Then, Line~5 of Algorithm~\ref{algorithm:inner-loop-slow} can be written as
\begin{align}\label{sa:reformulation_slow}
	q_{k+1}^i=q_k^i+\alpha_k F^i(q_k^i,S_k,A_k^i,A_k^{-i},S_{k+1}).
\end{align}
Denote the stationary distribution of the Markov chain $\{S_k\}$ induced by the joint policy $\pi_k=(\pi_k^1,\pi_k^2)$ by $\mu_k\in\Delta(\mathcal{S})$, the existence and uniqueness of which are guaranteed by Lemma~\ref{le:boundedness_stochastic_slow} and Lemma~\ref{le:exploration} (1). Let $\Bar{F}_k^i:\mathbb{R}^{n m_i}\to \mathbb{R}^{n m_i}$ be defined as
\begin{align*}
	\Bar{F}_k^i(q^i)
	=
    \mathbb{E}_{S_0\sim \mu_k(\cdot),A_0^i\sim \pi_k^i(\cdot|S_0), A_0^{-i}\sim \pi_k^{-i}(\cdot|S_0), S_1\sim p(\cdot|S_0,A_0^i,A_0^{-i})}
    \left[F^i(q^i,S_0,A_0^i,A_0^{-i},S_1)\right]
\end{align*}
for all $q^i\in\mathbb{R}^{nm_i}$.

\begin{lemma}\label{le:q-function-drift-slow}
	The following inequality holds for all $k\geq z_k$:
	\begin{align*}
    \mathbb{E}[\mathcal{L}_q(k+1)]
    \leq
    \left(1-\frac{3\mu_{\min}\tau\alpha_k}{2m}\right)
    \mathbb{E}[\mathcal{L}_q(k)]
    +\frac{32nm^2\beta_k^2}{\mu_{\min}\tau (1-\gamma)^2\alpha_k}+\frac{84z_k\alpha_k\alpha_{k-z_k,k-1}}{(1-\gamma)^2}.
\end{align*}
\end{lemma}

\proof{Proof of Lemma~\ref{le:q-function-drift-slow}.}
Using the binomial theorem and the equivalent formulation of the update equation in \eqref{sa:reformulation_slow}, we have for all $k\geq 0$ and $i\in \{1,2\}$ that
\begin{align}\label{eq:q_Lyapunov_decomposition_slow}
	\mathbb{E}[\|q_{k+1}^i-\bar{q}_{k+1}^i\|_2^2]
	=\,&
    \mathbb{E}[\|q_{k+1}^i-q_k^i+q_k^i-\bar{q}_k^i+\bar{q}_k^i-\bar{q}_{k+1}^i\|_2^2]\nonumber\\
	=\,&
    \mathbb{E}[\|q_k^i-\bar{q}_k^i\|_2^2]
    +\alpha_k^2\mathbb{E}[\|F^i(q_k^i,S_k,A_k^i,A_k^{-i},S_{k+1})\|_2^2]\nonumber\\
    &+\mathbb{E}[\|\bar{q}_k^i-\bar{q}_{k+1}^i\|_2^2]
    +2\alpha_k\mathbb{E}[(q_k^i-\bar{q}_k^i)^\top \bar{F}_k^i(q_k^i)]\nonumber\\
	&+2\alpha_k\mathbb{E}[(F^i(q_k^i,S_k,A_k^i,A_k^{-i},S_{k+1})-\bar{F}_k^i(q_k^i))^\top (q_k^i-\bar{q}_k^i)]\nonumber\\
    &+2\alpha_k\mathbb{E}[(\bar{q}_k^i-\bar{q}_{k+1}^i)^\top F^i(q_k^i,S_k,A_k^i,A_k^{-i},S_{k+1})]\nonumber\\
    &+2\mathbb{E}[(q_k^i-\bar{q}_k^i)^\top (\bar{q}_k^i-\bar{q}_{k+1}^i)]\nonumber\\
    \leq \,&
    \left(1-\frac{3\mu_{\min}\Bar{\epsilon}\alpha_k}{2m}\right)
    \mathbb{E}[\|q_k^i-\bar{q}_k^i\|_2^2]
    +2\alpha_k^2\mathbb{E}[\|F^i(q_k^i,S_k,A_k^i,A_k^{-i},S_{k+1})\|_2^2]\nonumber\\
    &+\frac{4m\mathbb{E}[\|\bar{q}_k^i-\bar{q}_{k+1}^i\|_2^2]}{\mu_{\min}\Bar{\epsilon} \alpha_k}\nonumber\\
    &+2\alpha_k\mathbb{E}[(F^i(q_k^i,S_k,A_k^i,A_k^{-i},S_{k+1})-\bar{F}_k^i(q_k^i))^\top (q_k^i-\bar{q}_k^i)],
\end{align}
where the last line follows from Lemma~\ref{le:operators}, the Cauchy--Schwarz inequality, and $a^2/c+b^2c\geq 2ab$ for any $a,b\in\mathbb{R}$ and $c>0$.

Next, we bound the terms $\mathbb{E}[\|F^i(q_k^i,S_k,A_k^i,A_k^{-i},S_{k+1})\|_2^2]$, $\mathbb{E}[\|\bar{q}_k^i-\bar{q}_{k+1}^i\|_2^2]$, and $\mathbb{E}[(F^i(q_k^i,S_k,A_k^i,A_k^{-i},S_{k+1})-\bar{F}_k^i(q_k^i))^\top (q_k^i-\bar{q}_k^i)]$ in the following lemma, which is a restatement of Lemma~\ref{le:other_terms} and Lemma~\ref{le:noise}.

\begin{lemma}\label{le:first_3}
    The following inequalities hold:
    \begin{enumerate}[(1)]
        \item $\|F^i(q_k^i,S_k,A_k^i,A_k^{-i},S_{k+1})\|_2^2\leq \frac{4}{(1-\gamma)^2}$ for all $k\geq 0$.
        \item $\|\bar{q}_k^i-\bar{q}_{k+1}^i\|_2^2\leq \frac{4nm\beta_k^2}{(1-\gamma)^2}$ for all $k\geq 0$.
        \item $\mathbb{E}[(F^i(q_k^i,S_k,A_k^i,A_k^{-i},S_{k+1})-\bar{F}_k^i(q_k^i))^\top (q_k^i-\bar{q}_k^i)]\leq \frac{17z_k\alpha_{k-z_k,k-1}}{(1-\gamma)^2}$ for all $k\geq z_k$.
    \end{enumerate}
\end{lemma}

Using Lemma~\ref{le:first_3} in \eqref{eq:q_Lyapunov_decomposition_slow}, and recalling that $\Bar{\epsilon}=\tau$, we have
\begin{align*}
	\mathbb{E}[\|q_{k+1}^i-\bar{q}_{k+1}^i\|_2^2]
    \leq\,&
    \left(1-\frac{3\mu_{\min}\tau\alpha_k}{2m}\right)
    \mathbb{E}[\|q_k^i-\bar{q}_k^i\|_2^2]
    +\frac{8\alpha_k^2}{(1-\gamma)^2}\\
    &+\frac{16nm^2\beta_k^2}{\mu_{\min}\tau(1-\gamma)^2\alpha_k}
    +\frac{34z_k\alpha_k\alpha_{k-z_k,k-1}}{(1-\gamma)^2}\\
    \leq\,&
    \left(1-\frac{3\mu_{\min}\tau\alpha_k}{2m}\right)
    \mathbb{E}[\|q_k^i-\bar{q}_k^i\|_2^2]
    +\frac{16nm^2\beta_k^2}{\mu_{\min}\tau(1-\gamma)^2\alpha_k}\\
    &+\frac{42z_k\alpha_k\alpha_{k-z_k,k-1}}{(1-\gamma)^2},
\end{align*}
where the last inequality follows from $z_k\geq 1$ and $\alpha_{k-z_k,k-1}\geq \alpha_k$, which imply $\alpha_k^2\leq z_k\alpha_k\alpha_{k-z_k,k-1}$. Summing the previous inequality over $i\in\{1,2\}$ gives
\begin{align*}
    \mathbb{E}[\mathcal{L}_q(k+1)]
    \leq
    \left(1-\frac{3\mu_{\min}\tau\alpha_k}{2m}\right)
    \mathbb{E}[\mathcal{L}_q(k)]
    +\frac{32nm^2\beta_k^2}{\mu_{\min}\tau(1-\gamma)^2\alpha_k}+\frac{84z_k\alpha_k\alpha_{k-z_k,k-1}}{(1-\gamma)^2}.
\end{align*}
This completes the proof.
\hfill\qed
\endproof

\subsection{Solving Coupled Lyapunov Drift Inequalities}

We first restate the Lyapunov drift inequalities for the value functions, the $q$-functions, and the policies.

\begin{itemize}
	\item \textbf{Lemma~\ref{le:outer-loop-slow}:} It holds for all $t\geq 0$ that
	\begin{align}\label{eq:Lyapunov_v_slow}
		\mathcal{L}_v(t+1)
        \leq \,&
        \gamma \mathcal{L}_v(t)
        +4\mathcal{L}_{\text{sum}}(t)
        +2\mathcal{L}_q^{1/2}(t,K)
        +4\mathcal{L}_\pi(t,K)
        +6\tau \log(m).
	\end{align}

	\item \textbf{Lemma~\ref{le:outer-sum-slow}:} It holds for all $t\geq 0$ that
	\begin{align}\label{eq:Lyapunov_v+_slow}
		\mathcal{L}_{\text{sum}}(t+1)
        \leq
        \gamma\mathcal{L}_{\text{sum}}(t)
        +2\mathcal{L}_q^{1/2}(t,K).
	\end{align}

	\item \textbf{Lemma~\ref{le:policy_drift_slow}:} It holds for all $t,k\geq 0$ that
	\begin{align}\label{eq:Lyapunov_pi_slow}
	\mathbb{E}_t[\mathcal{L}_\pi(t,k+1)]
    \leq\,&
    (1-\beta)\mathbb{E}_t[\mathcal{L}_\pi(t,k)]
    +4\beta\mathcal{L}_{\text{sum}}(t)
    +c_\pi\beta\mathbb{E}_t[\mathcal{L}_q^{1/2}(t,k)]\nonumber\\
    &+\frac{12m\tau\beta}{1-\gamma}
    +2\Bar{L}_\tau \beta^2,
	\end{align}
	where $c_\pi=\frac{8m}{(1-\gamma)\tau}$ and $\mathbb{E}_t[\cdot]$ denotes the conditional expectation given the history up to the beginning of the $t$-th outer loop.

	\item \textbf{Lemma~\ref{le:q-function-drift-slow}:} It holds for all $t\geq 0$ and $k\geq z_\beta$ that
    \begin{align}\label{eq:Lyapunov_q_slow}
    \mathbb{E}_t[\mathcal{L}_q(t,k+1)]
    \leq
    \left(1-c_{q,1}\alpha\right)\mathbb{E}_t[\mathcal{L}_q(t,k)]
    +c_{q,2}\frac{\beta^2}{\alpha}
    +c_{q,3}z_\beta^2\alpha^2,
    \end{align}
	where $c_{q,1}=\frac{3\mu_{\min}\tau}{2m}$, $c_{q,2}=\frac{32nm^2}{\mu_{\min}\tau (1-\gamma)^2}$, and $c_{q,3}=\frac{84}{(1-\gamma)^2}$.
\end{itemize}

Iterating \eqref{eq:Lyapunov_q_slow}, we have for all $k\geq z_\beta$ that
\begin{align}
	\mathbb{E}_t[\mathcal{L}_q(t,k)]
    \leq\,&
    (1-c_{q,1}\alpha)^k\mathcal{L}_{q,\max}
    +\frac{c_{q,2}\beta^2}{c_{q,1}\alpha^2}
    +\frac{c_{q,3}z_\beta^2\alpha}{c_{q,1}},
    \label{solving_recursion_q_slow}
\end{align}
where $\mathcal{L}_{q,\max}=\frac{8nm}{(1-\gamma)^2}$. Using the previous inequality in \eqref{eq:Lyapunov_pi_slow}, and using $\beta/\alpha\leq c_{q,1}/2$, which implies $(1-c_{q,1}\alpha)^{1/2}\leq 1-\beta$, we have
\begin{align*}
    \mathbb{E}_t[\mathcal{L}_\pi(t,k+1)]
    \leq\,&
    (1-\beta)\mathbb{E}_t[\mathcal{L}_\pi(t,k)]
    +4\beta\mathcal{L}_{\text{sum}}(t)
    +\frac{12m\tau\beta}{1-\gamma}
    +2\Bar{L}_\tau \beta^2\\
    &+c_\pi\beta
    \left[
    (1-\beta)^k\mathcal{L}_{q,\max}^{1/2}
    +\frac{c_{q,2}^{1/2}\beta}{c_{q,1}^{1/2}\alpha}
    +\frac{c_{q,3}^{1/2}z_\beta\alpha^{1/2}}{c_{q,1}^{1/2}}
    \right].
\end{align*}
Iterating the previous inequality gives
\begin{align}
    \mathbb{E}_t[\mathcal{L}_\pi(t,k)]
    \leq\,&
    (1-\beta)^k(\mathcal{L}_{\pi,\max}+c_\pi\beta k\mathcal{L}^{1/2}_{q,\max})
    +4\mathcal{L}_{\text{sum}}(t)
    +\frac{c_\pi c_{q,2}^{1/2}\beta}{c_{q,1}^{1/2}\alpha}
    +\frac{c_\pi c_{q,3}^{1/2}z_\beta\alpha^{1/2}}{c_{q,1}^{1/2}}\nonumber\\
    &+\frac{12m\tau}{1-\gamma}
    +2\Bar{L}_\tau \beta.
    \label{recursion_pi_before}
\end{align}
Similarly, using \eqref{solving_recursion_q_slow} in \eqref{eq:Lyapunov_v+_slow}, and using $(1-c_{q,1}\alpha)^{1/2}\leq 1-\beta$, we have
\begin{align*}
    \mathbb{E}_t[\mathcal{L}_{\text{sum}}(t+1)]
    \leq\,&
    \gamma\mathcal{L}_{\text{sum}}(t)
    +2(1-\beta)^K\mathcal{L}_{q,\max}^{1/2}
    +\frac{2c_{q,2}^{1/2}\beta}{c_{q,1}^{1/2}\alpha}
    +\frac{2c_{q,3}^{1/2}z_\beta\alpha^{1/2}}{c_{q,1}^{1/2}}.
\end{align*}
Taking total expectation on both sides and iterating, since $\mathcal{L}_{\text{sum}}(0)\leq 2/(1-\gamma)$ by Lemma~\ref{le:boundedness_stochastic_slow}, we have
\begin{align}
    \mathbb{E}[\mathcal{L}_{\text{sum}}(t)]
    \leq\,&
    \frac{2\gamma^t}{1-\gamma}
    +\frac{2\mathcal{L}_{q,\max}^{1/2}}{1-\gamma}(1-\beta)^K
    +\frac{2c_{q,2}^{1/2}\beta}{c_{q,1}^{1/2}(1-\gamma)\alpha}
    +\frac{2c_{q,3}^{1/2}z_\beta\alpha^{1/2}}{c_{q,1}^{1/2}(1-\gamma)}.
    \label{solving_recursion_vsum}
\end{align}

Using \eqref{solving_recursion_vsum}, \eqref{recursion_pi_before}, and \eqref{solving_recursion_q_slow} in \eqref{eq:Lyapunov_v_slow}, we obtain
\begin{align*}
    \mathbb{E}[\mathcal{L}_v(t+1)]
    \leq\,&
    \gamma \mathbb{E}[\mathcal{L}_v(t)]
    +4\mathbb{E}[\mathcal{L}_{\text{sum}}(t)]
    +2\mathbb{E}[\mathcal{L}_q^{1/2}(t,K)]
    +4\mathbb{E}[\mathcal{L}_\pi(t,K)]
    +6\tau \log(m)\\
    \leq\,&
    \gamma \mathbb{E}[\mathcal{L}_v(t)]
    +\frac{40\gamma^t}{1-\gamma}
    +\frac{40\mathcal{L}_{q,\max}^{1/2}}{1-\gamma}(1-\beta)^K
    +\frac{40c_{q,2}^{1/2}\beta}{c_{q,1}^{1/2}(1-\gamma)\alpha}\\
    &+\frac{40c_{q,3}^{1/2}z_\beta\alpha^{1/2}}{c_{q,1}^{1/2}(1-\gamma)}
    +4(1-\beta)^K(\mathcal{L}_{\pi,\max}+c_\pi\beta K\mathcal{L}^{1/2}_{q,\max})\\
    &+\frac{4c_\pi c_{q,2}^{1/2}\beta}{c_{q,1}^{1/2}\alpha}
    +\frac{4c_\pi c_{q,3}^{1/2}z_\beta\alpha^{1/2}}{c_{q,1}^{1/2}}
    +\frac{48m\tau}{1-\gamma}
    +8\Bar{L}_\tau \beta\\
    &+2(1-\beta)^K\mathcal{L}_{q,\max}^{1/2}
    +\frac{2c_{q,2}^{1/2}\beta}{c_{q,1}^{1/2}\alpha}
    +\frac{2c_{q,3}^{1/2}z_\beta\alpha^{1/2}}{c_{q,1}^{1/2}}
    +6\tau \log(m)\\
    \leq\,&
    \gamma \mathbb{E}[\mathcal{L}_v(t)]
    +\frac{40\gamma^t}{1-\gamma}
    +\frac{7 c_\pi c_{q,2}^{1/2}\beta}{c_{q,1}^{1/2}\alpha}
    +\frac{7 c_\pi c_{q,3}^{1/2}z_\beta\alpha^{1/2}}{c_{q,1}^{1/2}}\\
    &+(1-\beta)^K
    \left(
    4\mathcal{L}_{\pi,\max}
    +4c_\pi\beta K\mathcal{L}^{1/2}_{q,\max}
    +\frac{42\mathcal{L}_{q,\max}^{1/2}}{1-\gamma}
    \right)
    +\frac{54m\tau}{1-\gamma}
    +8\Bar{L}_\tau \beta,
\end{align*}
where the last inequality follows from $c_\pi\geq 16/(1-\gamma)^2$. Iterating the previous inequality and using $\mathcal{L}_v(0)\leq 4/(1-\gamma)$ by Lemma~\ref{le:boundedness_stochastic_slow}, we obtain
\begin{align}
    \mathbb{E}[\mathcal{L}_v(t)]
    \leq\,&
    \frac{4\gamma^t}{1-\gamma}
    +\frac{40t\gamma^{t-1}}{1-\gamma}
    +\frac{7 c_\pi c_{q,2}^{1/2}\beta}{c_{q,1}^{1/2}\alpha(1-\gamma)}
    +\frac{7 c_\pi c_{q,3}^{1/2}z_\beta\alpha^{1/2}}{c_{q,1}^{1/2}(1-\gamma)}\nonumber\\
    &+\frac{1}{1-\gamma}(1-\beta)^K
    \left(
    4\mathcal{L}_{\pi,\max}
    +4c_\pi\beta K\mathcal{L}^{1/2}_{q,\max}
    +\frac{42\mathcal{L}_{q,\max}^{1/2}}{1-\gamma}
    \right)
    +\frac{54m\tau}{(1-\gamma)^2}
    +\frac{8\Bar{L}_\tau \beta}{1-\gamma}\nonumber\\
    \leq\,&
    \frac{4\gamma^t}{1-\gamma}
    +\frac{40t\gamma^{t-1}}{1-\gamma}
    +\frac{7 c_\pi c_{q,2}^{1/2}\beta}{c_{q,1}^{1/2}\alpha(1-\gamma)}
    +\frac{7 c_\pi c_{q,3}^{1/2}z_\beta\alpha^{1/2}}{c_{q,1}^{1/2}(1-\gamma)}\nonumber\\
    &+\frac{175n^{1/2}m^{3/2}K}{(1-\gamma)^3\tau}(1-\beta)^K
    +\frac{54m\tau}{(1-\gamma)^2}
    +\frac{8\Bar{L}_\tau \beta}{1-\gamma},
    \label{recursion_Lv_slow}
\end{align}
where the last line follows from
$\mathcal{L}_{q,\max}\leq \frac{8nm}{(1-\gamma)^2}$ and
$\mathcal{L}_{\pi,\max}\leq \frac{4}{1-\gamma}+2\tau \log(m)$.

Using \eqref{solving_recursion_q_slow}, \eqref{solving_recursion_vsum}, and \eqref{recursion_Lv_slow} together with Lemma~\ref{le:Nash_Combine_slow}, we have
\begin{align*}
    \mathbb{E}[\text{NG}(\pi_{T,K}^1,\pi_{T,K}^2)]
    \leq\,&
    \frac{4}{1-\gamma}
    \mathbb{E}\left[
    2\mathcal{L}_{\text{sum}}(T)
    +\mathcal{L}_v(T)
    +\mathcal{L}_\pi(T,K)
    +2\tau \log(m)
    \right]\\
    \lesssim\,&
    \frac{T\gamma^{T-1}}{(1-\gamma)^2}
    +\frac{c_\pi c_{q,2}^{1/2}\beta}{c_{q,1}^{1/2}\alpha(1-\gamma)^2}
    +\frac{c_\pi c_{q,3}^{1/2}z_\beta\alpha^{1/2}}{c_{q,1}^{1/2}(1-\gamma)^2}\\
    &+\frac{n^{1/2}m^{3/2}}{(1-\gamma)^4\tau}K(1-\beta)^K
    +\frac{m\tau}{(1-\gamma)^3}
    +\frac{m^2\beta}{\tau(1-\gamma)^4}\\
    \lesssim\,&
    \frac{T\gamma^{T-1}}{(1-\gamma)^2}
    +\frac{n^{1/2}m^{5/2}}{\mu_{\min}\tau^2(1-\gamma)^4}\frac{\beta}{\alpha}
    +\frac{m^{3/2}z_\beta\alpha^{1/2}}{\mu_{\min}^{1/2}\tau^{3/2}(1-\gamma)^4}\\
    &+\frac{n^{1/2}m^{3/2}}{(1-\gamma)^4\tau}K(1-\beta)^K
    +\frac{m\tau}{(1-\gamma)^3}
    +\frac{m^2\beta}{\tau(1-\gamma)^4},
\end{align*}
where the last line follows from the definitions of $c_{q,1}$, $c_{q,2}$, $c_{q,3}$, and $c_\pi$.

In view of the previous inequality, to achieve $\mathbb{E}[\text{NG}(\pi_{T,K}^1,\pi_{T,K}^2)]\leq \epsilon$, it is sufficient to choose $T=\tilde{\mathcal{O}}(1)$, $\tau=\Theta(\epsilon)$, $\alpha=\tilde{\Theta}(\epsilon^5)$, and $\beta=\tilde{\Theta}(\epsilon^8)$, and then take $K=\tilde{\mathcal{O}}(\epsilon^{-8})$. Under this choice, all terms in the preceding bound are at most order $\epsilon$. It follows that the total sample complexity is $TK=\tilde{\mathcal{O}}(\epsilon^{-8})$.

\subsection{Proof of Corollary \ref{co:rationality}}\label{pf:co:rationality}

The following proof idea was previously used in \cite{sayin2021decentralized} to show the rationality of their decentralized $Q$-learning algorithm.

Observe that Theorem~\ref{thm:stochastic_game_slow} can be generalized to the case where the reward is corrupted by noise. Specifically, suppose that player $i$ takes action $a^i$ and player $-i$ takes action $a^{-i}$. Instead of assuming player $i$ receives a deterministic reward $R_i(s,a^i,a^{-i})$, we assume that player $i$ receives a random reward $r^i(s,a^i,a^{-i},\xi)$, where $\xi\in \Xi$ is a random variable with distribution $\mu_\xi(s)$ and is independent of everything else. The proof is identical as long as $r^i+r^{-i}=0$ and the reward is uniformly bounded, i.e., $\max_{s,a^i,a^{-i},\xi}|r^i(s,a^i,a^{-i},\xi)|<\infty$.

Now consider the case where player $i$'s opponent follows a stationary policy $\pi^{-i}$. We incorporate the randomness of player $-i$'s action into the model and introduce a fictitious opponent with only one action $a^*$. In particular, let the random reward function be defined as $\hat{r}^i(s,a^i,a^*,A^{-i})=R_i(s,a^i,A^{-i})$ for all $(s,a^i)$, where $A^{-i}\sim \pi^{-i}(\cdot\mid s)$, and let $\hat{p}(s'\mid s,a^i,a^*)=\sum_{a^{-i}\in\mathcal{A}^{-i}}\pi^{-i}(a^{-i}\mid s)p(s'\mid s,a^i,a^{-i})$. Now the problem can be reformulated as player $i$ playing against the fictitious player with a single action $a^*$, with reward function $\hat{r}^i$ and transition probabilities $\hat{p}$. Using the same proof for Theorem~\ref{thm:stochastic_game_slow}, we obtain the desired finite-sample bound.

\subsection{Statements and Proofs of Supporting Lemmas}

Recall that $\Pi_{\Bar{\epsilon}}=\{(\mu^1,\mu^2)\mid \min_{a^1}\mu^1(a^1)\geq \Bar{\epsilon}/m,\min_{a^2}\mu^2(a^2)\geq \Bar{\epsilon}/m\}$.

\begin{lemma}\label{le:properties_Lyapunov_main_slow}
	The function $V_X(\cdot,\cdot)$ has the following properties.
	\begin{enumerate}[(1)]
		\item $V_X(\cdot,\cdot)$ is $L_{\tau,\Bar{\epsilon}}$ -- smooth on $\Pi_{\Bar{\epsilon}}$, where
        $L_{\tau,\Bar{\epsilon}}=2\left(\frac{m\tau}{\Bar{\epsilon}}+\frac{\max(\|X_1\|_2^2,\|X_2\|_2^2)}{\tau}+\|X_1+X_2^\top\|_2\right)$.

		\item It holds for any $(\pi^1,\pi^2)\in\Delta(\mathcal{A}^1)\times \Delta(\mathcal{A}^2)$ that
		\begin{align*}
			&\langle \nabla_1V_X(\pi^1,\pi^2),\sigma_\tau(X_1\pi^2)-\pi^1 \rangle
            +\langle \nabla_2V_X(\pi^1,\pi^2),\sigma_\tau(X_2\pi^1)-\pi^2 \rangle\\
			\leq\,&
            -V_X(\pi^1,\pi^2)
            +4\max_{a^1,a^2}\left|\sum_{i=1,2}X_i(a^i,a^{-i})\right|.
		\end{align*}

		\item For any $q^i\in\mathbb{R}^{m_i}$, $i\in \{1,2\}$, we have for all $(\pi^1,\pi^2)\in\Pi_{\Bar{\epsilon}}$ that
		\begin{align*}
		&\langle \nabla_1V_X(\pi^1,\pi^2),\sm(q^1)-\sigma_\tau(X_1\pi^2)\rangle
        +\langle \nabla_2V_X(\pi^1,\pi^2),\sm(q^2)-\sigma_\tau(X_2\pi^1)\rangle\\
		\leq\,&
        4\Bar{\epsilon}\left(\frac{2m\tau}{\Bar{\epsilon}}
		+ \sum_{i=1,2}\|X_i\|_2\right)
        +\frac{2}{\tau}\left(\frac{m\tau}{\Bar{\epsilon}}
		+ \max(\|X_1\|_2,\|X_2\|_2)\right)
        \sum_{i=1,2}\| q^i-X_i\pi^{-i}\|_2.
	\end{align*}
	\end{enumerate}
\end{lemma}

\proof{Proof of Lemma~\ref{le:properties_Lyapunov_main_slow}.}
The proofs of Lemma~\ref{le:properties_Lyapunov_main_slow} (1) and (2) are identical to those of Lemma~\ref{le:properties_Lyapunov_main} (1) and (2), and therefore are omitted. Here, we prove Lemma~\ref{le:properties_Lyapunov_main_slow} (3).

For any $q^1\in\mathbb{R}^{m_1}$ and $q^2\in\mathbb{R}^{m_2}$, using the formula of the gradient of $V_X(\cdot,\cdot)$ from \eqref{eq:V_gradient} and the first-order optimality condition of the softmax map in inner-product form, we have
\begin{align*}
	&\langle \nabla_1V_X(\mu^1,\mu^2),\sm(q^1)-\sigma_\tau(X_1\mu^2)\rangle\\
    =\,&
    \tau\langle \nabla \nu(\sigma_\tau(X_1 \mu^2))- \nabla \nu(\mu^1),
    \sm(q^1)-\sigma_\tau(X_1\mu^2) \rangle\\
    &+
    ( \sigma_\tau(X_2 \mu^1)-\mu^2)^\top X_2
    ( \sm(q^1)-\sigma_\tau(X_1\mu^2) )\\
    \leq\,&
    \tau\|\nabla \nu(\sigma_\tau(X_1 \mu^2))- \nabla \nu(\mu^1)\|_2
    \|\sm(q^1)-\sigma_\tau(X_1\mu^2) \|_2\\
    &+
    \|\sigma_\tau(X_2 \mu^1)-\mu^2\|_2 \|X_2\|_2
    \| \sm(q^1)-\sigma_\tau(X_1\mu^2)\|_2\\
    \leq\,&
    2\left(\frac{m\tau}{\Bar{\epsilon}}+\|X_2\|_2\right)
    \left(
    \| \sm(q^1)-\sigma_\tau(q^1)\|_2
    +\| \sigma_\tau(q^1)-\sigma_\tau(X_1\mu^2)\|_2
    \right)\\
    \leq\,&
    2\left(\frac{m\tau}{\Bar{\epsilon}}+\|X_2\|_2\right)
    \left(
    2\Bar{\epsilon}
    +\frac{1}{\tau}\| q^1-X_1\mu^2\|_2
    \right),
\end{align*}
where the second last inequality follows from the $m/\Bar{\epsilon}$ -- smoothness of $\nu(\cdot)$ on $\Pi_{\Bar{\epsilon}}$ and the fact that the softmax function is $\frac{1}{\tau}$ -- Lipschitz continuous with respect to $\|\cdot\|_2$ \citep{gao2017properties}, and the last inequality follows from Lemma~\ref{le:sm_to_sigma_tau}. Similarly, for any $q^2\in\mathbb{R}^{m_2}$,
\begin{align*}
    \langle \nabla_2V_X(\mu^1,\mu^2),\sm(q^2)-\sigma_\tau(X_2\mu^1) \rangle
    \leq
    2\left(\frac{m\tau}{\Bar{\epsilon}}+\|X_1\|_2\right)
    \left(
    2\Bar{\epsilon}
    +\frac{1}{\tau}\| q^2-X_2\mu^1\|_2
    \right).
\end{align*}
The result follows by adding the previous two inequalities.
\hfill\qed
\endproof
\section{Numerical Simulations}\label{ap:numerical}
We conduct numerical simulations to investigate the impact of choosing different $\tau$, which is used to define the softmax operator in Algorithms \ref{algo:matrix_fast} and \ref{algo:stochastic_game}. Our theoretical results indicate that there is an asymptotically non-vanishing bias due to using a positive $\tau$. Intuitively, since a softmax policy always has strictly positive entries while a Nash equilibrium policy can have zero entries, we cannot, in general, expect the Nash gap to converge to zero. To demonstrate this phenomenon, consider the following example of a zero-sum matrix game. Let 
\begin{align*}
    R_1=\begin{bmatrix}
        N & 1 & -1\\
        -1 & 0 & 1\\
        1 & -1 & 0
    \end{bmatrix}
\end{align*}
be the payoff matrix for player $1$, and let $R_2=-(R_1)^{\top}$, where $N>0$ is a tunable parameter. Note that this matrix game has a unique Nash equilibrium, which goes to the joint policy $\pi^1=(1/3,2/3,0)$, $\pi^2=(0,2/3,1/3)$ as $N\rightarrow\infty$. In our simulations, we use constant stepsizes $\alpha_k\equiv 0.5$ and $\beta_k\equiv 0.01$ and run Algorithm \ref{algo:matrix_fast} for $100$ trajectories (each has $K=2000$ iterations). Then, we plot the average Nash gap (averaged over the $100$ trajectories) as a function of the number of iterations $k$ in Figure \ref{figure1} for different temperatures $\tau$. To enable a fair comparison, we use the normalized $q$-function to compute the softmax, that is, instead of directly using $\sigma_\tau(q_k^i)$ in Algorithm \ref{algo:matrix_fast}, we use $\sigma_\tau(q_k^i/\|q_k^i\|_2)$. As we can see in Figure \ref{figure1}, as $\tau$ increases, the asymptotic error also increases, which is consistent with our theoretical results.
\begin{figure}[ht]
\begin{center}
    \includegraphics[width=0.6\textwidth]{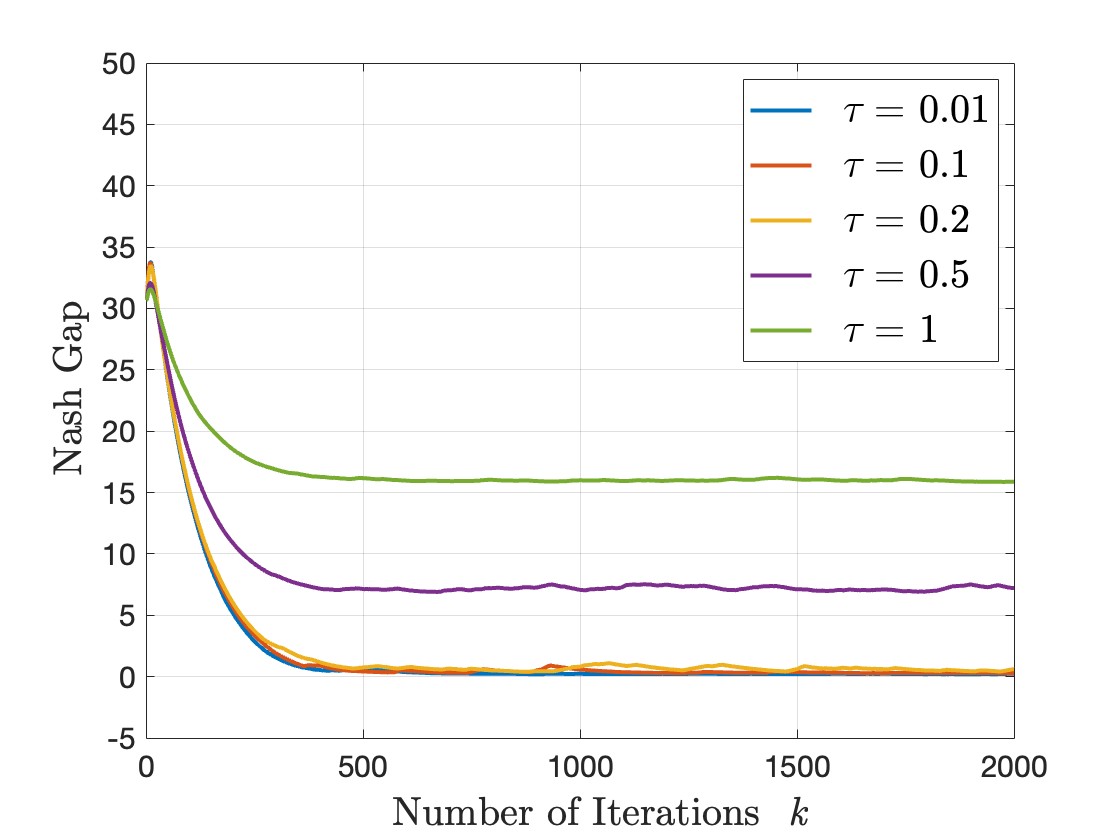}
\caption{\centering The Nash Gap for Different Temperatures $\tau$}
\end{center}
\label{figure1}
\end{figure}

\end{document}